\documentclass[lettersize,journal]{IEEEtran}
\usepackage{amsmath,amsfonts} 
\usepackage{array}

\usepackage{textcomp}
\usepackage{stfloats}
\usepackage{url}
\usepackage{verbatim}
\usepackage{graphicx}
\usepackage{cite}
\usepackage{pifont}
\usepackage{booktabs} 
\usepackage{bm}
\usepackage[T1]{fontenc} 
\usepackage[linesnumbered,ruled,vlined]{algorithm2e}
\usepackage{hyperref}
\usepackage{subcaption}
\usepackage{orcidlink}
\usepackage{makecell}
\usepackage{amssymb}

\newcommand{\bE}{\mathbb{E}}

\definecolor{mine}{RGB}{205, 232, 248}%

\allowdisplaybreaks

\usepackage{algorithmic}
\usepackage{amsthm} 
\newtheorem{theorem}{Theorem}

\newtheorem{assumption}{Assumption}
\theoremstyle{remark}

\begin{document}

\captionsetup[figure]{labelformat=simple, labelsep=period} 

\captionsetup[table]{name={TABLE},labelsep=space} 

\title{Mildly Conservative Regularized Evaluation for Offline Reinforcement Learning}

\author{Haohui Chen~\orcidlink{0000-0001-9660-0948} and Zhiyong Chen~\orcidlink{0000-0002-2033-4249}  

\thanks{Haohui Chen is with the School of Automation, Central South University, Changsha 410083, China (e-mail: \href{mailto:haohuichen@csu.edu.cn}{haohuichen@csu.edu.cn}).}
\thanks{Zhiyong Chen is with the School of Engineering, University of Newcastle, Callaghan, NSW 2308, Australia (e-mail: \href{mailto:zhiyong.chen@newcastle.edu.au}{zhiyong.chen@newcastle.edu.au}).}
\thanks{Corresponding author: Zhiyong Chen.}}

\maketitle

\begin{abstract}
Offline reinforcement learning (RL) seeks to learn optimal policies from static datasets without further environment interaction. A key challenge is the distribution shift between the learned and behavior policies, leading to out-of-distribution (OOD) actions and overestimation. To prevent gross overestimation, the value function must remain conservative; however, excessive conservatism may hinder performance improvement. 
To address this, we propose the mildly conservative regularized evaluation (MCRE) framework, which balances conservatism and performance by combining temporal difference (TD) error with a behavior cloning term in the Bellman backup. Building on this, we develop the mildly conservative regularized Q-learning (MCRQ) algorithm, which integrates MCRE into an off-policy actor–critic framework. Experiments show that MCRQ outperforms strong baselines and state-of-the-art offline RL algorithms on benchmark datasets.
\end{abstract}

  

\section{Introduction} \label{section:introduction}


\IEEEPARstart{R}{einforcement} learning (RL) has achieved success across various domains. In classical online RL, agents learn optimal policies through real-time interaction. However, in real-world settings, continuous interaction is often impractical due to data collection challenges, safety concerns, and high costs. Offline RL, also known as batch RL \cite{lange2012batch}, addresses this by learning policies from static datasets generated by unknown behavior policies, eliminating the need for environment interaction. This makes offline RL suitable for real-world applications without simulators, such as energy optimization \cite{huOptimizingReinforcementLearning2025}, robotics \cite{liu2024diffskill}, and recommendation systems \cite{zhangTextBasedInteractiveRecommendation2022}.
 
A key challenge in offline RL is the distribution shift between the learned policy and the behavior policy, as the latter often fails to sufficiently cover the state–action space. As a result, directly applying off-policy online RL algorithms in offline settings typically yields poor performance \cite{rezaeifar2022offline}. The evaluation of out-of-distribution (OOD) actions introduces extrapolation errors, which can be amplified through bootstrapping, leading to significant overestimation bias \cite{sikchi2024dualf-DVL}.

In critic regularization, the regularization term is typically integrated directly into the critic loss. The core idea is to mitigate overestimation of OOD actions by regularizing the Q-function or state value function, thereby addressing distribution shift. However, overly conservative regularization can lead to excessively low Q-values \cite{pessimismofflineRLzhang2024}, limiting the agent’s ability to explore and exploit, and exposing the actor to inaccurate Q-values, ultimately slowing convergence or leading to suboptimal policies \cite{rao2025isfors}. The mildly conservative regularized evaluation (MCRE) proposed in this paper, a form of critic regularization, highlights how over-conservatism impedes policy improvement by analyzing the gap between the learned and true Q-values, as further supported by ablation experiments.
 
In actor regularization, policy update stability is often achieved by constraining the target policy to remain close to the behavior policy distribution. However, excessive regularization can hinder exploration, preventing the target policy from fully utilizing the Q-function and impairing convergence. Over-conservatism poses challenges even when learning from expert datasets, regardless of whether regularization is applied to the critic or actor \cite{yangHundredsGuideMillions2024}. Therefore, conservatism should remain mild \cite{lyuMildlyConservativeQLearning2022}.
  
Temporal difference (TD) error is a fundamental mechanism for aligning value estimates with temporal targets, guiding policy optimization through iterative feedback. However, in offline RL, distribution shift and limited data coverage introduce compounding biases and OOD actions, which are amplified through TD updates. To address this, the proposed MCRE framework integrates TD error with a behavior cloning term to improve value estimation and reduce the impact of distribution shift and OOD actions. MCRE combines two complementary components: TD error regularization, which refines value estimates via temporal feedback, and behavior cloning, which anchors the policy to the dataset’s empirical action distribution. Together, they form a mildly conservative regularization term.

This dual mechanism ensures that TD updates operate within a policy-constrained regime, suppressing OOD actions while preserving their role as corrective signals. Crucially, MCRE imposes only a mild constraint on the TD error, allowing the target policy to deviate slightly from the behavior policy and thus avoid over-conservatism. By embedding this constraint into the standard Bellman backup, MCRE enables effective policy optimization without incurring performance degradation due to excessive conservatism.

The primary contributions of this paper are summarized as follows.

1) We propose a novel framework that integrates TD error with a behavior cloning term within the standard Bellman backup. This unified design balances value estimation and policy conservatism, with the behavior cloning term constraining policy updates to the support of the behavior distribution, effectively suppressing OOD actions.

2) We theoretically prove that MCRE converges under both sampling and no-sampling error conditions. In both cases, the gap between the learned and true Q-value and state-value functions is effectively bounded. Furthermore, we analyze policy suboptimality in offline RL, showing that the difference between the learned suboptimal policy and the true optimal policy is also upper-bounded.

3) We provide a theoretical analysis of the Q-function learned by MCRE under both sampling and no-sampling error, showing that stronger constraints lead to over-conservatism, increasing the gap between the learned and true Q-functions and degrading performance. This finding is supported by ablation studies, which reveal a positive correlation between conservatism in MCRE and the strength of its constraints.

4) Based on MCRE, we propose a novel and efficient offline RL algorithm, mildly conservative regularized Q-learning (MCRQ), which integrates an offline actor-critic framework to effectively mitigate the challenges of over-conservatism and distribution shift. Experimental results demonstrate that MCRQ outperforms baselines and state-of-the-art (SOTA) algorithms on various MuJoCo tasks in the D4RL benchmark.

The remainder of the paper is organized as follows. Section~\ref{section:related-work} reviews related work and the motivation behind MCRE. Section~\ref{section:preliminaries} introduces the necessary preliminaries. Section~\ref{section:mcre} details the MCRE framework and provides its theoretical analysis. Section~\ref{section:mcrq} presents the practical implementation of the MCRQ algorithm. Section~\ref{section:experiments} reports experimental results on D4RL tasks and compares MCRQ with baselines and recent SOTA algorithms to demonstrate its superior performance. Finally, Section~\ref{section:conclusion} concludes the paper and discusses potential future directions.

\section{Related Work} \label{section:related-work}

Offline RL algorithms that mitigate over-conservatism and distribution shift can be classified into critic and actor regularization methods \cite{huangEfficientOfflineReinforcement2024}, as discussed below.

\subsection{Critic Regularization}

Critic regularization typically imposes a penalty on the TD target or critic loss to make the learned value function closer to the true one. Fakoor et al. \cite{Fakoor2021ContinuousDoublyConstrained} proposed continuous doubly constrained (CDC), a method for batch RL that adds two novel batch-RL regularizers to the standard off-policy actor critic (AC) algorithm, for mitigating the overestimation bias caused by distribution shift. To reduce high variance observed in multi-step evaluation, Brandfonbrener et al. \cite{brandfonbrener2021offline} developed a one-step framework, which implements constrained policy improvement. Fujimoto et al. \cite{Fujimoto2021td3bc} enhanced the traditional twin delayed deep deterministic policy gradient (TD3) by integrating behavior cloning to alleviate the impact of distribution shift. To mitigate the optimistic Q-values caused by iterative errors during policy optimization, Kumar et al. \cite{kumarConservativeQLearningOffline2020} proposed conservative Q-learning (CQL), which aims to learn a conservative Q-function. 
To alleviate the computational burden of CQL, Kostrikov et al. \cite{kostrikovOfflineReinforcementLearning2021} replaced the KL divergence in CQL with the fisher divergence. 
Kostrikov et al. \cite{Kostrikovimplicitq-learning2021} introduced implicit Q-learning (IQL), an on-policy SARSA-style offline RL algorithm that uses advantage-weighted behavior cloning to avoid querying OOD actions. Bai et al. \cite{bai2022pessimistic} introduced pessimistic bootstrapping for offline RL (PBRL), which applies pessimistic updates to Q-functions. Huang et al. \cite{huangOfflineReinforcementLearning2024} proposed an offline actor critic framework with behavior value regularization, which helps mitigate over-optimistic Q-values and reduce bias.

\subsection{Actor Regularization}

Actor regularization usually works through ensuring target policy stays within the distribution of behavior policy. However, this constraint tends to be overly conservative because it requires the target policy to choose an action similar to the behavior policy with a given state. This restriction may affect the performance of the target policy, especially if the behavior policy is not optimal \cite{ran2023policy}. Actor regularization methods generally add constraints on actor loss \cite{huangEfficientOfflineReinforcement2024}. 
Fujimoto et al. \cite{fujimotoOffPolicyDeepReinforcement2019} proposed batch-constrained deep Q-learning (BCQ), which employs a generative model to generate only previously observed actions.
However, Kumar et al. \cite{kumarStabilizingOffPolicyQLearning2019} pointed out that BCQ is overly aggressive, and developed the bootstrapping error accumulation reduction (BEAR), an innovative constraint-based approach specifically designed to break the chain of error propagation in Q-learning updates.
Wu et al.\cite{wuUncertaintyWeightedActorCritic2021} proposed uncertainty weighted actor critic (UWAC), which utilizes Monte Carlo dropout for uncertainty estimation, significantly reducing computational overhead. Nair et al. \cite{nair2020awac} proposed advantage weighted actor critic (AWAC), which utilizes offline data and supports online fine-tuning. Ran et al. \cite{ran2023policy} proposed policy regularization with dataset constraint (PRDC), which utilizes OOD actions as guidance for policy optimization when updating a policy in a given state. PRDC imposes a milder constraint, yet still helps avoid OOD actions. 

\medskip

Although critic and actor regularization methods have made significant progress in mitigating OOD actions and distribution shift, critic regularization may lead to excessive pessimism, especially with OOD actions, thus hindering policy optimization. Similarly, actor regularization may limit the flexibility of a policy by constraining the target policy to remain within the distribution of behavior policy, which may restrict overall performance. These overly conservative constraints can lead to convergence to a local optimum or even cause the policy to diverge, thereby hindering further policy improvement.
To alleviate these problems, we propose the MCRE framework, which reduces the risk of OOD actions and prevents performance degradation due to over-conservatism.

\section{Preliminaries} \label{section:preliminaries}


The RL paradigm formalized as a Markov decision process (MDP) is characterized by the tuple $M = (\mathcal{S}, \mathcal{A}, \mathcal{P} , r, \rho_0, \gamma)$, where $\mathcal{S} $ and $ \mathcal{A} $ represent the state and action spaces respectively, $\mathcal{P} $ defines the state transition dynamics, $r: \mathcal{S} \times \mathcal{A} \rightarrow \mathbb{R} $ denotes the bounded reward function with $ \|r\|_\infty \leq r_{\max}$, $\rho_0$ specifies the initial state distribution, $\gamma \in [0,1)$ is the discount factor. 

During environment interaction at timestep $t$, a deterministic target policy $\pi \in \Pi: \mathcal{S} \rightarrow \mathcal{A}$ generates an action $a_t = \pi(s_t)$, which leads to a transition to the next state according to $\mathcal{P}(s_{t+1} \mid s_t, a_t)$ and yields an instantaneous scalar reward $r_t$.
The $\gamma$-discounted trajectory return is formulated as $R_t^\gamma = \sum_{i=t}^{\infty} \gamma^{i-t} r_i$. For a policy $\pi$, it is defined as
\begin{align} \label{equation:eta}
  \Gamma(\pi) = \bE[R_0^\gamma].
\end{align}

The objective in RL is to derive a Bellman-optimal mapping $\pi_*$ that maximizes $\Gamma(\pi)$
\begin{align} \label{equation:pi*}
  \pi_* = \arg\max_\pi \Gamma(\pi).
\end{align}
The state value function (V-function) is $V^\pi(s) = \bE[R_0^\gamma \mid s_0=s]$, and the state-action value function (Q-function) is $Q^\pi(s, a) = \bE[R_0^\gamma \mid s_0=s, a_0=a]$. For deterministic policies, $V^\pi(s)=Q^\pi(s, a)$. Consequently, both are the unique fixed points of the Bellman equations:
\begin{align}
  Q^\pi(s, a) &= r(s, a) + \gamma \bE_{s^{\prime} \sim \mathcal{P} (\cdot \mid s, a)}[Q^\pi(s^{\prime}, \pi(s^{\prime}))] \label{equation:Bellman-Q} \\
  V^\pi(s) &= r(s, \pi(s)) + \gamma \bE_{s^{\prime} \sim \mathcal{P} (\cdot \mid s, \pi(s))}[V^\pi(s^{\prime})]. \label{equation:Bellman-V}
\end{align}  
The Bellman backup for obtaining the corresponding Q-function is defined as
\begin{align}\label{equation:Bellman-operator}
(\mathcal{T}^\pi Q)(s, a) = r(s, a) + \gamma \bE_{s^{\prime} \sim \mathcal{P} (\cdot \mid s, a)}[Q(s^{\prime}, \pi(s^{\prime}))],
\end{align}
which is also the actual Bellman operator. Thus, we can rewrite \eqref{equation:Bellman-Q} as $Q^\pi(s, a) = (\mathcal{T}^\pi Q^\pi)(s, a)$.

Given a dataset $\mathcal{D} = \{(s, a, r, s^{\prime})\}$ 
consisting of tuples from trajectories collected under some unknown behavior policies, the Q-function is updated according to the following rule:
\begin{align}\label{equation:policy-evaluation}
\hat{Q}_{k+1} = \arg\min_Q \bE_{(s,a)\sim \mathcal{D}} [(\hat{\mathcal{T}}^\pi \hat{Q}_k)(s, a) - Q(s,a)]^2,
\end{align}
where  $\hat{\mathcal{T}}^\pi$
denotes an empirical Bellman operator, which approximates the Bellman operator in \eqref{equation:Bellman-operator} using a finite number of samples. Specifically, 
\begin{align} \label{equation:empirical-Bellman}
(\hat{\mathcal{T}}^\pi \hat{Q}_k)(s, a)=r+\gamma \bE_{s^{\prime} \sim \hat{\mathcal{P}}(\cdot\mid s, a)}[\hat{Q}_k(s^{\prime}, \pi(s^{\prime}))],    
\end{align}
where $\hat{\mathcal{P}}$ is the empirical state transition probability.  The expectation $\bE_{(s,a)\sim \mathcal{D}}$ in \eqref{equation:policy-evaluation} is computed as the empirical mean over the samples in $\mathcal{D}$, and is subsequently abbreviated as $\bE$ for notational simplicity.  
Accordingly, the policy improvement step is defined as: 
\begin{align}\label{equation:policy-improvement}
\pi_{k+1} = \arg\max_\pi \bE[\hat{Q}_{k+1}(s, \pi(s))]. 
\end{align}


Under the offline RL paradigm, agents infer the optimal policy from a fixed dataset $\mathcal{D}$ generated by an unknown behavior policy, denoted as $\pi_\beta$.   Online off-policy algorithms are susceptible to distribution shift because $a^{\prime}\sim\pi(s^{\prime})$ used in the Bellman backup \eqref{equation:empirical-Bellman} may not be sampled from the behavior policy  $\pi_{\beta}$, since $a^{\prime}$ can fall outside the distribution of $\pi_{\beta}$ owing to the disparity between the distributions of $\pi$ and $\pi_{\beta}$,  which significantly disrupts the training process \cite{fujimotoOffPolicyDeepReinforcement2019}.

During policy evaluation, the Bellman operator employs actions drawn from the evaluation policy to compute state-action value updates. In contrast, the Q-function is exclusively updated on behavior policy-generated action samples derived from the static dataset. This discrepancy leads to extrapolation errors, causing the off-policy AC algorithm to produce inaccurately high Q-values for OOD actions \cite{Fujimoto2021td3bc}, \cite{Tarasov2023revisitingminimalist}.

\section{Mildly Conservative Regularized Evaluation} \label{section:mcre}

This section introduces the MCRE framework and analyzes its properties.

\subsection{Framework of MCRE}

TD error assesses the accuracy of the current Q estimate relative to the TD target \cite{silverMasteringGameGo2016}. 
When the current estimate $Q(s_t, \pi(s_t))$ becomes an unbiased approximation of the true value $Q^\pi(s,a)$, TD error satisfies zero under expectation. The TD error reflects not only the discrepancy between two time steps but also, more importantly, the discrepancy between the estimate $Q(s_t, \pi(s_t))$ and the true state-action value $Q^\pi(s,a)$. The core mechanism of TD learning is correcting $Q(s_t, \pi(s_t))$ through newly obtained information \cite{zhao2025RLBook}. 
However, under the offline RL setting, $Q(s_t,\pi(s_t))$ might already deviate significantly from $Q^\pi(s,a)$ due to distribution shift or limited coverage of the dataset. Consequently, TD error itself could also be biased.

Unlike prior works, we propose a mechanism to constrain deviations of the learned Q-function from its true value by quantifying the discrepancy between the current Q-estimate and the TD error. However, inaccuracies in TD error may arise due to policy mismatch between $\pi(s)$ and $\pi_\beta(s)$. To mitigate this, we introduce a behavior cloning term, which encourages the learned actions to align with those in the dataset and helps avoid querying OOD actions. Furthermore, by integrating TD error as a correction term into the standard Bellman backup, the proposed MCRE framework jointly suppresses distribution shift and OOD actions inherent to offline RL.

In contrast to \eqref{equation:policy-evaluation}, a new Q-function update rule is given by
\begin{align}\label{equation:Q-update}
\hat{Q}_{k+1}=\arg \min_Q \bE [(\hat{\mathcal{Z}}^\pi \hat{Q}_k)(s, a)-Q(s, a)]^2,
\end{align}
where $(\hat{\mathcal{Z}}^\pi \hat{Q}_k)(s, a)$ is the  modified empirical Bellman operator, defined as
\begin{align}\label{equation:Zhat}
(\hat{\mathcal{Z}}^\pi \hat{Q}_k)(s, a) = & (1-\upsilon)(\hat{\mathcal{T}}^\pi \hat{Q}_k)(s, a) \nonumber \\ & +\upsilon(\hat{\mathcal{H}}^\pi \hat{Q}_k)(s, a)  -\gamma \mathcal{I}^\pi(s,a)
\end{align}
which consists of three components: the empirical Bellman operator defined in \eqref{equation:empirical-Bellman}, the empirical TD Bellman operator defined as
\begin{align}
(\hat{\mathcal{H}}^\pi \hat{Q}_k)(s, a)=  r +\gamma \bE_{s^{\prime} \sim \hat{\mathcal{P}}(\cdot \mid s, a)} \big[\hat{Q}_k(s^{\prime},\pi(s^{\prime})) \nonumber\\
- \big(r + \gamma \hat{Q}_k(s^{\prime},\pi(s^{\prime})) - \hat{Q}_k(s,\pi(s))\big)\big],    \label{HQk}
\end{align} and the behavior cloning term
\begin{align*}
    \mathcal{I}^\pi (s,a) = \omega (\pi(s) - a)^2.
\end{align*} 
Here, $\upsilon \in [0,1]$ and $\omega \geq 0$ are hyperparameters that modulate the weights of their respective components.

Applying the modified empirical Bellman operator \eqref{equation:Zhat} in the update rule \eqref{equation:Q-update} is referred to as MCRE. The concept of mild conservativeness is described below.
The empirical TD Bellman operator \eqref{HQk} can be  rewritten as
\begin{align*} 
(\hat{\mathcal{H}}^\pi \hat{Q}_k)(s, a) = & (\hat{\mathcal{T}}^\pi \hat{Q}_k)(s, a) \nonumber \\
& - \gamma \big((\hat{\mathcal{T}}^\pi \hat{Q}_k)(s, a) - \hat{Q}_k(s,\pi(s)) \big).
\end{align*}
When $(\hat{\mathcal{T}}^\pi \hat{Q}_k)(s, a) > \hat{Q}_k(s, \pi(s))$, indicating an overestimation, $(\hat{\mathcal{H}}^\pi \hat{Q}_k)(s, a)$ yields a lower value than $(\hat{\mathcal{T}}^\pi \hat{Q}_k)(s, a)$. Conversely, when $(\hat{\mathcal{T}}^\pi \hat{Q}_k)(s, a) < \hat{Q}_k(s, \pi(s))$, indicating an underestimation, $(\hat{\mathcal{H}}^\pi \hat{Q}_k)(s, a)$ yields a higher value. In both cases, $(\hat{\mathcal{H}}^\pi \hat{Q}_k)(s, a)$ is drawn closer to $\hat{Q}_k(s, \pi(s))$. This adjustment effect becomes more pronounced as the discount factor $\gamma$ approaches 1. 


As $\gamma$ approaches 1, encouraging agents to prioritize long-term rewards over short-sighted decisions \cite{zhao2025RLBook}, $(\hat{\mathcal{H}}^\pi \hat{Q}_k)(s, a)$ tends to align more closely with $\hat{Q}_k(s, \pi(s))$. This calibration is further justified by the regularization term $\mathcal{I}^\pi(s, a)$, which promotes alignment with $\hat{Q}_k(s, \pi(s))$ only when the corresponding state-action pairs are well supported by the offline dataset.


In offline RL, conservatism refers to avoiding the overestimation of Q-values for OOD actions, thereby mitigating extrapolation errors. CQL \cite{kumarConservativeQLearningOffline2020} addresses this by minimizing the maximum Q-value produced by the target policy while maximizing the Q-value derived from the behavior policy. SVR \cite{mao2023SupportedValueRegularization} enforces the Q-values of all OOD actions to equal the minimum Q-value observed in the offline dataset. Both methods adopt overly conservative strategies.

In contrast, MCRE introduces a mildly conservative mechanism via $(\hat{\mathcal{H}}^\pi \hat{Q}_k)(s,a)$, which pulls the Q-function toward $\hat{Q}_k(s, \pi(s))$. Here, $\pi(s)$ is itself regularized through $\mathcal{I}^\pi$, ensuring it remains close to the dataset distribution. Consequently, if $(\hat{\mathcal{T}}^\pi \hat{Q}_k)(s,a)$ becomes overly optimistic or overly conservative, this formulation introduces a corrective adjustment toward more reliable value estimates. Rather than aggressively suppressing Q-values in unsupported regions, MCRE promotes reliability through targeted regularization.
Thus, compared to CQL and SVR, the conservatism imposed by MCRE is mild, offering a balance between caution and expressiveness in value estimation.


The iterative update rule for $\hat{Q}_{k+1}$ can be derived by equating the gradient of the loss function in \eqref{equation:Q-update} to zero  
\begin{align*}
\frac{\partial}{\partial Q(s,a)} \bE[(\hat{\mathcal{Z}}^\pi \hat{Q}_k)(s,a) - Q(s,a)]^2 = 0,
\end{align*}
which yields the following closed-form expression:
\begin{align}\label{equation:Q-Zhat}
\hat{Q}_{k+1}(s, a)=(\hat{\mathcal{Z}}^\pi \hat{Q}_k)(s, a).
\end{align}



\subsection{Convergence Analysis of MRCE}



The operator $\hat{\mathcal{T}}^\pi$ in \eqref{equation:Zhat} is the same as that used in MOAC, where its contraction property has already been established \cite{huangMildPolicyEvaluation2024}. The key distinction in the MCRE setting lies in the incorporation of the additional term $(\hat{\mathcal{H}}^\pi \hat{Q}_k)(s, a)$.  
Accordingly, we focus on establishing the contraction property of the operator $\hat{\mathcal{H}}^\pi$. It follows that the overall operator $\hat{\mathcal{Z}}^\pi$ also satisfies the contraction property, thereby guaranteeing the convergence of the Q-function update rule. The result is summarized in the following theorem. 

\begin{theorem}\label{theorem:convergence-operator}
The MCRE operator $\hat{\mathcal{Z}}^\pi$ is a contraction mapping in the $\mathcal{L}_\infty$ norm, given any initial Q-value function, and the sequence $\hat{Q}_0(s, a), \hat{Q}_1(s, a), \ldots$ generated by the iterative update rule \eqref{equation:Q-Zhat} converges to a fixed point.
\end{theorem}

\begin{proof} We begin by proving that $\hat{\mathcal{Z}}^\pi$ is a contraction mapping. For any two distinct Q-functions $\hat{Q}_1(s, a)$ and $\hat{Q}_2(s, a)$, it follows from the definition \eqref{equation:Zhat} that
  \begin{align}\label{equation:Zhat-Q1-Q2-blue}
  &\bigg\|(\hat{\mathcal{Z}}^\pi \hat{Q}_1)(s, a)-(\hat{\mathcal{Z}}^\pi \hat{Q}_2)(s, a)\bigg\|_{\infty} \nonumber \\
  \leq & (1-\upsilon)\big\|\big(\hat{\mathcal{T}}^\pi \hat{Q}_1\big)(s, a)-\big(\hat{\mathcal{T}}^\pi \hat{Q}_2\big)(s, a)\big\|_{\infty} \nonumber \\
  & + \upsilon\big\|\big(\hat{\mathcal{H}}^\pi \hat{Q}_1\big)(s, a)-\big(\hat{\mathcal{H}}^\pi \hat{Q}_2\big)(s, a)\big\|_{\infty}.
\end{align}
The contraction property of $\hat{\mathcal{T}}^\pi$ is given by \cite{huangMildPolicyEvaluation2024} as follows:
\begin{align}\label{equation:That-Q1-Q2-blue}
    & \big\|(\hat{\mathcal{T}}^\pi \hat{Q}_1)(s, a)-(\hat{\mathcal{T}}^\pi \hat{Q}_2)(s, a)\big\|_{\infty} \nonumber \\
    \leq & \gamma \max_{s,a}\big|\hat{Q}_1(s, a)-\hat{Q}_2(s, a)\big|.
  \end{align}
Next, we perform the following calculation for $\hat{\mathcal{H}}^\pi$:
\begin{align}\label{equation:Hhat-Q1-Q2-blue}
  & \bigg\|(\hat{\mathcal{H}}^\pi \hat{Q}_1)(s, a)-(\hat{\mathcal{H}}^\pi \hat{Q}_2)(s, a)\bigg\|_{\infty} \nonumber \\
  = & \gamma \Bigg\|(1 - \gamma) \bE_{s^{\prime} \sim \hat{\mathcal{P}}(\cdot \mid s, a)} \Big[ \hat{Q}_1(s^{\prime},\pi(s^{\prime})) - \hat{Q}_2(s^{\prime},\pi(s^{\prime}))\Big] \nonumber \\
  & - \Big[\hat{Q}_1(s,\pi(s))-\hat{Q}_2(s,\pi(s))\Big] \Bigg\|_{\infty} \nonumber \\
  \leq & (\gamma - \gamma^2) \max_{s^{\prime}} \Big| \hat{Q}_1(s^{\prime},\pi(s^{\prime}))-\hat{Q}_2(s^{\prime},\pi(s^{\prime})) \Big| \nonumber \\
  & + \gamma \max_{s}\Big| \hat{Q}_1(s,\pi(s))-\hat{Q}_2(s,\pi(s)) \Big| \nonumber \\
  \leq & (2\gamma - \gamma^2) \max_{s,a} \big|\hat{Q}_1(s,a)-\hat{Q}_2(s,a)\big|.
\end{align}  
Then, substituting \eqref{equation:That-Q1-Q2-blue} and \eqref{equation:Hhat-Q1-Q2-blue} into \eqref{equation:Zhat-Q1-Q2-blue} yields
\begin{align}\label{equation:Zhat-Q1-Q2-Result-blue}
& \|(\hat{\mathcal{Z}}^\pi \hat{Q}_1)(s, a)-(\hat{\mathcal{Z}}^\pi \hat{Q}_2)(s, a)\|_{\infty} \nonumber \\
\leq & (\gamma + \upsilon\gamma - \upsilon\gamma^2) \max_{s,a} |\hat{Q}_1(s,a)-\hat{Q}_2(s,a)|.
\end{align}
Since $(\gamma + \upsilon\gamma - \upsilon\gamma^2) \in [0,1)$ for $\gamma \in [0,1)$ and $\upsilon \in [0,1]$, it follows that $\hat{\mathcal{Z}}^\pi$ is a contraction mapping under the $\mathcal{L}_\infty$ norm.




This contraction property guarantees the presence and uniqueness of a fixed point $\hat{Q}^\pi_*$ satisfying $\hat{Q}^\pi_* = (\hat{\mathcal{Z}}^\pi \hat{Q}^\pi_*)$. The convergence of the sequence generated by the iterative update rule \eqref{equation:Q-Zhat} is rigorously characterized by $\|\hat{Q}_{k} - \hat{Q}^\pi_*\|_\infty \leq (\gamma + \upsilon\gamma - \upsilon\gamma^2)^{k} \|\hat{Q}_0 - \hat{Q}^\pi_*\|_\infty$. This result establishes a geometric convergence rate, ensuring that $\hat{Q}_k$ converges uniformly to $\hat{Q}^\pi_*$ as $k \to \infty$.
\end{proof}


\subsection{Error Analysis of MCRE Value Estimates}

When there is no sampling error, the empirical expectation $\bE_{s^{\prime} \sim \hat{\mathcal{P}}(\cdot \mid s, a)}$ is replaced by the true expectation $\bE_{s^{\prime} \sim \mathcal{P}(\cdot \mid s, a)}$, and $\hat{\mathcal{T}}^\pi$ is replaced by ${\mathcal{T}}^\pi$ in \eqref{equation:Zhat} and 
\eqref{HQk}. Under this setting, we define the corresponding 
modified actual Bellman operator $\mathcal{Z}^\pi$ and the 
actual TD Bellman operator $\mathcal{H}^\pi$.
We first conduct an error analysis of the Q-function obtained via MCRE using these actual operators. Specifically, our goal is to characterize an upper bound on the discrepancy between the learned Q-function and the true value function. 
Moreover, we also characterize the discrepancy between the V-function obtained through MCRE and the true value.

The results in the following two theorems hold when $\upsilon$ is sufficiently small in \eqref{equation:Zhat} to ensure that $\hat{\mathcal{T}}^\pi$ is not dominated; specifically, when   $\upsilon <(1-\gamma)/(\gamma^2+ \gamma)$, or equivalently, 
$1-\gamma-\upsilon\gamma^2-\upsilon\gamma>0$.

\begin{theorem}\label{theorem:Q-M-Q-no-sampling-error}
(Absence of Sampling Error) For $(s, a) \in \mathcal{D}$, the learned Q-function $Q_*^\pi(s, a)$, obtained from \eqref{equation:Q-Zhat} without sampling error, and the true Q-function $Q^\pi(s, a)$ satisfy
\begin{align}\label{equation:Q-M-Q-no-sampling-error}
  \|Q_*^\pi(s,a)-Q^\pi(s,a)\|_{\infty} \leq \frac{\gamma \max_{s,a} \mathcal{I}^\pi(s,a)}{1-\gamma-\upsilon\gamma^2-\upsilon\gamma}.
\end{align} 
Moreover, the learned V-function $V_*^\pi(s)$ exactly matches the true V-function $V^\pi(s)$, i.e., $V_*^\pi(s) = V^\pi(s)$.
\end{theorem}

\begin{proof}
Since $Q_*^\pi =(\mathcal{Z}^\pi Q_*^\pi)$, as established in Theorem~\ref{theorem:convergence-operator}, and based on \eqref{equation:Zhat} without sampling error, we have
\begin{align*}
   Q_*^\pi(s, a)  = &  (1-\upsilon)(\mathcal{T}^\pi Q_*^\pi)(s, a) \\& +\upsilon(\mathcal{H}^\pi Q_*^\pi)(s, a)  -\gamma \mathcal{I}^\pi(s,a).
\end{align*}
According to \eqref{equation:Bellman-Q}, \eqref{equation:Bellman-operator} and \eqref{HQk}, we obtain
\begin{align} %
  & \Big|Q_*^\pi(s,a)-Q^\pi(s,a)\Big| \nonumber \\
  = & \Bigg|r(s,a) + \gamma\bE_{s^{\prime}\sim \mathcal{P} (\cdot\mid s, a)}\bigg[(1 - \upsilon)Q_*^\pi(s^{\prime}, \pi(s^{\prime})) \nonumber \\
  & + \upsilon \Big(Q_*^\pi(s^{\prime},\pi(s^{\prime})) -  \big(r(s,a) + \gamma Q_*^\pi(s^{\prime},\pi(s^{\prime})) \nonumber \\
  & - Q_*^\pi(s,\pi(s))\big)\Big) - \mathcal{I}^\pi(s,a)\bigg]  - r(s, a) \nonumber  \\
  & - \gamma \bE_{s^{\prime} \sim \mathcal{P} (\cdot \mid s, a)}\Big[Q^\pi(s^{\prime}, \pi(s^{\prime}))\Big] \Bigg|. \label{Qpi*-Qpi}
\end{align}
The calculation then proceeds as follows:
  \begin{align}
 & \Big|Q_*^\pi(s,a)-Q^\pi(s,a)\Big| \nonumber \\  
  %
  \leq & \gamma \Bigg| \bE_{s^{\prime}\sim \mathcal{P} (\cdot\mid s, a)}\bigg[Q_*^\pi(s^{\prime}, \pi(s^{\prime})) - Q^\pi(s^{\prime}, \pi(s^{\prime})) \bigg] \nonumber \\
  & + \upsilon \bE_{s^{\prime}\sim \mathcal{P} (\cdot\mid s, a)}\bigg[ -\gamma Q_*^\pi(s^{\prime},\pi(s^{\prime})) - r(s,a) \nonumber \\
  & + Q_*^\pi(s,\pi(s)) \bigg] \Bigg| + \gamma \max_{s,a} \mathcal{I}^\pi(s,a) \nonumber \\
 = & \gamma \Bigg| \bE_{s^{\prime}\sim \mathcal{P} (\cdot\mid s, a)}\bigg[Q_*^\pi(s^{\prime}, \pi(s^{\prime})) - Q^\pi(s^{\prime}, \pi(s^{\prime})) \bigg] \nonumber \\
  & + \upsilon \bE_{s^{\prime}\sim \mathcal{P} (\cdot\mid s, a)}\bigg[ -\gamma \big(Q_*^\pi(s^{\prime},\pi(s^{\prime})) -  Q^\pi(s^{\prime}, \pi(s^{\prime})) \big)  \nonumber \\
  & + Q_*^\pi(s,\pi(s)) - Q^\pi(s, a) \bigg] \Bigg| + \gamma \max_{s,a} \mathcal{I}^\pi(s,a) \nonumber \\
  \leq & \gamma \Bigg| \bE_{s^{\prime}\sim \mathcal{P} (\cdot\mid s, a)}\bigg[Q_*^\pi(s^{\prime}, \pi(s^{\prime})) - Q^\pi(s^{\prime}, \pi(s^{\prime})) \bigg] \Bigg| \nonumber \\
  & + \upsilon\gamma \Bigg| \bE_{s^{\prime}\sim \mathcal{P} (\cdot\mid s, a)}\bigg[ -\gamma \big(Q_*^\pi(s^{\prime},\pi(s^{\prime})) -  Q^\pi(s^{\prime}, \pi(s^{\prime})) \big) \bigg] \Bigg| \nonumber \\
  & + \upsilon\gamma \Bigg| Q_*^\pi(s,\pi(s)) - Q^\pi(s, a)  \Bigg| + \gamma \max_{s,a} \mathcal{I}^\pi(s,a) \nonumber \\
  \leq & (\gamma+\upsilon\gamma^2+\upsilon\gamma)\max_{s,a}\Big|Q_*^\pi(s,a)-Q^\pi(s,a)\Big|  \nonumber \\ & + \gamma \max_{s,a} \mathcal{I}^\pi(s,a).
  \label{equation:Q-M-Q-proof}
\end{align}
Under the condition $1 - \gamma - \upsilon \gamma^2 - \upsilon \gamma > 0$, the bound in \eqref{equation:Q-M-Q-no-sampling-error} follows directly from \eqref{equation:Q-M-Q-proof}.

Next, since $\pi$ is a deterministic policy, and using a similar calculation as in \eqref{Qpi*-Qpi}, we have
  \begin{align}
    & V_*^\pi(s) = Q_*^\pi(s, \pi(s))  \nonumber \\
    = & r(s,\pi(s)) + \gamma\bE_{s^{\prime}\sim \mathcal{P} (\cdot\mid s,\pi(s))}\bigg[(1 - \upsilon)V_*^\pi(s^{\prime}) \nonumber \\
    & + \upsilon \Big(V_*^\pi(s^{\prime}) - \big(r(s,\pi(s)) + \gamma V_*^\pi(s^{\prime}) - V_*^\pi(s)\big)\Big) \bigg] \nonumber \\
    = & (1 - \upsilon\gamma) r(s,\pi(s)) + \upsilon\gamma V_*^\pi(s) \nonumber \\
    & + \big(\gamma-\upsilon\gamma^2\big)\bE_{s^{\prime}\sim \mathcal{P} (\cdot\mid s,\pi(s))}\big[V_*^\pi(s^{\prime})\big].
  \end{align}
By rearranging terms and solving for $V_*^\pi(s)$ under the condition $1 - \upsilon \gamma > 0$, we obtain the recursive relationship:
\begin{align}\label{equation:V-M-no-sampling-error}
  V_*^\pi(s) = r(s,\pi(s)) + \gamma \bE_{s^{\prime}\sim \mathcal{P} (\cdot\mid s,\pi(s))}[V_*^\pi(s^{\prime})].
\end{align}
This shows $V_*^\pi(s)$ is the solution of \eqref{equation:Bellman-V}, and hence $V_*^\pi(s)=V^\pi(s)$. The theorem is thus proved.
\end{proof}

In practical offline RL settings where sampling error is present, the actual operators are typically inaccessible. However, when the dataset $\mathcal{D}$ provides sufficient coverage of the state-action space, their empirical counterparts serve as reasonable approximations. In what follows, we conduct an error analysis of the Q-function obtained via MCRE under this empirical scenario.
 
To rigorously characterize the discrepancy between the empirical and true operators, we introduce the following assumption on the deviation between $\hat{\mathcal{P}}(s^{\prime} \mid s, a)$ and $\mathcal{P}(s^{\prime} \mid s, a)$.



\begin{assumption}\label{assumption:state-sampling-error}
  \cite{kumarConservativeQLearningOffline2020, huangMildPolicyEvaluation2024}:
  For any state-action pair $(s, a) \in \mathcal{D}$, the $\ell_1$-norm deviation between the empirical and actual state transition dynamics satisfies the following bound with probability larger than $1-\delta$:  
  \begin{align}\label{equation:state-sampling-error}
  |\hat{\mathcal{P}}(s^{\prime} \mid s, a)-\mathcal{P} (s^{\prime} \mid s, a)| \leq \frac{c_p}{\sqrt{\mathcal{D}_c}}
  \end{align}
  where $\delta \in (0, 1)$, $c_p > 0$ is a constant, and $\mathcal{D}_c \neq 0$ denotes the count of state-action pairs in $\mathcal{D}$. When $\mathcal{D}_c = 0$, the bound holds trivially if $\delta \geq {2 r_{\max}}/{(1 - \gamma)}$.
\end{assumption}

Under Assumption~\ref{assumption:state-sampling-error}, the gap between $(\hat{\mathcal{T}}^\pi \hat{Q}_k)(s,a)$ and $(\mathcal{T}^\pi \hat{Q}_k)(s,a)$ for any given policy $\pi$ satisfies \cite{huangEfficientOfflineReinforcement2024}
\begin{align}\label{equation:That-T}
& \big|(\hat{\mathcal{T}}^\pi \hat{Q}_k)(s, a)-(\mathcal{T}^\pi \hat{Q}_k)(s, a)\big| \nonumber \\
= &  \gamma\big|\sum_{s^{\prime}}(\hat{\mathcal{P}}(s^{\prime} \mid s, a)-\mathcal{P} (s^{\prime} \mid s, a)) \hat{Q}_k(s^{\prime}, \pi(s^{\prime}))\big| \nonumber \\
\leq & \frac{\gamma c_p r_{\max}}{(1-\gamma)\sqrt{\mathcal{D}_c}}.
\end{align}
We now enhance Theorem~\ref{theorem:Q-M-Q-no-sampling-error} to the following version.

\begin{theorem}\label{theorem:Q-M-Q-sampling-error}
(Presence of Sampling Error) For $(s, a) \in \mathcal{D}$, the learned Q-function $\hat{Q}_*^\pi(s, a)$, obtained from \eqref{equation:Q-Zhat} under Assumption~\ref{assumption:state-sampling-error}, and the true Q-function $Q^\pi(s, a)$ satisfy 
\begin{align}\label{equation:Qhat-M-Q-sampling-error-Result}
  \|\hat{Q}_*^\pi(s,a)-Q^\pi(s,a) \|_{\infty} 
  \leq  \frac{\gamma \max_{s,a} \mathcal{I}^\pi(s,a)}{1-\gamma-\upsilon\gamma^2-\upsilon\gamma} \nonumber \\
   + \frac{\gamma c_p r_{\max}}{(1-\gamma-\upsilon\gamma^2-\upsilon\gamma)(1-\gamma) \sqrt{\mathcal{D}_c}}.
\end{align}
Moreover, the learned V-function $\hat{V}_*^\pi(s)$ and the true V-function $V^\pi(s)$ satisfy
\begin{align}\label{equation:V-M-sampling-error-Result}
  \|\hat{V}_*^\pi(s) - V^\pi(s)\|_{\infty} \leq \frac{\gamma c_p r_{\max}}{(1-\gamma)^2\sqrt{\mathcal{D}_c}}.
\end{align}
\end{theorem}

\begin{proof}
The proof is similar to that of Theorem~\ref{theorem:Q-M-Q-no-sampling-error}. 
Specifically, the calculation in \eqref{equation:Q-M-Q-proof} becomes
\begin{align}\label{equation:Q-M-Q-sampling-error}
  & \Big|\hat{Q}_*^\pi(s,a)-Q^\pi(s,a)\Big| \nonumber \\
  \leq & \gamma\Big|\bE_{s^{\prime}\sim \hat{\mathcal{P}}(\cdot\mid s, a)}\big[\hat{Q}_*^\pi(s^{\prime},\pi(s^{\prime}))\big]  - \bE_{s^{\prime}\sim \mathcal{P} (\cdot\mid s, a)}\big[Q^\pi(s^{\prime},\pi(s^{\prime}))\big]\Big| \nonumber \\
  & + \upsilon\gamma \Bigg| \bE_{s^{\prime}\sim \hat{\mathcal{P}} (\cdot\mid s, a)}\bigg[ -\gamma \big(\hat{Q}_*^\pi(s^{\prime},\pi(s^{\prime})) -  Q^\pi(s^{\prime}, \pi(s^{\prime})) \big) \bigg] \Bigg| \nonumber \\
  & + \upsilon\gamma \Bigg| Q_*^\pi(s,\pi(s)) - Q^\pi(s, a)  \Bigg| + \gamma \max_{s,a} \mathcal{I}^\pi(s,a) \nonumber \\
  \leq &  (\gamma+\upsilon\gamma^2+\upsilon\gamma)\max_{s,a}\Big|\hat{Q}_*^\pi(s,a)-Q^\pi(s,a)\Big| \nonumber\\ & 
  + \gamma \max_{s,a} \mathcal{I}^\pi(s,a)   + \gamma\Big|\bE_{s^{\prime}\sim \hat{\mathcal{P}}(\cdot\mid s, a)}[Q^\pi(s^{\prime},\pi(s^{\prime}))] \nonumber \\
  & - \bE_{s^{\prime}\sim \mathcal{P} (\cdot\mid s, a)}[Q^\pi(s^{\prime},\pi(s^{\prime}))]\Big|, 
  \end{align}
where an additional term is incurred compared with \eqref{equation:Q-M-Q-proof}. 
According to \eqref{equation:That-T}, this additional term is bounded as follows:
\begin{align}\label{equation:Qhat-M-Q-sampling-error-distribution}
  &  \gamma\Big|\bE_{s^{\prime}\sim \hat{\mathcal{P}}(\cdot\mid s, a)}[Q^\pi(s^{\prime},\pi(s^{\prime}))]  
    - \bE_{s^{\prime}\sim \mathcal{P} (\cdot\mid s, a)}[Q^\pi(s^{\prime},\pi(s^{\prime}))]\Big| \nonumber \\
  \leq 
  &  \gamma\big|\sum_{s^{\prime}}\big(\hat{\mathcal{P}}(s^{\prime} \mid s,a) - \mathcal{P} (s^{\prime} \mid s, a)\big) Q^\pi(s^{\prime}, \pi(s^{\prime}))\big| \nonumber \\
  \leq &  \frac{\gamma c_p r_{\max}}{(1-\gamma) \sqrt{\mathcal{D}_c}}.
\end{align}
By substituting \eqref{equation:Qhat-M-Q-sampling-error-distribution} into \eqref{equation:Q-M-Q-sampling-error}, we obtain
\begin{align}\label{equation:Qhat-M-Q-sampling-error-final}
  & \max_{s,a} \Big|\hat{Q}_*^\pi(s,a)-Q^\pi(s,a)\Big| \nonumber \\
  \leq & (\gamma+\upsilon\gamma^2+\upsilon\gamma)\max_{s,a}\Big|\hat{Q}_*^\pi(s,a)-Q^\pi(s,a)\Big|   \nonumber \\
  & + \gamma \max_{s,a} \mathcal{I}^\pi(s,a) + \frac{\gamma c_p r_{\max}}{(1-\gamma) \sqrt{\mathcal{D}_c}}.
\end{align}
Under the condition $1 - \gamma - \upsilon \gamma^2 - \upsilon \gamma > 0$, the bound in  \eqref{equation:Qhat-M-Q-sampling-error-Result} follows directly from \eqref{equation:Qhat-M-Q-sampling-error-final}.
 
Next, from \eqref{equation:V-M-no-sampling-error}, we have
  \begin{align}\label{equation:Vhat-M-sampling-error}
    \hat{V}_*^\pi(s) = r(s,\pi(s)) + \gamma \bE_{s^{\prime}\sim \hat{\mathcal{P}} (\cdot\mid s,\pi(s))}[\hat{V}_*^\pi(s^{\prime})].
  \end{align}
Combining \eqref{equation:Vhat-M-sampling-error} and \eqref{equation:Bellman-V} gives
\begin{align}\label{equation:Vhat-M-sampling-error-Result}
  & \Big|\hat{V}_*^\pi(s) - V^\pi(s)\Big| \nonumber \\
  = & \gamma \Bigg| \bE_{s^{\prime}\sim \hat{\mathcal{P}}(\cdot\mid s,\pi(s))}\Big[\hat{V}_*^\pi(s^{\prime})\Big] - \bE_{s^{\prime}\sim \mathcal{P} (\cdot\mid s,\pi(s))}\Big[V^\pi(s^{\prime})\Big] \Bigg|, 
\end{align}
which is established in Theorem~6 of \cite{huangMildPolicyEvaluation2024} and it follows that the bound \eqref{equation:V-M-sampling-error-Result} holds. The theorem is thus proved.
\end{proof}

Theorems~\ref{theorem:Q-M-Q-no-sampling-error} and \ref{theorem:Q-M-Q-sampling-error} characterize the discrepancy between the learned Q-function and the true Q-function under settings with and without sampling error. These results indicate that a stronger behavior cloning term $\mathcal{I}^\pi$, or a larger modification parameter $\upsilon$ in the Bellman operator, both introduced by MCRE, lead to looser bounds on the Q-function error. This highlights an inherent trade-off: while these components are essential for enforcing mild conservatism, their magnitudes must be carefully chosen to avoid compromising the theoretical bound on the discrepancy. Additionally, the theorems also quantify the corresponding gap between the learned and true V-functions under the same conditions.



\subsection{Analysis of the Policy Derived from MCRE}

Sampling error and value function bias inherent in offline actor–critic (AC) with MCRE introduce a performance discrepancy between the learned policy $\hat{\pi}*$ and the optimal policy $\pi$, as defined in \eqref{equation:pi*}. Specifically, $\hat{\pi}*$ and $\pi$ satisfy
\begin{align}
& \hat{\pi}_*=\operatorname{argmax}_\pi \hat{V}_*^\pi(s), \forall s \in \mathcal{D}, \\
& \pi_*=\operatorname{argmax}_\pi V^\pi(s), \forall s \in \mathcal{D} .
\end{align}
The suboptimality of $\hat{\pi}_*$ is defined as $\Pi(\hat{\pi}_*) = \Gamma(\hat{\pi}_*)-\Gamma(\pi_*)$, where $\Gamma(\cdot)$ represents the performance metric defined in \eqref{equation:eta}. This suboptimality can equivalently be rewritten as
\begin{align}\label{equation:expectation-Pi}
\Pi(\hat{\pi}_*)=\bE[V^{\hat{\pi}_*}(s_0)-V^{\pi_*}(s_0)],
\end{align}
where $s_0$ is sampled from the initial state distribution $\rho_0$. To rigorously bound $\Pi(\hat{\pi}_*)$, we begin by introducing the following assumption.

\begin{assumption}\label{assumption:reward} \cite{huangMildPolicyEvaluation2024} The reward function $r(s, a)$ satisfies $|r(s, \hat{a}_1)-r(s, \hat{a}_2)| \leq \ell\|\hat{a}_1-\hat{a}_2\|_{\infty},\; \forall s \in \mathcal{S},\; \hat{a}_1, \hat{a}_2 \in \mathcal{A}$, where $\ell > 0$ is an unknown constant. The action $a$ is bounded by a constant, such that $\|a\|_{\infty} \leq a_{\max}$, $\forall a \in \mathcal{A}$.
\end{assumption}


Based on this assumption, we establish the following two theorems.

\begin{theorem} \label{theorem:pihat*-pi*} (Absence of Sampling Error) 
Using the V-functions in Theorem~\ref{theorem:Q-M-Q-no-sampling-error}, and under Assumption~\ref{assumption:reward}, the performance difference between the learned suboptimal policy $\hat{\pi}_*$ and the optimal policy $\pi_*$ satisfies
\begin{align}\label{equation:pihat*-pi*-no-sampling-error-Result}
  & \|\Pi(\hat{\pi}_*)\|_{\infty} 
  \leq  \frac{2\ell a_{\max}}{1-\gamma} \nonumber \\ & + \frac{2\gamma r_{\max}\max_s\mathrm{TV} \big(\mathcal{P} (\cdot \mid s, \hat{\pi}_*(s)), \mathcal{P} (\cdot \mid s, \pi_*(s))\big)}{(1-\gamma)^2}.
  \end{align}
where $\mathrm{TV}$ denotes the total variation distance between probability distributions \cite{levin2017markov}.
\end{theorem} 

\begin{proof}
By \eqref{equation:expectation-Pi}, we have
\begin{align}\label{equation:Pihat*-pi*}
 \big|\Pi(\hat{\pi}_*)\big|
= & \big|\bE\big[V^{\hat{\pi}_*}(s_0)-V^{\pi_*}(s_0)\big]\big| \nonumber \\
\leq & \max_s\big|V^{\hat{\pi}_*}(s)-V^{\pi_*}(s)\big|
\end{align}
Moreover, 
\begin{align}\label{equation:Pihat*-pi*2}
    & |V^{\hat{\pi}_*}(s)-V^{\pi_*}(s)| \nonumber \\
\leq & |V^{\hat{\pi}_*}_*(s)- V^{\hat{\pi}_*}(s)|+|V^{\hat{\pi}_*}_*(s)-V^{\pi_*}(s)|
\end{align}
According to Theorem~\ref{theorem:Q-M-Q-no-sampling-error}, we have $|V^{\hat{\pi}_*}_*(s) -V^{\hat{\pi}_*}(s)| = 0$.
To further analyze $|V_*^{\hat{\pi}_*}(s)-V^{\pi_*}(s)|$, we proceed with the following calculation:
\begin{align*}
    & \Big|V_*^{\hat{\pi}_*}(s) - V^{\pi_*}(s)\Big| \nonumber \\
  = & \Bigg| r(s,\hat{\pi}_*(s)) -r(s, \pi_*(s))  + \gamma\bE_{s^{\prime}\sim \mathcal{P} (\cdot\mid s,\hat{\pi}_*(s))}\Big[V_*^{\hat{\pi}_*}(s^{\prime})\Big] \nonumber \\
  & -\gamma \bE_{s^{\prime} \sim \mathcal{P} (\cdot \mid s, \pi_*(s))}\Big[V^{\pi_*}(s^{\prime})\Big]\Bigg|.
\end{align*}
Under Assumption~\ref{assumption:reward}, we have
\begin{align}\label{equation:Vhat-M-V}
    & \Big|V_*^{\hat{\pi}_*}(s) - V^{\pi_*}(s)\Big| \nonumber \\
  \leq & 2\ell a_{\max}  + \gamma \Bigg| \bE_{s^{\prime}\sim \mathcal{P} (\cdot\mid s,\hat{\pi}_*(s))}\Big[V_*^{\hat{\pi}_*}(s^{\prime}) - V^{\hat{\pi}_*}(s^{\prime}) \Big]  \nonumber \\
  & +  \bE_{s^{\prime} \sim \mathcal{P} (\cdot \mid s, \hat{\pi}_*(s))}\Big[V^{\hat{\pi}_*}(s^{\prime}) - V^{\pi_*}(s^{\prime})\Big] \nonumber \\ 
  & + \bE_{s^{\prime} \sim \mathcal{P} (\cdot \mid s, \hat{\pi}_*(s))}\Big[V^{\pi_*}(s^{\prime})\Big]  - \bE_{s^{\prime} \sim \mathcal{P} (\cdot \mid s, \pi_*(s))}\Big[V^{\pi_*}(s^{\prime})\Big] \Bigg|   \nonumber \\ 
  \leq & 2\ell a_{\max}  + \gamma \max_s \Bigg| V^{\hat{\pi}_*}(s) - V^{\pi_*}(s) \Bigg|   \nonumber \\ 
  & + \frac{2\gamma r_{\max}\max_s\mathrm{TV} \big(\mathcal{P} (\cdot \mid s, \hat{\pi}_*(s)), \mathcal{P} (\cdot \mid s, \pi_*(s))\big)}{1-\gamma}
\end{align}
Substituting \eqref{equation:Vhat-M-V} into \eqref{equation:Pihat*-pi*2}, and subsequently into \eqref{equation:Pihat*-pi*}, completes the proof of \eqref{equation:pihat*-pi*-no-sampling-error-Result}.
 \end{proof}

\begin{theorem} \label{theorem:pihat*-pi*-sampling-error}
(Presence of Sampling Error) Using the V-functions in Theorem~\ref{theorem:Q-M-Q-sampling-error}, and under Assumption~\ref{assumption:reward}, the performance difference between the learned suboptimal policy $\hat{\pi}_*$ and the optimal policy $\pi_*$ satisfies
\begin{align}\label{equation:pihat*-pi*-sampling-error-Result}
  & \|\Pi(\hat{\pi}_*)\|_{\infty}  
  \leq  \frac{2\ell a_{\max}}{1-\gamma} + \frac{(1+\gamma)\gamma c_p r_{\max}}{(1-\gamma)^3\sqrt{\mathcal{D}_c}} \nonumber\\
  & + \frac{2\gamma r_{\max} \max_s\mathrm{TV} \big(\hat{\mathcal{P}}(\cdot \mid s, \hat{\pi}_*(s)), \mathcal{P} (\cdot \mid s, \pi_*(s))\big)}{(1-\gamma)^2}. 
\end{align}
\end{theorem}  
\begin{proof}  The proof follows a similar structure to that of Theorem~\ref{theorem:pihat*-pi*}. The only difference is that $V^{\hat{\pi}_*}_*$ is replaced by${\hat V}^{\hat{\pi}_*}_*$, and the equality $|{\hat V}^{\hat{\pi}_*}_*(s) -V^{\hat{\pi}_*}(s)| = 0$ no longer holds in the presence of sampling error.  Instead, this term is bounded as shown in \eqref{equation:V-M-sampling-error-Result}. This additional non-zero bound accounts for the extra term in \eqref{equation:pihat*-pi*-sampling-error-Result}, compared to \eqref{equation:pihat*-pi*-no-sampling-error-Result}.
\end{proof}

\section{The MCRQ Algorithm} \label{section:mcrq}

We introduce a novel offline algorithm, MCRQ, which is built upon the off-the-shelf off-policy online RL algorithm TD3 \cite{fujimotoAddressingFunctionApproximation}.  The core idea of MCRQ is to incorporate MCRE into the TD target $y$ under the AC setting.
Similar to TD3, MCRQ  consists of two critic networks, $Q_{\theta_1}$ and $Q_{\theta_2}$, along with an actor network, $\pi_\phi$, where $\theta_1$, $\theta_2$, and $\phi$ represent the network parameters. The associated target critic networks are $Q_{\theta_1^{\prime}}$ and $Q_{\theta_2^{\prime}}$, and the target actor network is $\pi_{\phi^{\prime}}$, where $\theta_1^{\prime}$, $\theta_2^{\prime}$, and $\phi^{\prime}$ represent the target network parameters. 
The action $a^{\prime}$ is defined as
\begin{align}\label{equation:action}
a^{\prime} = \pi_{\phi^{\prime}}(s^{\prime})+\epsilon,
\end{align}
where $\epsilon$ denotes added exploration noise.

The three components of MCRE, as defined in \eqref{equation:Zhat} and under the AC setting, are used to construct the TD target $y$ in MCRQ. The term
\begin{align*}
y_1 = r + \gamma \min\big(Q_{\theta_1^{\prime}}(s^{\prime},a^{\prime}), Q_{\theta_2^{\prime}}(s^{\prime},a^{\prime})\big)
\end{align*}
corresponds to the empirical Bellman operator $(\hat{\mathcal{T}}^\pi \hat{Q}_k)(s,a)$, and adopts the same TD target formulation as TD3. The term
\begin{align*}
y_2 = & r + \gamma\bigg( \max\big(Q_{\theta_1^{\prime}}(s^{\prime},a^{\prime}), Q_{\theta_2^{\prime}}(s^{\prime},a^{\prime}) \big) \nonumber \\
& - \Big(r + \gamma \min\big(Q_{\theta_1^{\prime}}(s^{\prime},a^{\prime}), Q_{\theta_2^{\prime}}(s^{\prime},a^{\prime})\big) \nonumber \\
& - \max\big(Q_{\theta_1}(s,\pi_\phi(s)), Q_{\theta_2}(s,\pi_\phi(s))\big)\Big) \bigg)
\end{align*}
corresponds to the empirical TD Bellman operator $(\hat{\mathcal{H}}^\pi \hat{Q}_k)(s,a)$. Here, the first $\max$ operator is used to optimistically estimate $Q_{\theta_i^{\prime}}$ at $s^{\prime}$, for $i = 1, 2$, while the second $\max$ estimates $Q_{\theta_i}$ at $s$. Both are used to provide a balanced evaluation against the $\min$ operator. The behavior cloning term is calculated as $\mathcal{I}^\pi = \omega (\pi_\phi(s) - a)^2$, which  constrains the policy $\pi_\phi(s)$ to stay close to the behavior policy. This constraint helps mitigate the effects of distributional shift and OOD actions, while ensuring that $Q_{\theta_i}$ remains supported by the offline dataset.

The TD target $y$ of MCRQ is formally defined as
\begin{align}\label{equation:TD-target}
  y = (1-\upsilon) y_1 + \upsilon y_2 - \gamma\mathcal{I}^\pi.
\end{align}
The loss function for each critic network, parameterized by $\theta_i$, is defined as
\begin{align}\label{equation:Q-loss}
\mathbb{L}(\theta_i) = \bE[y - Q_{\theta_i}(s,a)]^2.
\end{align}   

Through Theorems~\ref{theorem:Q-M-Q-no-sampling-error} and \ref{theorem:Q-M-Q-sampling-error}, we can drive $\pi_\phi(s)$ to approximate the actions in $\mathcal{D}$, thereby reducing the gap between the learned Q-function and the true one. Consequently, the difference between the action taken by the policy $\pi_\phi(s)$ and the selected action $a$ is treated as a penalty. Thus, the loss function for the actor network $\pi_\phi(s)$ is given by
\begin{align}\label{equation:actor-loss}
J(\phi) = -\bE[\lambda Q_{\theta_1}(s,\pi_\phi(s)) - (\pi_\phi(s) - a)^2]
\end{align}
where $\lambda = \frac{\alpha}{{N^{-1}} \sum_{{(s, a)} \in \mathcal{D}} | Q_{\theta_1}(s, a) |}$ balances the contributions of Q-function and policy penalty for $\alpha \geq 0$.
The MCRQ algorithm is summarized in Algorithm~\ref{algorithm:MCRQ}. A structural comparison between MCRQ and baseline algorithms is discussed below.

BCQ, TD3\_BC, CQL, IQL, and MCRQ each employ two critic networks with corresponding target networks. Among them, TD3\_BC, CQL, IQL, and MCRQ  incorporate an actor network and its associated target network. In contrast, BCQ uses a perturbation network, rather than directly generating actions, which is also accompanied by a target network.
It models the behavior policy with a variational auto-encoder (VAE), generating $n$ perturbed candidate actions. These perturbed actions are used to compute the TD target. A feature of BCQ is its convex combination of the two Q-network outputs, weighting the minimum value more heavily. The actor is updated using perturbed VAE-generated actions.

TD3\_BC utilizes clipped double Q-learning (CDQ) for critic updates, whereas MCRQ incorporates CDQ and additionally introduces a maximization operator. While MCRQ shares the same actor loss as TD3\_BC, it further mitigates distributional shift and reduces the discrepancy between the learned and true Q-functions.

CQL minimizes the maximum Q-value produced by the target policy while maximizing the Q-value derived from the behavior policy. It also regularizes the target policy to stay close to the behavior policy, reducing the impact of OOD actions. The actor network is updated using the soft actor-critic methodology.

IQL is a SARSA-style in-sample algorithm that computes the TD target using actions within the data distribution. It uses expectile regression to address the limitations of the traditional maximization operator. As large target values may stem from lucky transitions rather than a single optimal action, the value network is trained with expectile regression for TD updates. The actor is optimized via policy extraction.

In contrast to BCQ, which perturbs VAE-sampled actions, MCRQ refines critic updates directly, avoiding a perturbation model. Unlike TD3\_BC, which relies solely on a behavior cloning penalty, MCRQ explicitly reduces the gap between the learned and true Q-functions.  Compared to CQL, which strictly minimizes the maximum Q-value of the target policy while constraining it close to the behavior policy, MCRQ strikes a balance between conservatism and value learning through TD-guided updates. MCRQ differs significantly from IQL, which computes TD target using in-sample action from the dataset via expectile regression.

\begin{algorithm}[tb]
  \caption{\underline{M}ildly \underline{C}onservative \underline{R}egularized \underline{Q}-learning (MCRQ)}
  \label{algorithm:MCRQ}
  \begin{algorithmic}[1] 
  \STATE Initialize critic networks $Q_{\theta_1}, Q_{\theta_2}$ and actor network $\pi_\phi$ with random parameters $\theta_1, \theta_2, \phi$, and initialize target networks with parameters $\theta_1^{\prime}$, $\theta_2^{\prime}$, $\phi^{\prime}$.
  \STATE Initialize target network smoothing parameter $\tau$ and actor network update frequency $d$.
  \STATE Initialize offline dataset $\mathcal{D}$.
  \FOR{$t$ = 1 to $T$}
  \STATE Sample $\{(s_{t},a_{t},r_{t},s_{t+1}^{\prime})\}_{t=1}^{N}\sim\mathcal{D}$
  \STATE Calculate $a^{\prime}$ with \eqref{equation:action}
  \STATE Calculate $y$ with \eqref{equation:TD-target}
  \STATE Update critic $\theta_i$ by minimizing \eqref{equation:Q-loss}
  \IF{$t$ mod $d$}
  \STATE Update actor $\phi$ by minimizing \eqref{equation:actor-loss}
  \STATE Update target networks: \\ $\theta_i^{\prime}\leftarrow\tau \theta_i + (1 - \tau) \theta_i^{\prime}$ \\ $\phi^{\prime} \leftarrow \tau \phi + (1 - \tau) \phi^{\prime}$
  \ENDIF
  \ENDFOR
  \end{algorithmic}
  \end{algorithm}

\section{Experiments} \label{section:experiments} 

To evaluate the performance of MCRQ, we conduct experiments using datasets for deep data-driven RL (D4RL) \cite{fu2020d4rl} on three MuJoCo \cite{todorovMuJoCoPhysicsEngine2012} benchmark tasks from OpenAI Gym \cite{brockmanOpenAIGym2016}: HalfCheetah, Hopper, and Walker2d. This benchmark includes five distinct dataset categories: random (r), medium (m), medium-replay (m-r), medium-expert (m-e), and expert (e), resulting in a comprehensive set of 15 datasets.
We compare MCRQ against several recent and competitive baseline algorithms,  including BEAR \cite{kumarStabilizingOffPolicyQLearning2019}, UWAC \cite{wuUncertaintyWeightedActorCritic2021}, BC, CDC \cite{Fakoor2021ContinuousDoublyConstrained}, AWAC \cite{nair2020awac}, BCQ \cite{fujimotoOffPolicyDeepReinforcement2019}, OneStep \cite{brandfonbrener2021offline}, TD3\_BC \cite{Fujimoto2021td3bc}, CQL \cite{kumarConservativeQLearningOffline2020}, IQL \cite{Kostrikovimplicitq-learning2021}, and PBRL \cite{bai2022pessimistic}. These methods represent a diverse set of paradigms in offline RL.




The experimental results for BCQ, TD3\_BC, CQL, and IQL are obtained using the original implementations provided by the respective authors, and the implementation of BC is adapted from the TD3\_BC codebase. Additionally, the results for BEAR, UWAC, AWAC, and OneStep are taken from Table~1 in \cite{mao2023SupportedValueRegularization}. The results for CDC and PBRL are sourced from Table~1 in their respective original papers \cite{Fakoor2021ContinuousDoublyConstrained} and \cite{bai2022pessimistic}. The hyperparameter settings for MCRQ follow those used in TD3\_BC, except for $\upsilon$, $\omega$, and $\alpha$. The values of these parameters are listed in Table~\ref{table:hyperparameters}.

\begin{table}[h!] 
  \caption{Hyperparameter setup.
  \label{table:hyperparameters}}
  \centering
  \small
    \resizebox{\columnwidth}{!}{%
  \begin{tabular}{cccccccccc}
    \toprule
    Dataset & \multicolumn{3}{c}{HalfCheetah}  & \multicolumn{3}{c}{Hopper} & \multicolumn{3}{c}{Walker2d} \\
    & $\upsilon$ & $\omega$ & $\alpha$ & $\upsilon$ & $\omega$ & $\alpha$ & $\upsilon$ & $\omega$ & $\alpha$ \\
    \midrule
    r    & 0.0 & 2.5 & 25.0 & 0.0 & 0.0 & 20.0 & 0.3 & 2.0 & 15.0 \\
    m    & 0.0 & 0.0 & 25.0 & 0.0 & 2.0 & 10.0 & 0.0 & 1.0 & 5.0 \\
    m-r  & 0.1 & 0.0 & 25.0 & 0.0 & 1.0 & 20.0 & 0.0 & 2.0 & 10.0 \\
    m-e  & 0.2 & 2.0 & 2.5  & 0.0 & 2.0 & 2.5  & 0.0 & 1.0 & 5.0 \\
    e    & 0.2 & 0.5 & 2.5  & 0.3 & 1.5 & 2.5  & 0.0 & 2.5 & 5.0 \\
    \bottomrule
  \end{tabular}}
\end{table}


\subsection{Results on D4RL Datasets}
\label{subsection:results-on-d4rl-datasets}

We first evaluate the performance of MCRQ by comparing it with baseline algorithms on the D4RL datasets. Experiments were conducted using five random seeds and network initializations to ensure fair comparisons. Default environment settings were used to maintain a level playing field.
Each dataset was run for 1 million steps in the offline RL setting. To assess policy performance, online evaluations were performed every 5,000 steps by interacting with the environment to collect real-time rewards. During each evaluation, the agent was tested over 10 episodes with different random seeds, and the average reward across these episodes was reported as the final score.

The experimental results are shown in Fig.~\ref{fig:D4RL_comparison}, which compares BC, BCQ, TD3\_BC, CQL, IQL, and MCRQ. Solid curves represent the mean performance across evaluations, while shaded regions indicate one standard deviation over five runs. A sliding window of five was applied to smooth the curves.
Table~\ref{tabular:D4RL_comparison} presents the normalized average scores from the last ten evaluations, comparing baselines with MCRQ. Row-wise normalization is applied so that the 15 scores for each algorithm are scaled between 0 and 1. We summarize our key observations below. 

1) Random datasets: MCRQ outperforms most algorithms by a significant margin, and by a slight margin compared to CDC on halfcheetah-random, PBRL on hopper-random, and BEAR on walker2d-random. As shown in Fig.~\ref{fig:D4RL_comparison}, MCRQ demonstrates rapid improvement and ultimately achieves the highest performance in (a) halfcheetah-random and (f) hopper-random. While PBRL and CDC use critic ensembles for uncertainty estimation, MCRQ relies on double critics, offering substantially higher computational efficiency than PBRL.
 
2) Medium datasets: MCRQ outperforms all algorithms on halfcheetah-medium and hopper-medium, with slightly lower performance than PBRL on walker2d-medium. As shown in Fig.~\ref{fig:D4RL_comparison}, MCRQ exhibits rapid performance gains in (b) halfcheetah-medium and (g) hopper-medium, ultimately maintaining the highest performance among all methods. 

3) Medium-replay datasets: MCRQ achieves the best performance on halfcheetah-medium-replay and walker2d-medium-replay. On hopper-medium-replay, its performance is comparable to IQL but lower than OneStep, CQL, and PBRL. As shown in Fig.~\ref{fig:D4RL_comparison}, MCRQ shows rapid improvement in (c) halfcheetah-medium-replay and ultimately maintains the highest score among all algorithms.

4) Medium-expert datasets: MCRQ achieves performance comparable to BCQ on halfcheetah-medium-expert and performs competitively on hopper-medium-expert and walker2d-medium-expert.

5) Expert datasets: On halfcheetah-expert, MCRQ ranks second, just behind TD3\_BC. On hopper-expert, its performance is slightly lower than UWAC, TD3\_BC, and PBRL. On walker2d-expert, MCRQ performs slightly below IQL.

Fig.~\ref{fig:D4RL_comparison_normalized} presents a box plot based on 15 normalized average scores for each algorithm, derived from Table~\ref{tabular:D4RL_comparison}. In this figure, orange lines denote the mean, with the upper quartile defined as the mean plus the variance and the lower quartile as the mean minus the variance. The horizontal lines at the bottom and top correspond to the minimum and maximum normalized average scores, respectively, while the whiskers extend to these extremes.
It statistically shows that MCRQ achieves the highest mean among all evaluated algorithms, along with the narrowest box, indicating the smallest variance. Furthermore, MCRQ exhibits the shortest whiskers and the highest minimum score, demonstrating its strong competitiveness.

\begin{figure*}
  \centering
      \subcaptionbox{}{\includegraphics[width = 0.32\textwidth]{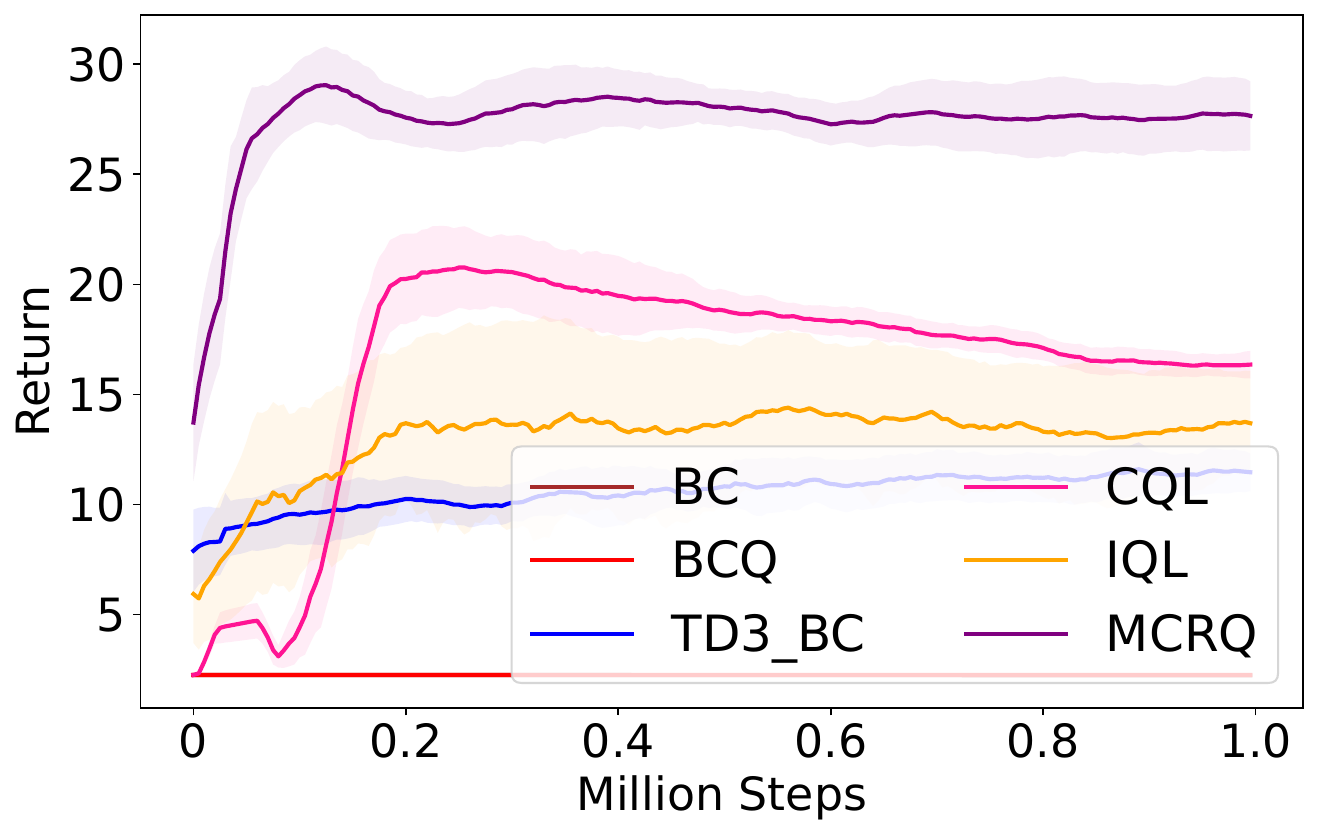}}
      \hfill
      \subcaptionbox{}{\includegraphics[width = 0.32\textwidth]{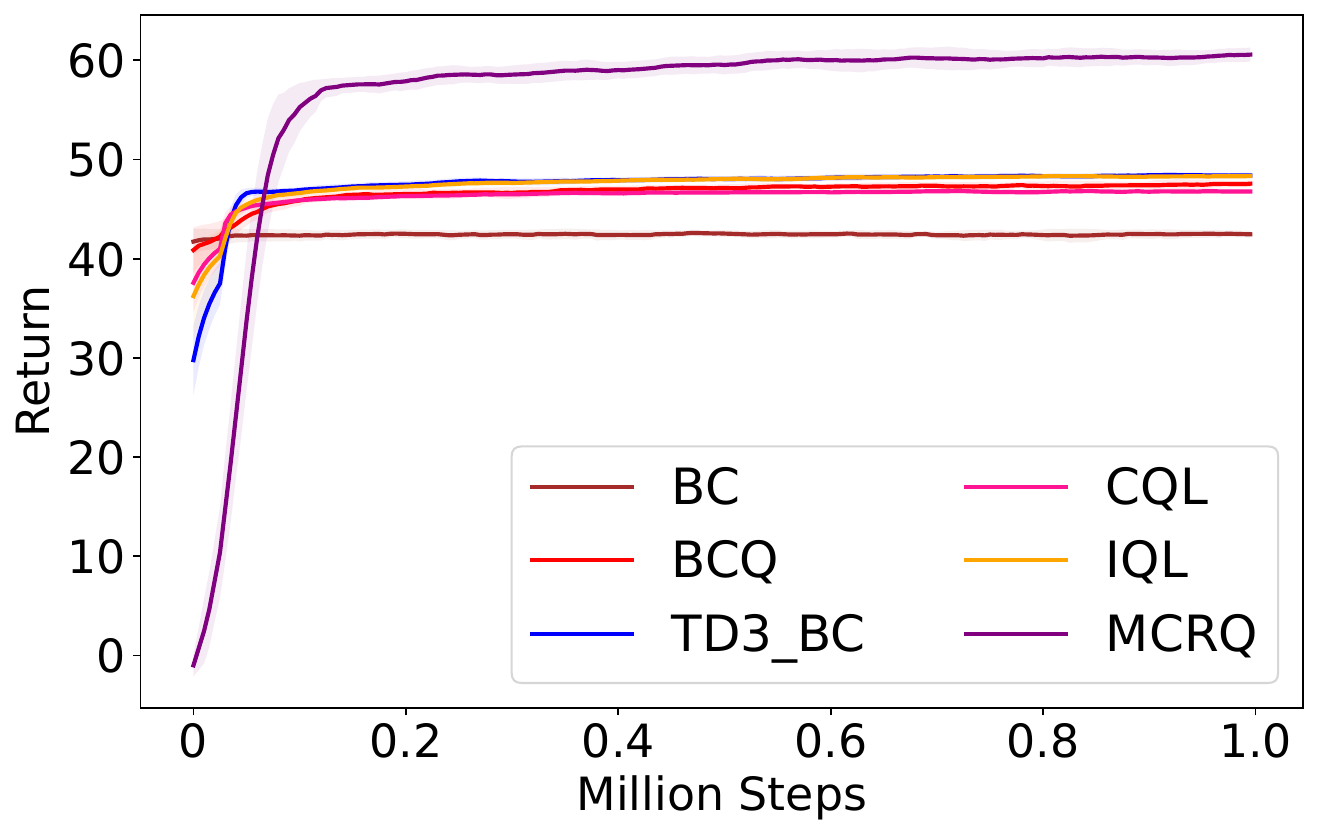}}
      \hfill
      \subcaptionbox{}{\includegraphics[width = 0.32\textwidth]{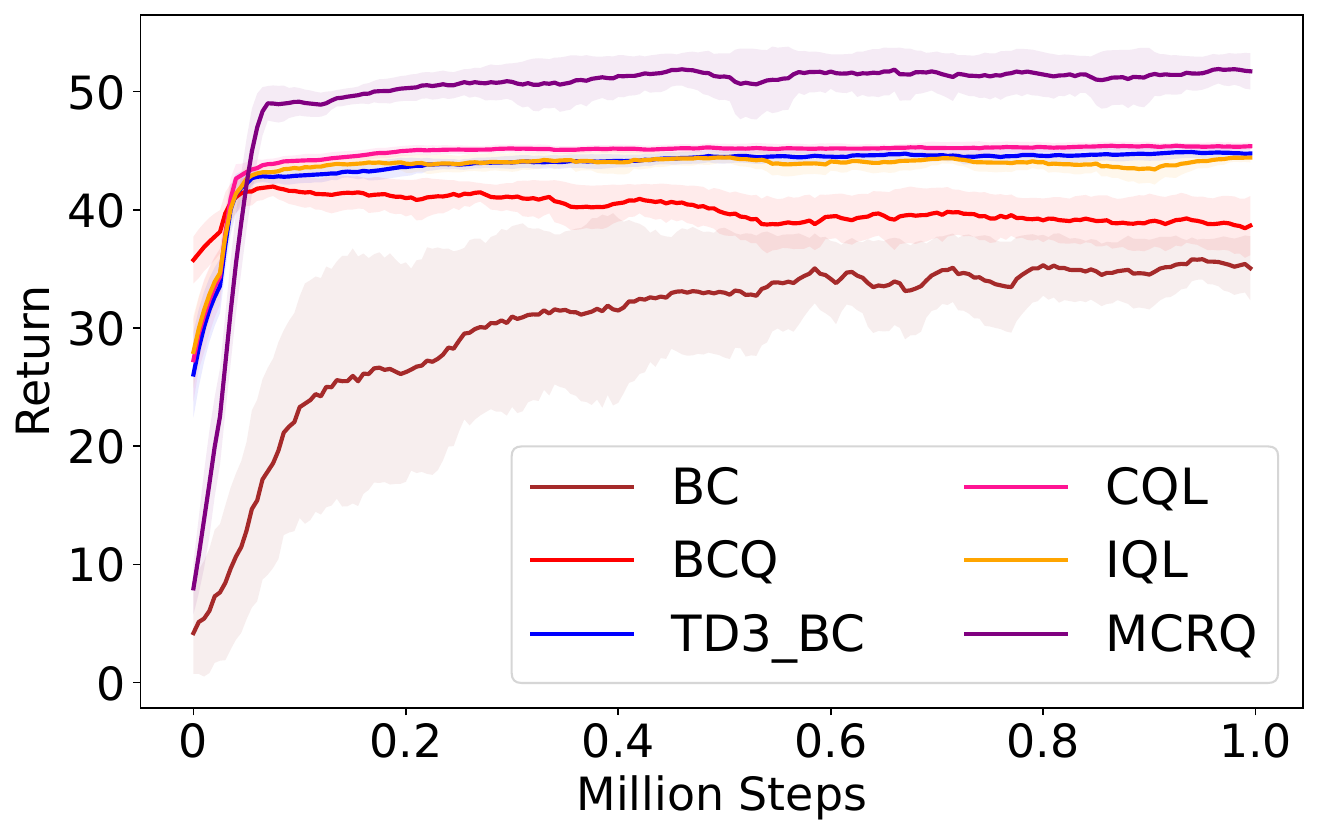}}
      \hfill
      \subcaptionbox{}{\includegraphics[width = 0.32\textwidth]{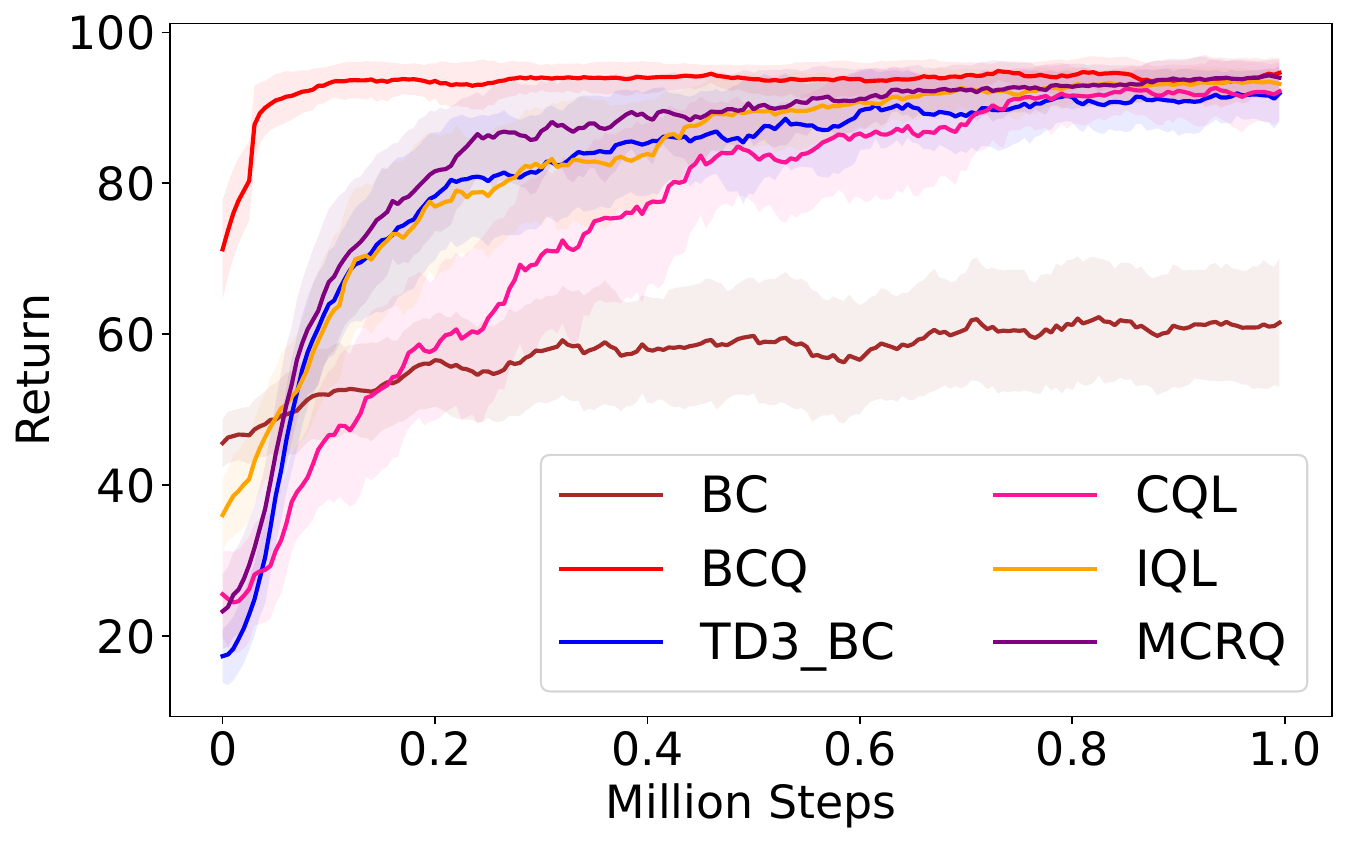}}
      \hfill
      \subcaptionbox{}{\includegraphics[width = 0.32\textwidth]{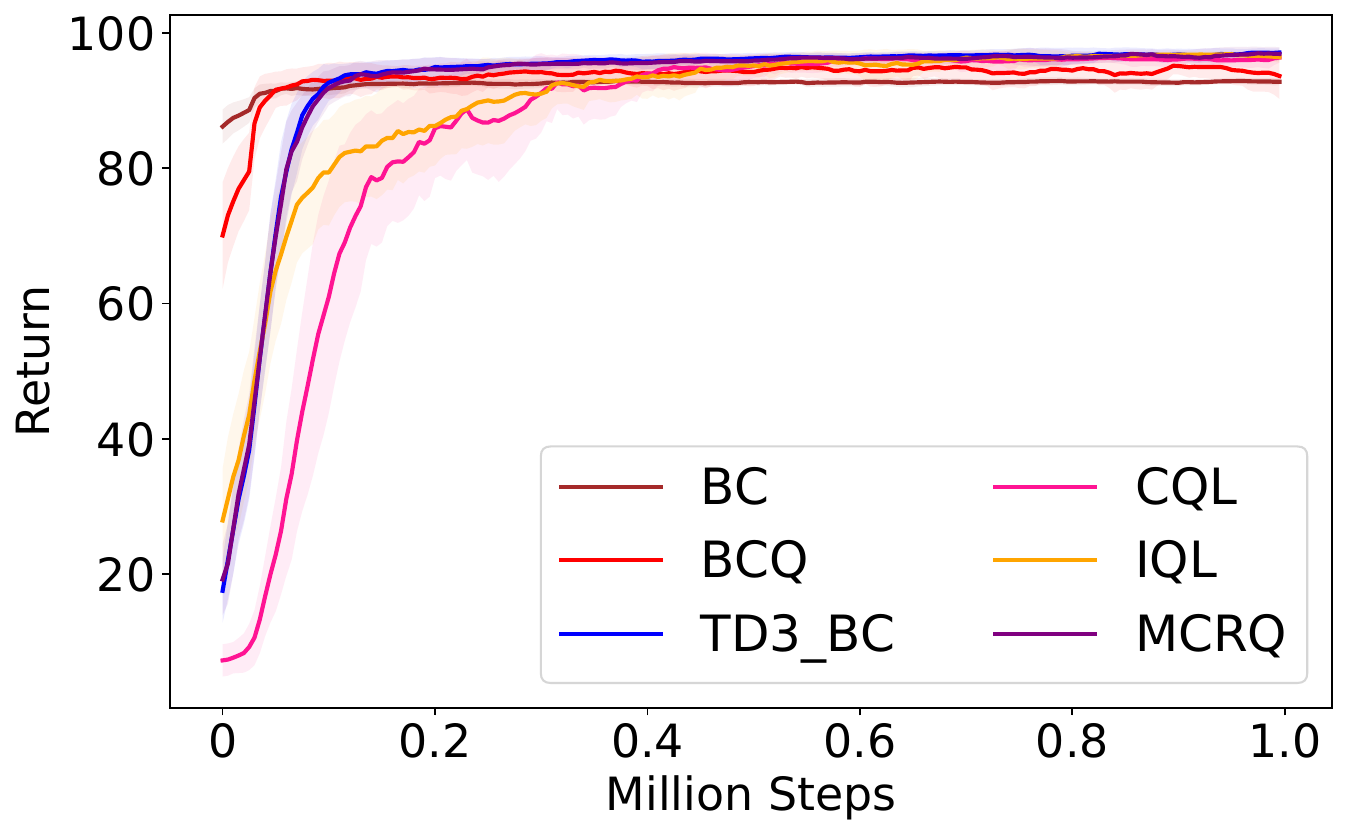}}
      \hfill
      \subcaptionbox{}{\includegraphics[width = 0.32\textwidth]{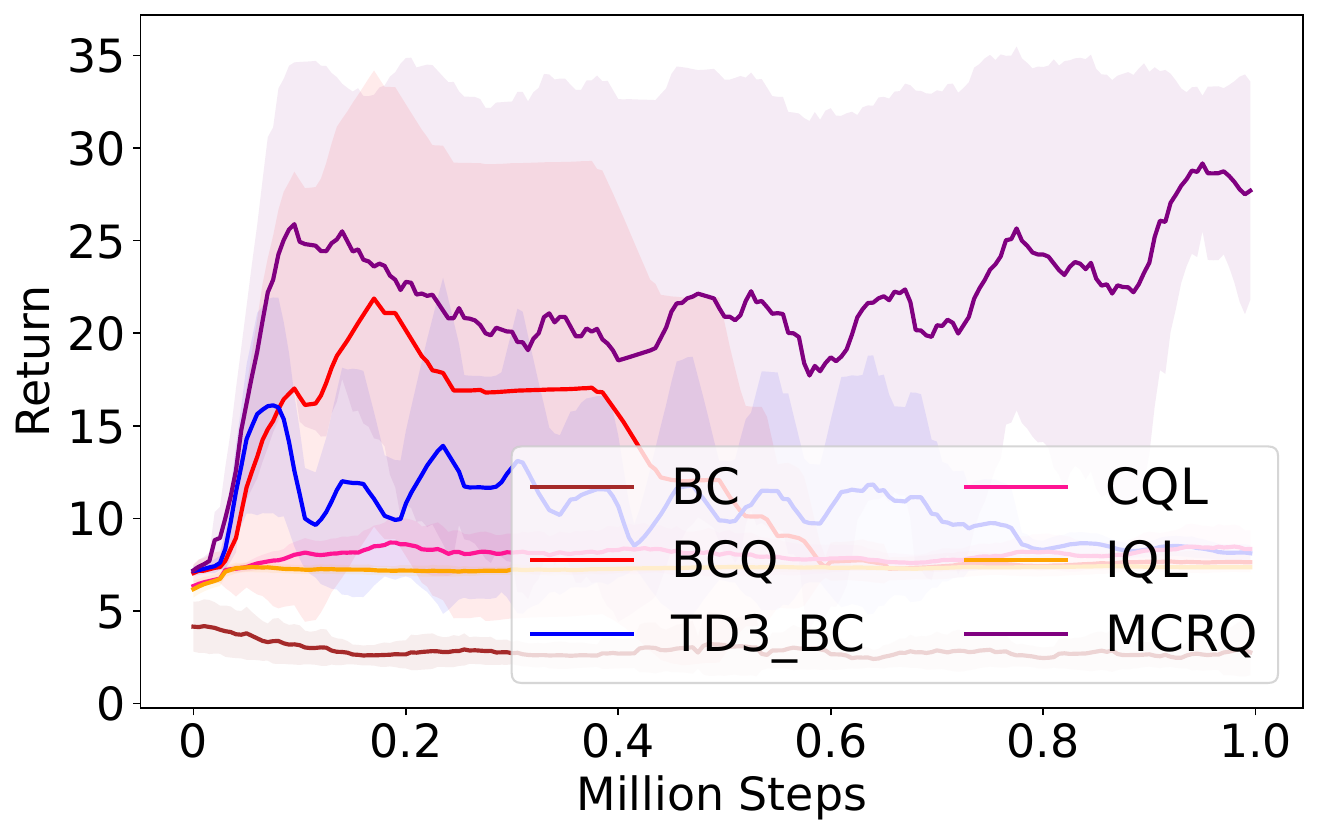}}
      \hfill
      \subcaptionbox{}{\includegraphics[width = 0.32\textwidth]{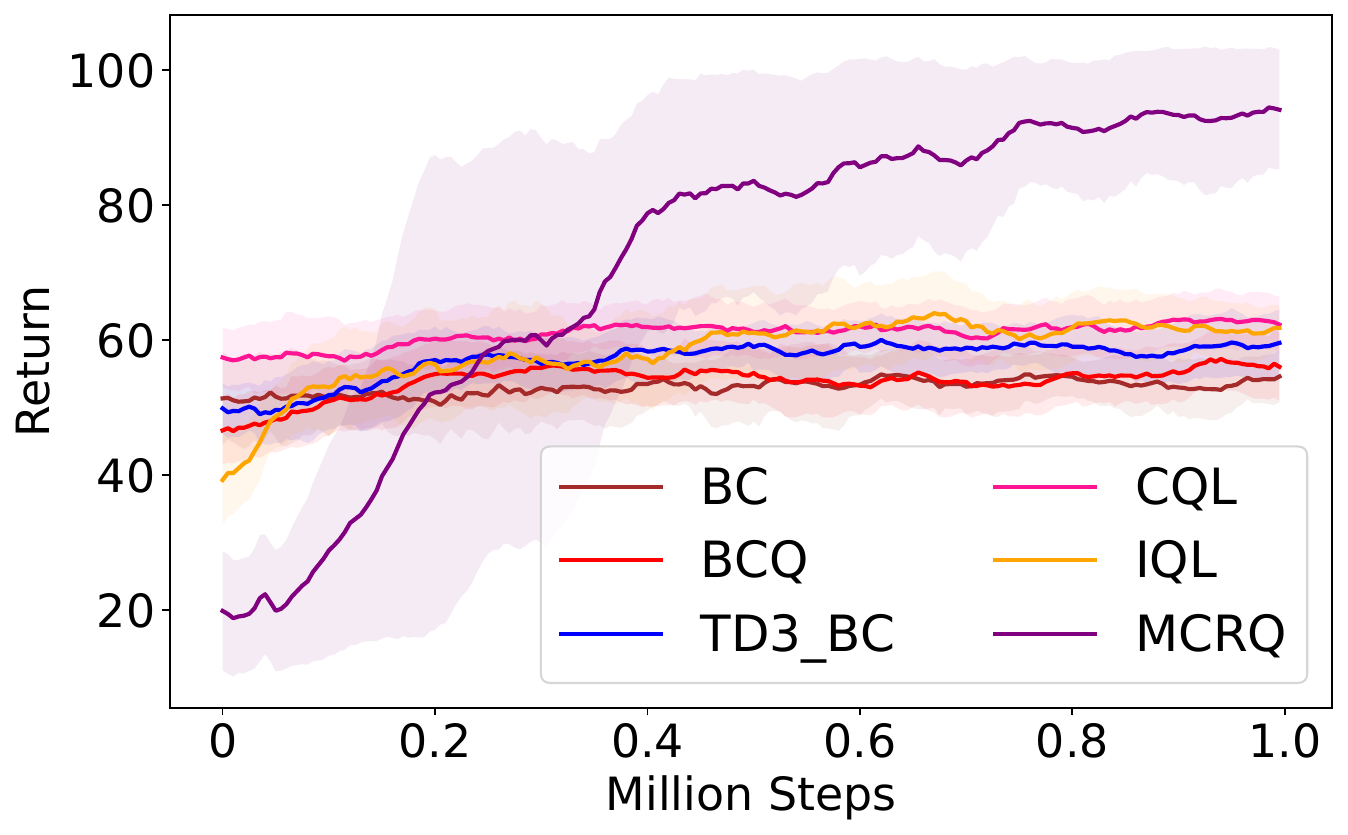}}
      \hfill
      \subcaptionbox{}{\includegraphics[width = 0.32\textwidth]{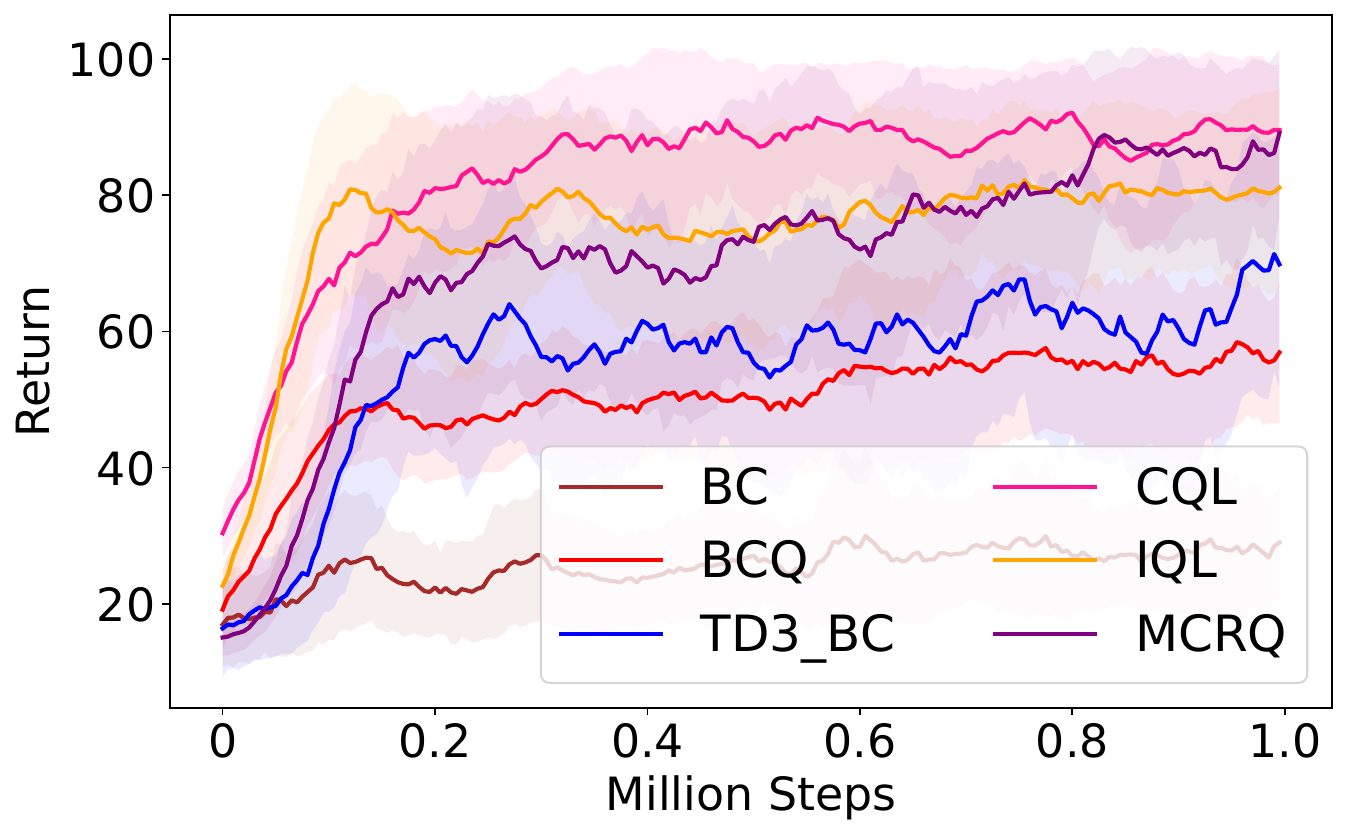}}
      \hfill
      \subcaptionbox{}{\includegraphics[width = 0.32\textwidth]{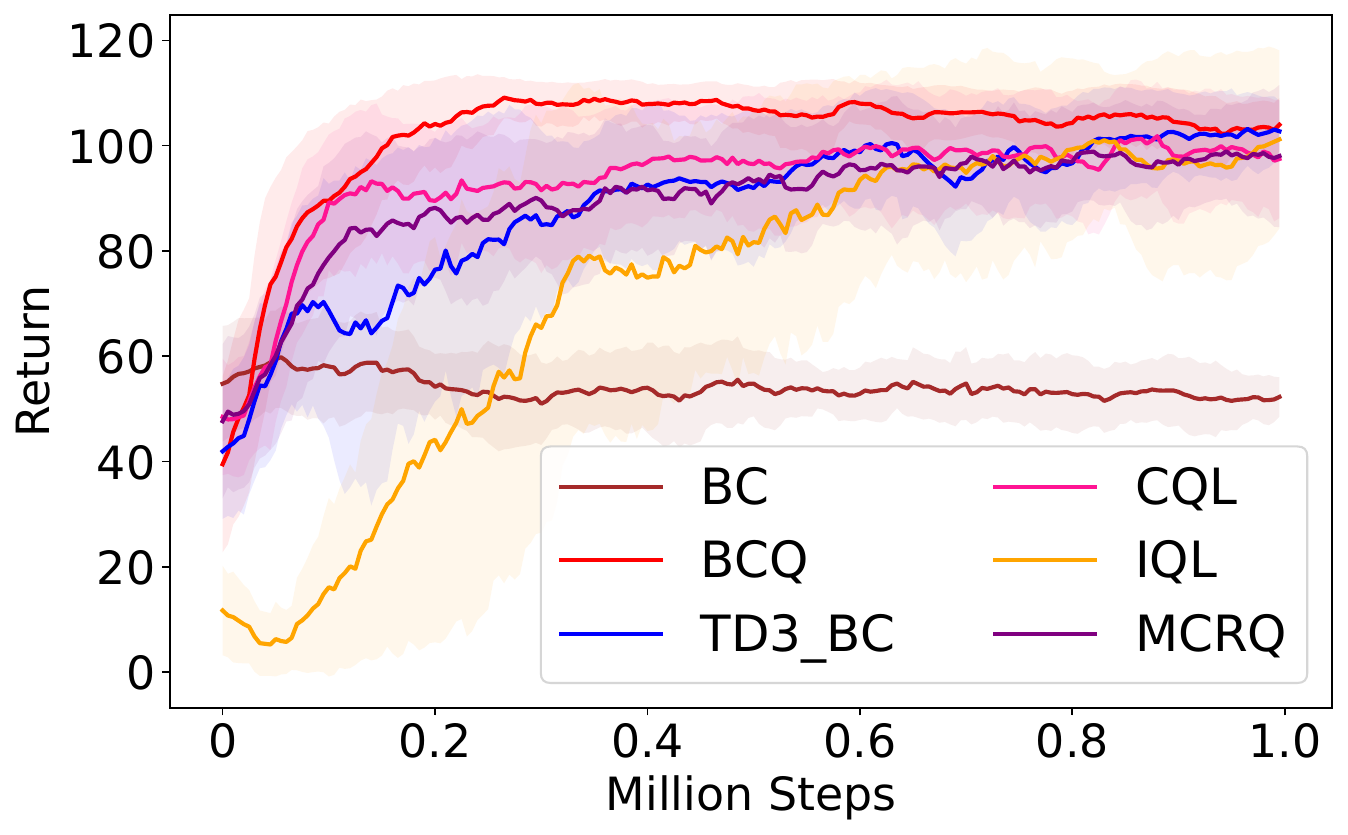}}
      \hfill
      \subcaptionbox{}{\includegraphics[width = 0.32\textwidth]{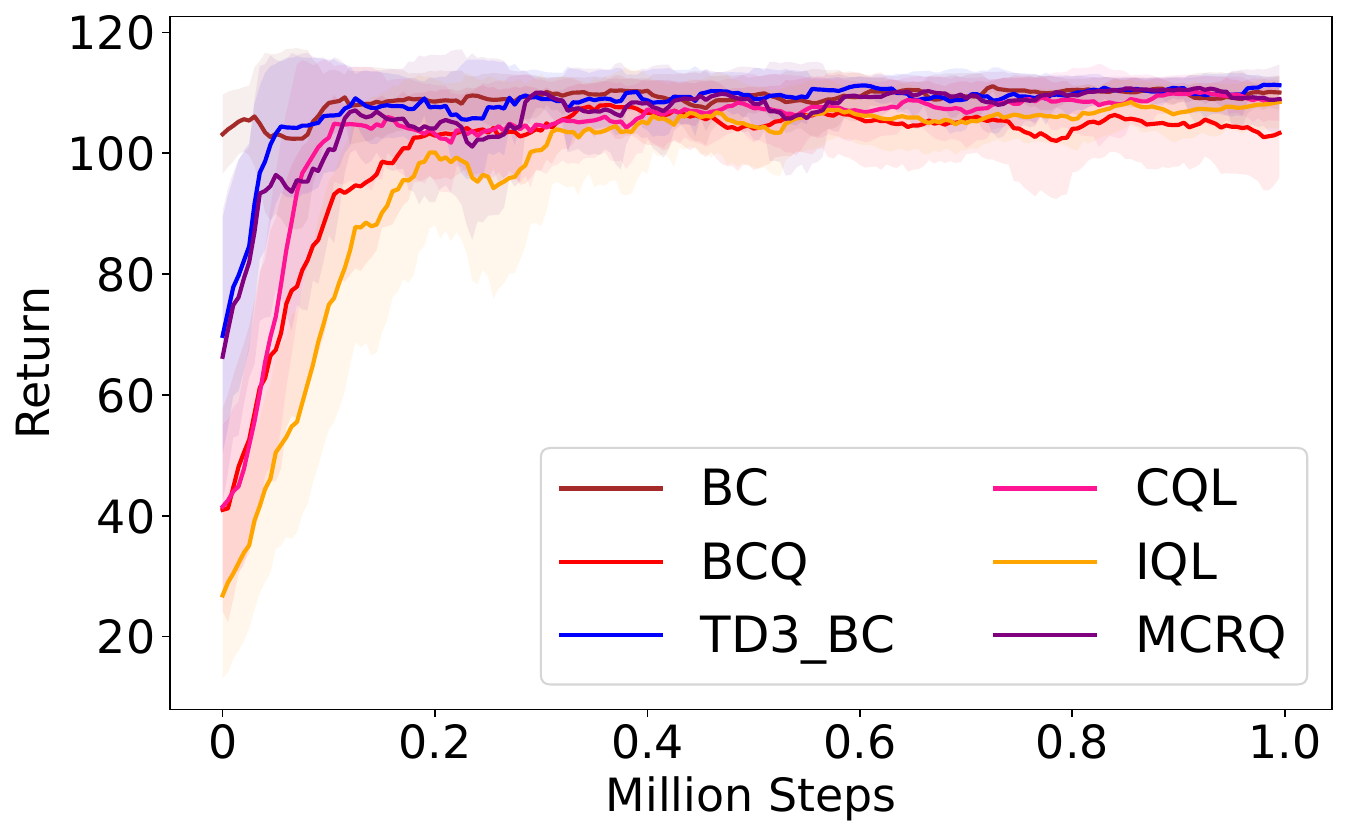}}
      \hfill
      \subcaptionbox{}{\includegraphics[width = 0.32\textwidth]{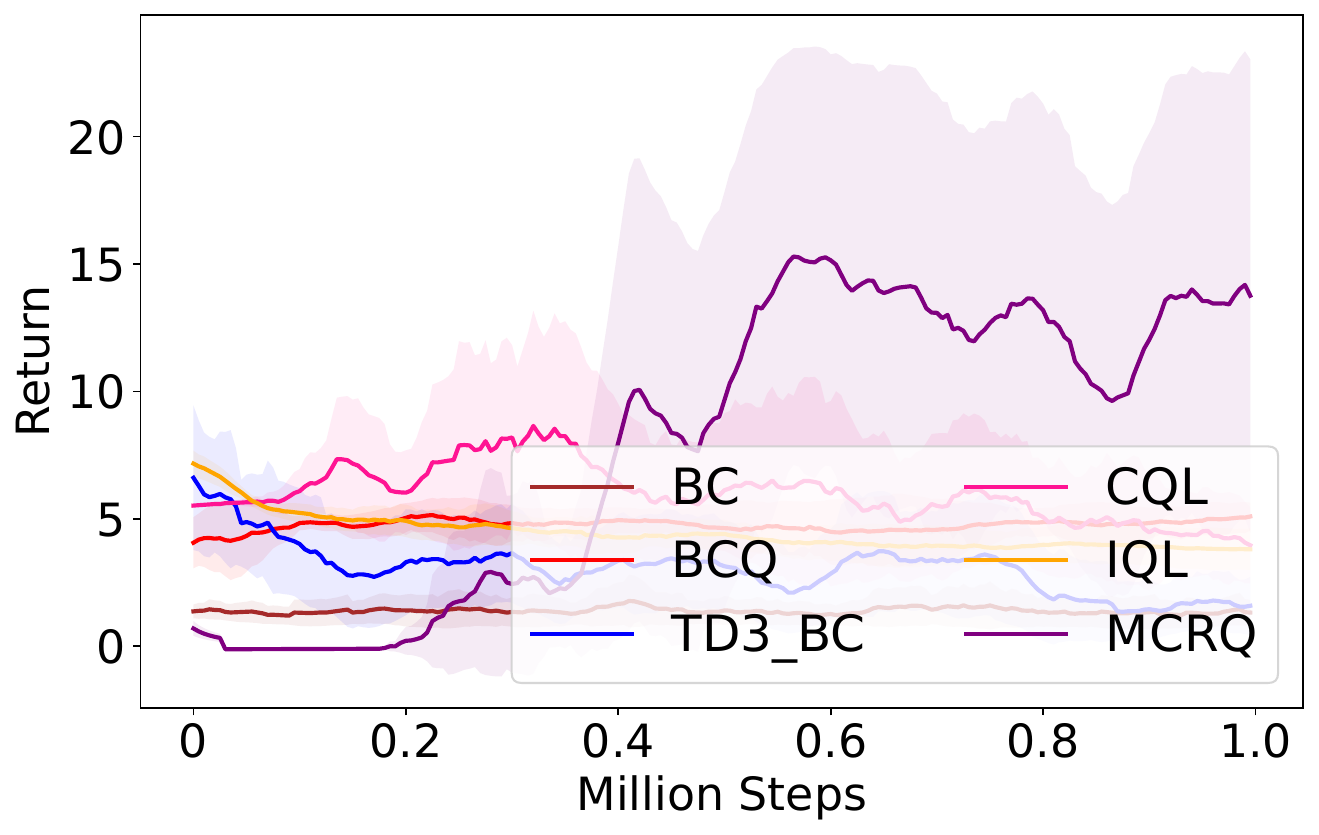}}
      \hfill
      \subcaptionbox{}{\includegraphics[width = 0.32\textwidth]{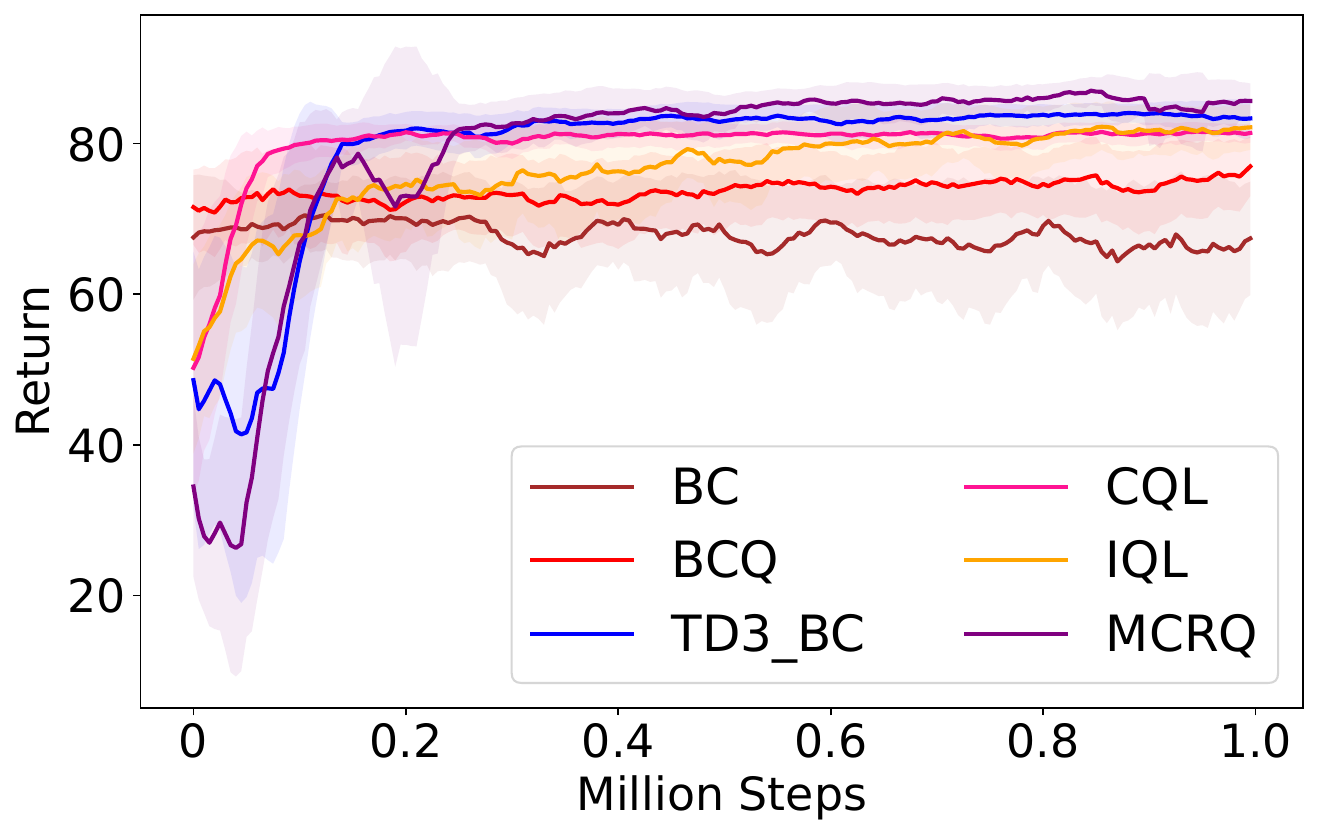}}
      \hfill
      \subcaptionbox{}{\includegraphics[width = 0.32\textwidth]{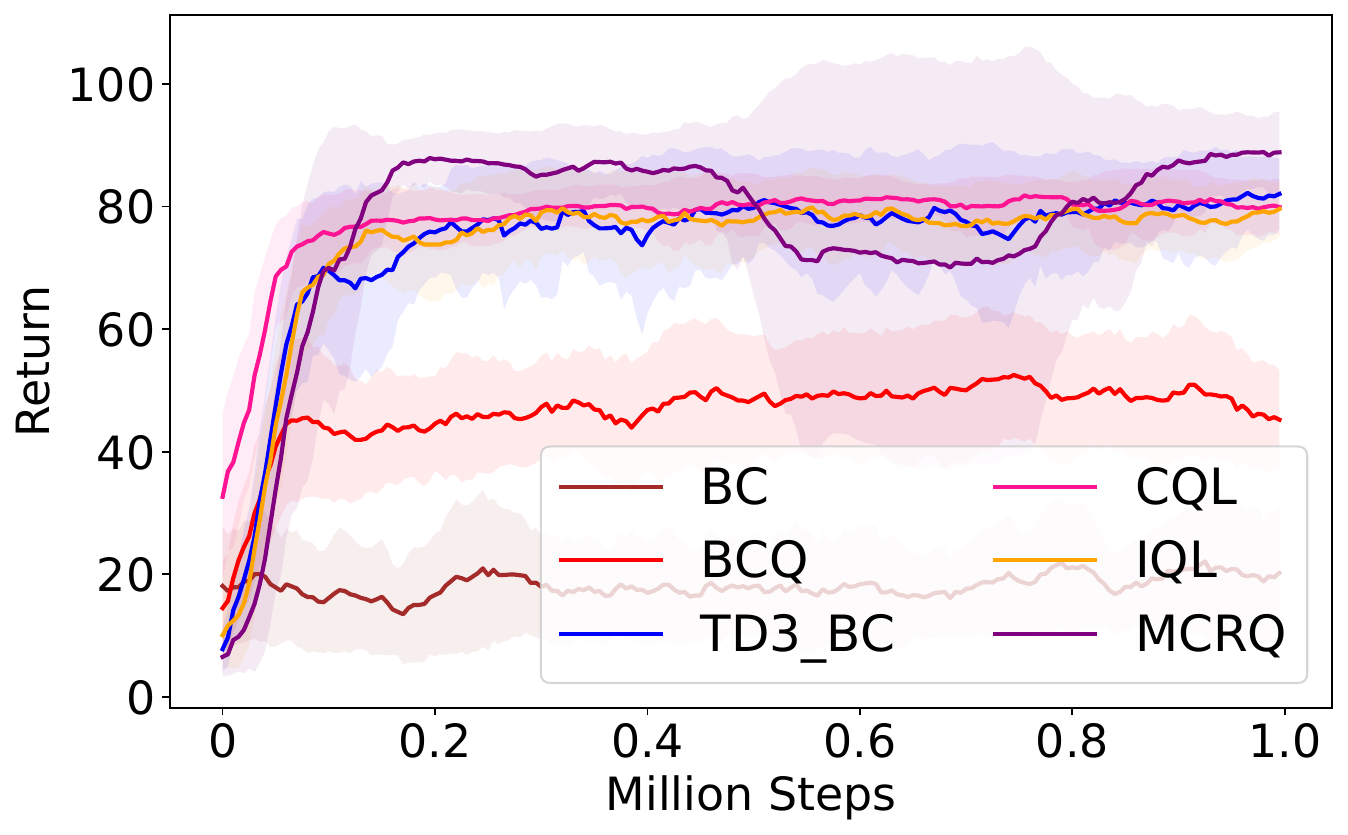}}
      \hfill
      \subcaptionbox{}{\includegraphics[width = 0.32\textwidth]{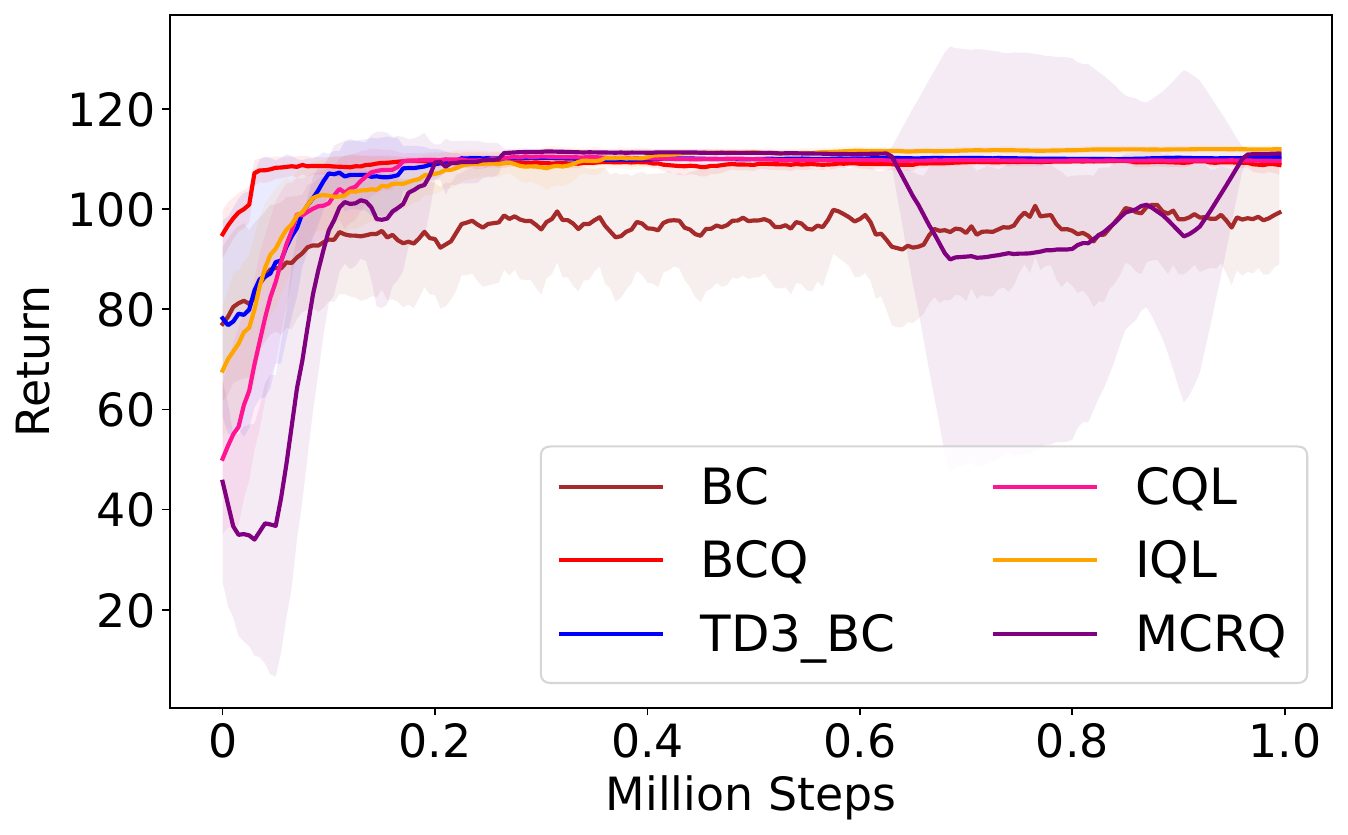}}
      \hfill
      \subcaptionbox{}{\includegraphics[width = 0.32\textwidth]{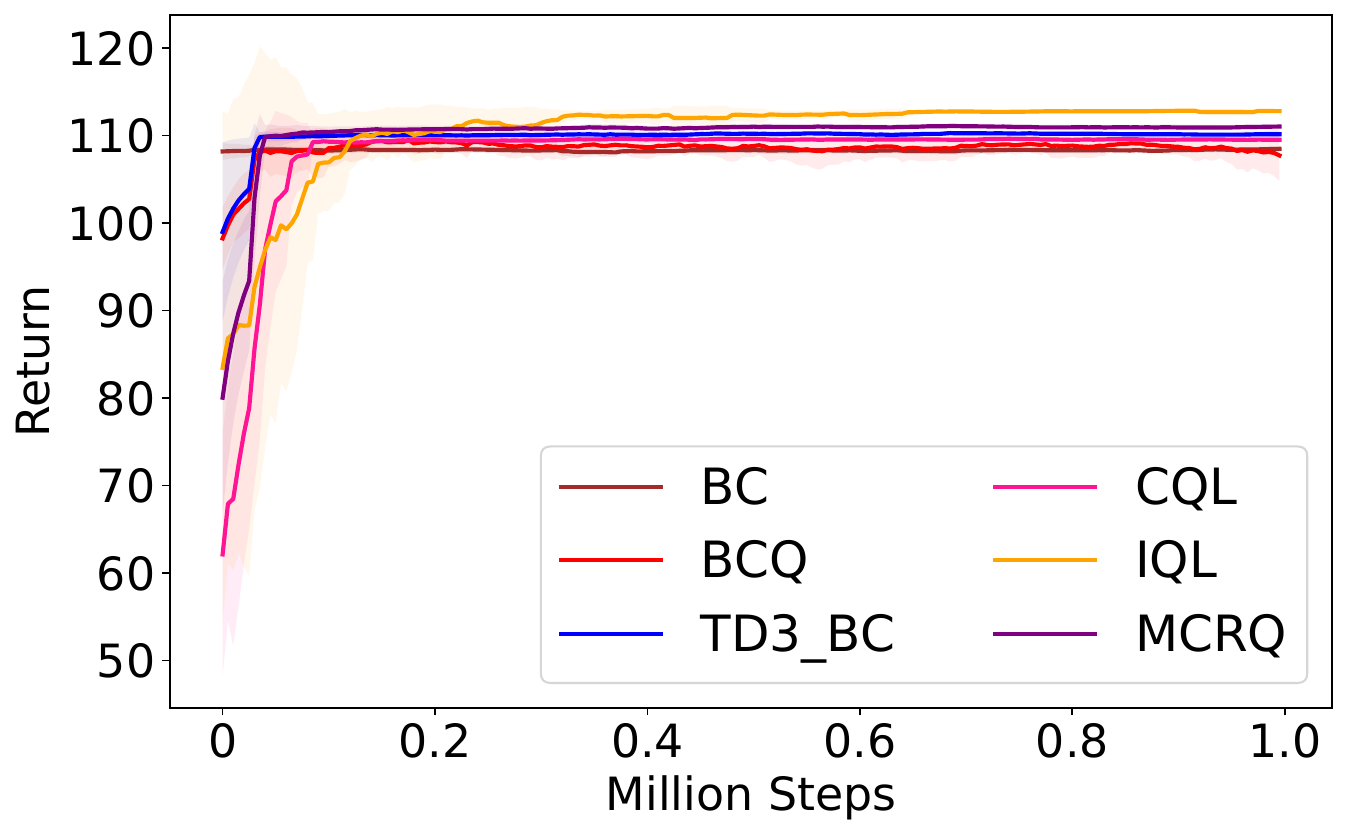}}
  \caption{Normalized score curves across D4RL benchmarks. (a) halfcheetah-random-v2, (b) halfcheetah-medium-v2, (c) halfcheetah-medium-replay-v2, (d) halfcheetah-medium-expert-v2, (e) halfcheetah-expert-v2, (f) hopper-random-v2, (g) hopper-medium-v2, (h) hopper-medium-replay-v2, (i) hopper-medium-expert-v2, (j) hopper-expert-v2, (k) walker2d-random-v2, (l) walker2d-medium-v2, (m) walker2d-medium-replay-v2, (n) walker2d-medium-expert-v2, and (o) walker2d-expert-v2.}
  \label{fig:D4RL_comparison}
  \end{figure*}

\begin{table*}[t]
  \caption{Normalized average scores for MCRQ and baseline algorithms.}
  \resizebox{\textwidth}{!}{
  \begin{tabular}{lcccccccccccc}
  \hline
  Task Name & BEAR & UWAC & BC & CDC & AWAC & BCQ & OneStep & TD3\_BC & CQL & IQL & PBRL & MCRQ \\
  \hline
  halfcheetah-r & 0.00 & 0.00 & 0.00 & 0.99 & 0.15 & 0.00 & 0.00 & 0.35 & 0.55 & 0.45 & 0.35 & 1.00 \\
  hopper-r & 0.05 & 0.00 & 0.00 & 0.49 & 0.26 & 0.20 & 0.12 & 0.22 & 0.23 & 0.19 & 0.97 & 1.00 \\
  walker2d-r & 0.95 & 0.14 & 0.09 & 0.53 & 0.00 & 0.36 & 0.51 & 0.11 & 0.30 & 0.27 & 0.60 & 1.00 \\
  \hline
  halfcheetah-m & 0.04 & 0.00 & 0.02 & 0.21 & 0.31 & 0.29 & 0.45 & 0.34 & 0.25 & 0.33 & 0.86 & 1.00 \\
  hopper-m & 0.02 & 0.00 & 0.07 & 0.22 & 0.21 & 0.12 & 0.85 & 0.20 & 0.28 & 0.23 & 0.57 & 1.00 \\
  walker2d-m & 0.00 & 0.84 & 0.74 & 0.92 & 0.93 & 0.85 & 0.95 & 0.94 & 0.91 & 0.91 & 1.00 & 0.95 \\
  \hline
  halfcheetah-m-r & 0.05 & 0.03 & 0.00 & 0.56 & 0.57 & 0.21 & 0.44 & 0.57 & 0.60 & 0.54 & 0.58 & 1.00 \\
  hopper-m-r & 0.36 & 0.00 & 0.04 & 0.41 & 0.59 & 0.42 & 0.97 & 0.57 & 0.85 & 0.73 & 1.00 & 0.82 \\
  walker2d-m-r & 0.00 & 0.20 & 0.15 & 0.20 & 0.87 & 0.48 & 0.67 & 0.92 & 0.89 & 0.88 & 0.86 & 1.00 \\
  \hline
  halfcheetah-m-e & 0.06 & 0.00 & 0.35 & 0.33 & 0.43 & 1.00 & 0.63 & 0.97 & 0.96 & 0.99 & 0.97 & 1.00 \\
  hopper-m-e & 0.09 & 0.00 & 0.11 & 0.64 & 0.84 & 0.89 & 0.97 & 0.85 & 0.79 & 0.83 & 1.00 & 0.81 \\
  walker2d-m-e & 0.00 & 0.83 & 0.85 & 0.54 & 0.98 & 0.97 & 0.99 & 0.98 & 0.95 & 1.00 & 0.98 & 0.99 \\
  \hline
  halfcheetah-e & 0.71 & 0.73 & 0.72 & 0.03 & 0.00 & 0.80 & 0.42 & 1.00 & 0.93 & 0.97 & 0.69 & 0.98 \\
  hopper-e & 0.00 & 1.00 & 0.99 & 0.86 & 0.98 & 0.87 & 0.94 & 1.00 & 0.97 & 0.96 & 1.00 & 0.97 \\
  walker2d-e & 0.75 & 0.83 & 0.83 & 0.00 & 0.89 & 0.82 & 0.92 & 0.89 & 0.87 & 1.00 & 0.82 & 0.93 \\
  \hline
  \end{tabular}
  }
  \label{tabular:D4RL_comparison}
\end{table*}

\begin{figure}
  \centering
  \includegraphics[width = 0.48\textwidth]{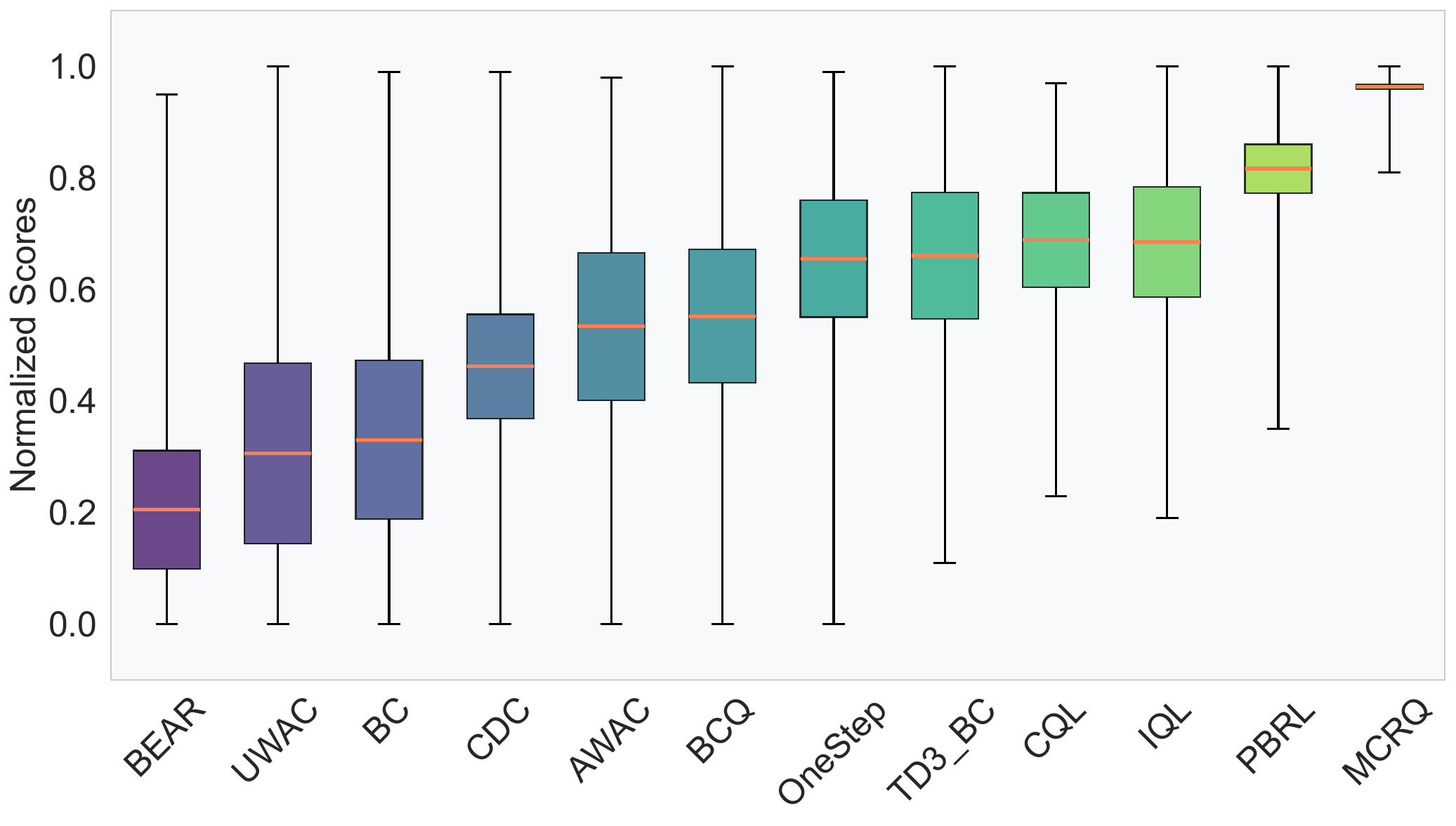}
  \hfill
  \caption{Box plot with minimum, maximum, mean, and variance of 15 normalized average scores for MCRQ and baseline algorithms.}
  \label{fig:D4RL_comparison_normalized}
\end{figure}

\subsection{KL Divergence Comparison}
\label{subsection:kl-divergence-comparison}

KL divergence is used to evaluate the similarity between the target policy and the behavior policy, with a lower value indicating closer alignment. We compare the KL divergence of BCQ, TD3\_BC, CQL, IQL, and MCRQ across five halfcheetah datasets. The results are shown in Fig.~\ref{fig:D4RL_KL_Divergence_Comparison} and Table~\ref{tabular:D4RL_KL_Divergence_Comparison}, where the table reports the average of the last ten KL divergence values.

In halfcheetah-random, CQL appears most conservative, while MCRQ exhibits a lower KL divergence than BCQ, TD3\_BC, and IQL. This suggests that MCRQ better aligns with the low-quality behavior data, avoiding excessive deviation from the behavior policy. In contrast, BCQ shows a significantly higher KL divergence, which corresponds with its poor performance (0.00) on halfcheetah-random, as seen in Table~\ref{tabular:D4RL_comparison}. For the remaining datasets, all algorithms maintain low KL divergence, indicating that their target policies remain closely aligned with the behavior policy.

\begin{figure*}
  \centering
    \subcaptionbox{}{\includegraphics[width = 0.32\textwidth]{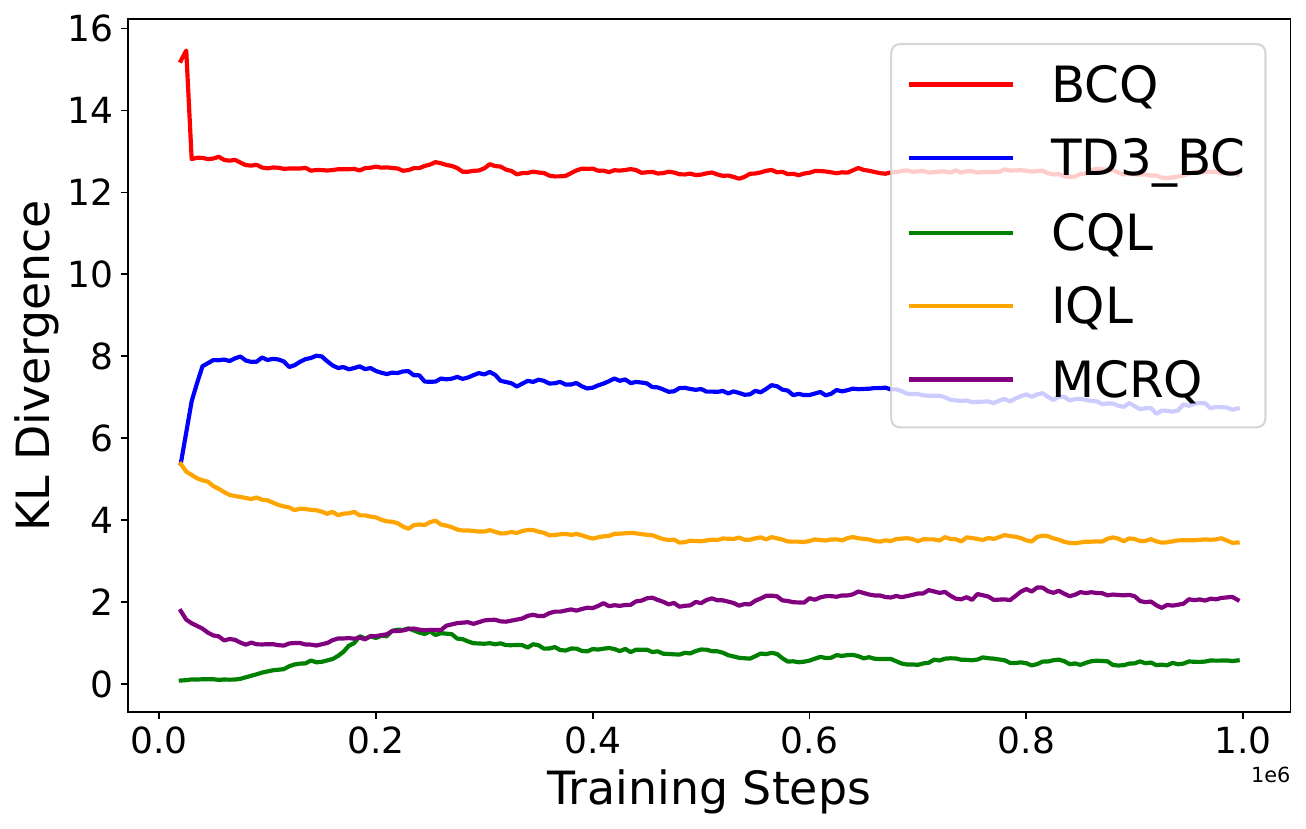}}
    \hfill
    \subcaptionbox{}{\includegraphics[width = 0.32\textwidth]{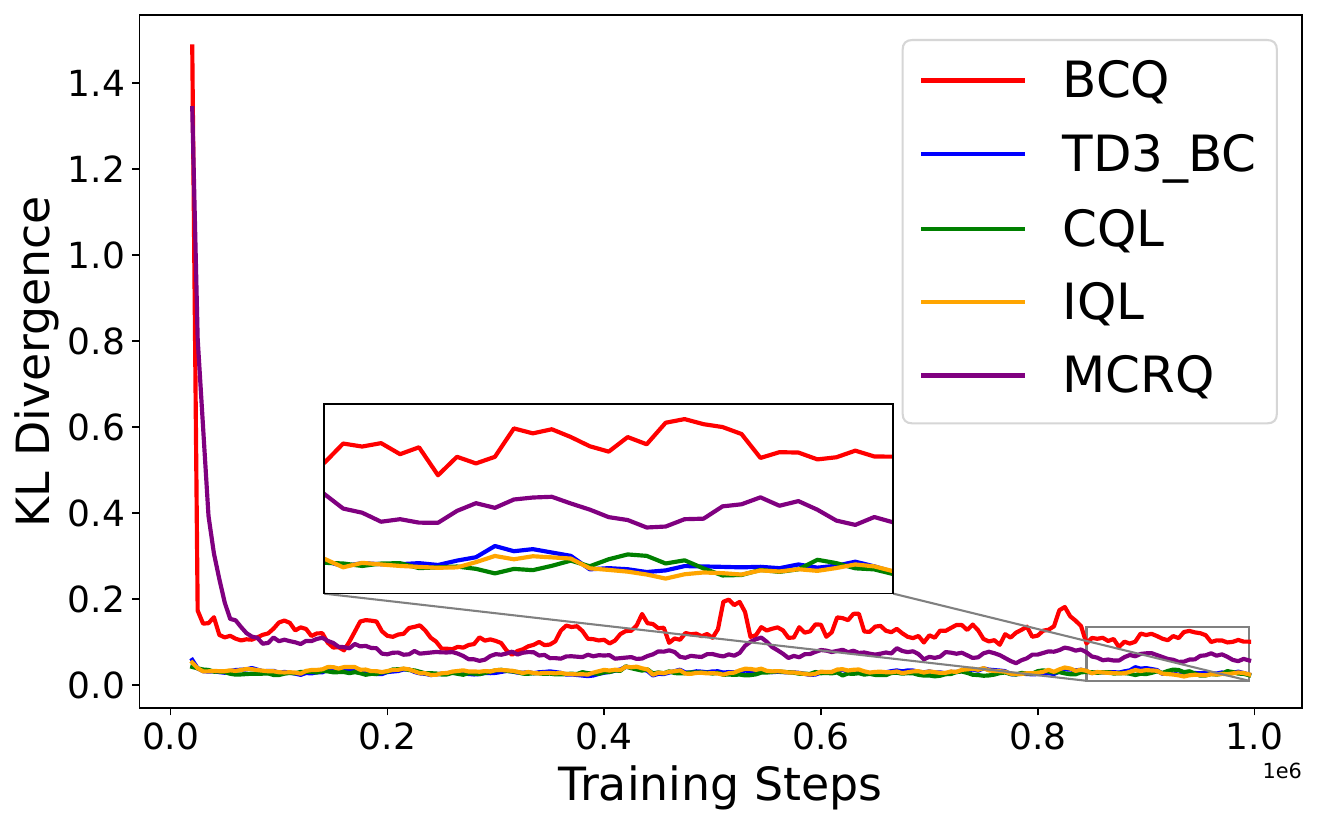}}
    \hfill
    \subcaptionbox{}{\includegraphics[width = 0.32\textwidth]{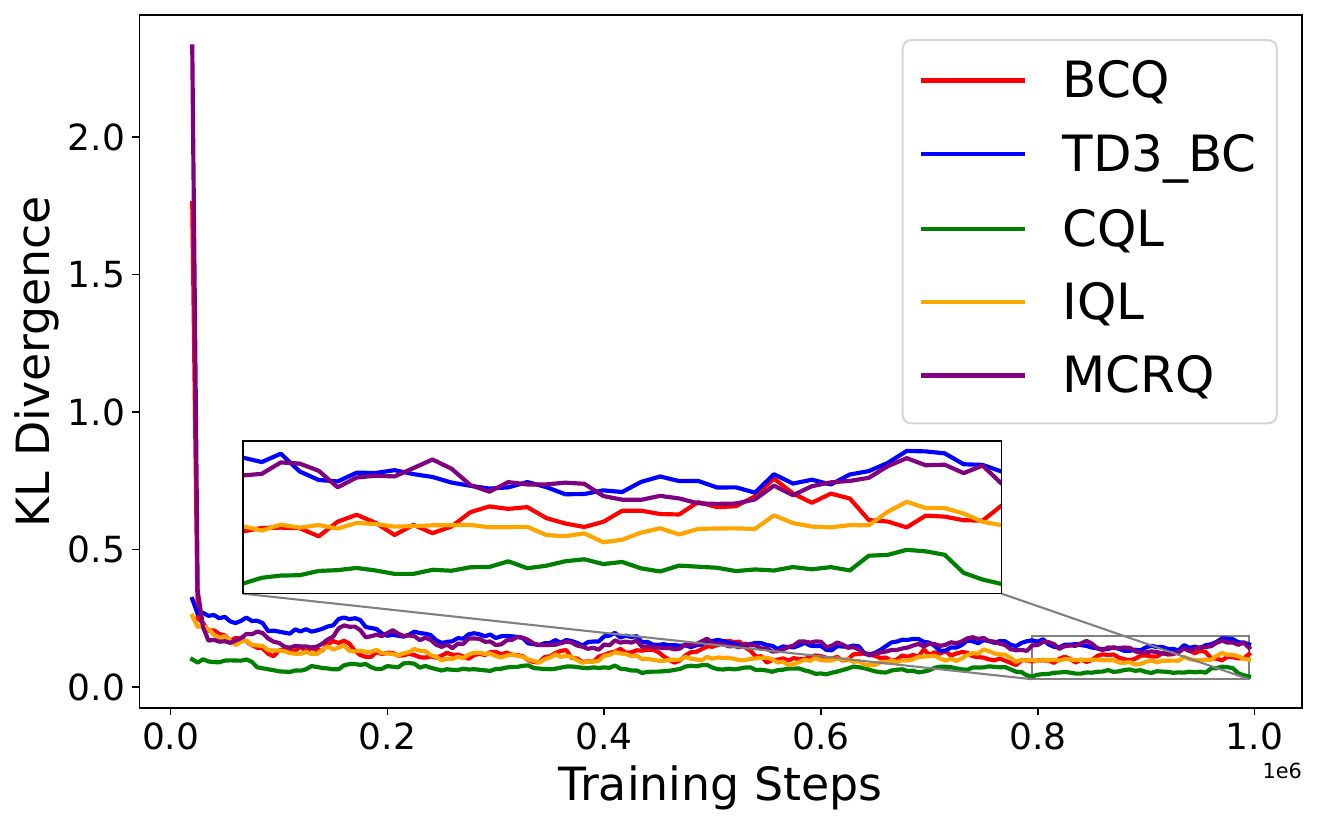}}
    \hfill
    \subcaptionbox{}{\includegraphics[width = 0.32\textwidth]{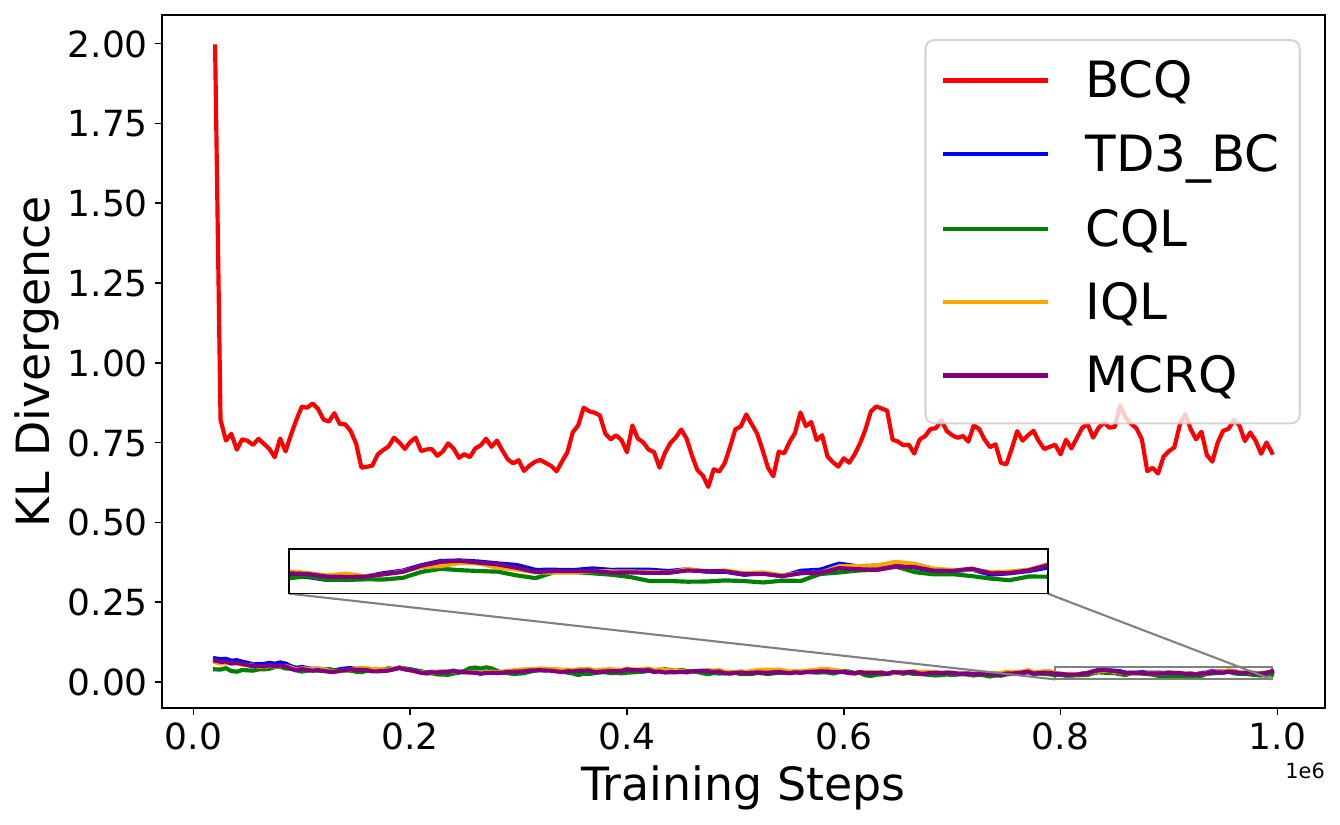}}
    \subcaptionbox{}{\includegraphics[width = 0.32\textwidth]{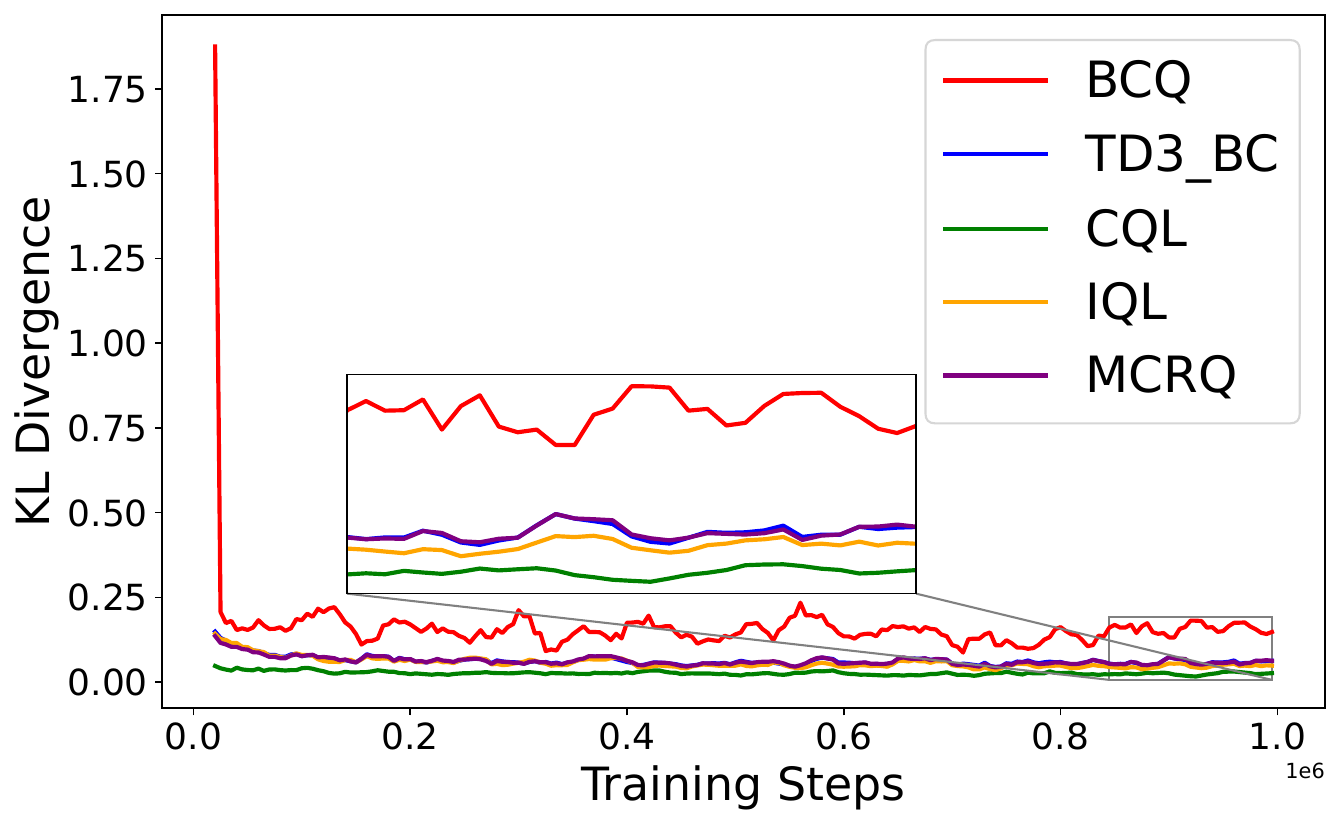}}
  \caption{Comparison of KL divergence across five halfcheetah datasets. (a) halfcheetah-random-v2, (b) halfcheetah-medium-v2, (c)  halfcheetah-medium-replay-v2, (d) halfcheetah-medium-expert-v2, and (e) halfcheetah-expert-v2.}
  \label{fig:D4RL_KL_Divergence_Comparison}
\end{figure*}

\begin{table}[!t]
  \caption{Comparison of KL divergence in Fig.~\ref{fig:D4RL_KL_Divergence_Comparison}.} 
  \centering
  \begin{tabular}{llllll}
    \toprule
    Algorithm & BCQ & TD3\_BC & CQL & IQL & MCRQ  \\
    \midrule
    halfcheetah-r & 12.48 & 6.78 & 0.56 & 3.51 & 2.10 \\
    halfcheetah-m & 0.11 & 0.03 & 0.03 & 0.03 & 0.07 \\
    halfcheetah-m-r & 0.11 & 0.16 & 0.06 & 0.11 & 0.16 \\
    halfcheetah-m-e & 0.77 & 0.03 & 0.03 & 0.03 & 0.03 \\
    halfcheetah-e & 0.16 & 0.06 & 0.03 & 0.05 & 0.06 \\
    \bottomrule
  \end{tabular}
  \label{tabular:D4RL_KL_Divergence_Comparison}
\end{table}

\subsection{Ablation Experiments}
\label{subsection:ablation-experiments}

We conduct ablation experiments on halfcheetah-medium and halfcheetah-expert, as shown in Fig.~\ref{fig:D4RL_ablation}. The corresponding analysis and discussion are presented below.

1) Impact of $\omega$: As $\omega$ increases, MCRQ becomes more conservative, causing the target policy's actions to more closely align with those in the offline dataset. On the selected datasets, MCRQ’s performance consistently declines when $\omega$ is set to  \{2.0, 2.5\}, regardless of variations in $\alpha$ and $\upsilon$.

2) Impact of $\upsilon$: When $\alpha$ and $\omega$ are fixed, MCRQ exhibits relatively stable performance across different values of $\upsilon$ in halfcheetah-medium, indicating that changes in $\upsilon$ have minimal impact on performance. In contrast, on halfcheetah-expert, MCRQ is more sensitive to $\upsilon$ when $\alpha$ is set to \{10.0, 15.0, 20.0, 25.0\}.

3) Impact of $\alpha$: In subfigures (a)-(f), when $\upsilon$ and $\omega$ are fixed, with $\omega \in \{0.0, 0.5, 1.0, 1.5\}$, the performance of MCRQ consistently improves as $\alpha$ increases. However, in subfigures (g)-(l), MCRQ becomes more sensitive to changes in $\alpha$. Notably, when $\alpha \in \{10.0, 15.0, 20.0, 25.0\}$, the performance deteriorates significantly, regardless of the values of $\upsilon$ and $\omega$.

\begin{figure*}
  \centering
    \subcaptionbox{}{\includegraphics[width = 0.24\textwidth]{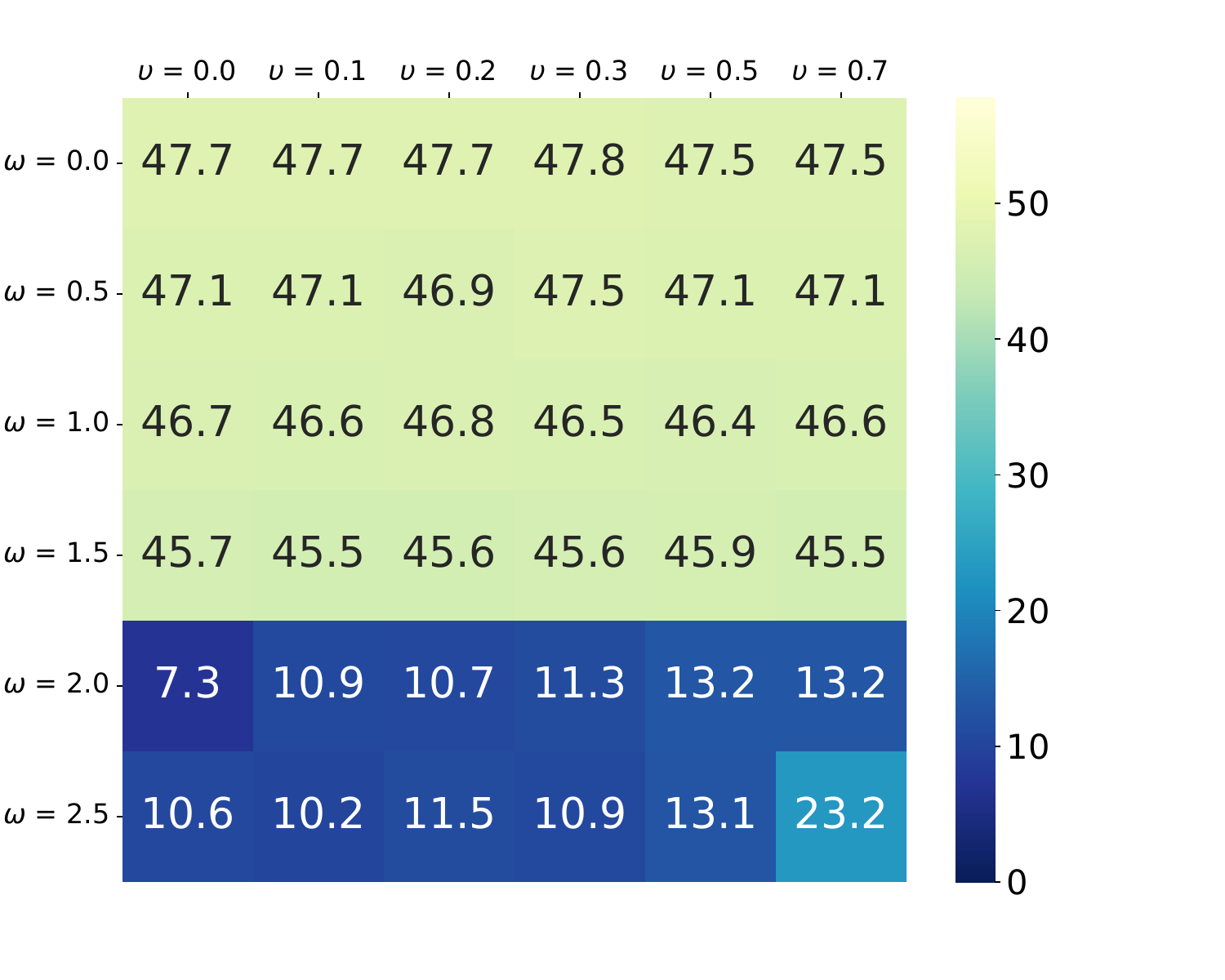}}
    \hfill
    \subcaptionbox{}{\includegraphics[width = 0.24\textwidth]{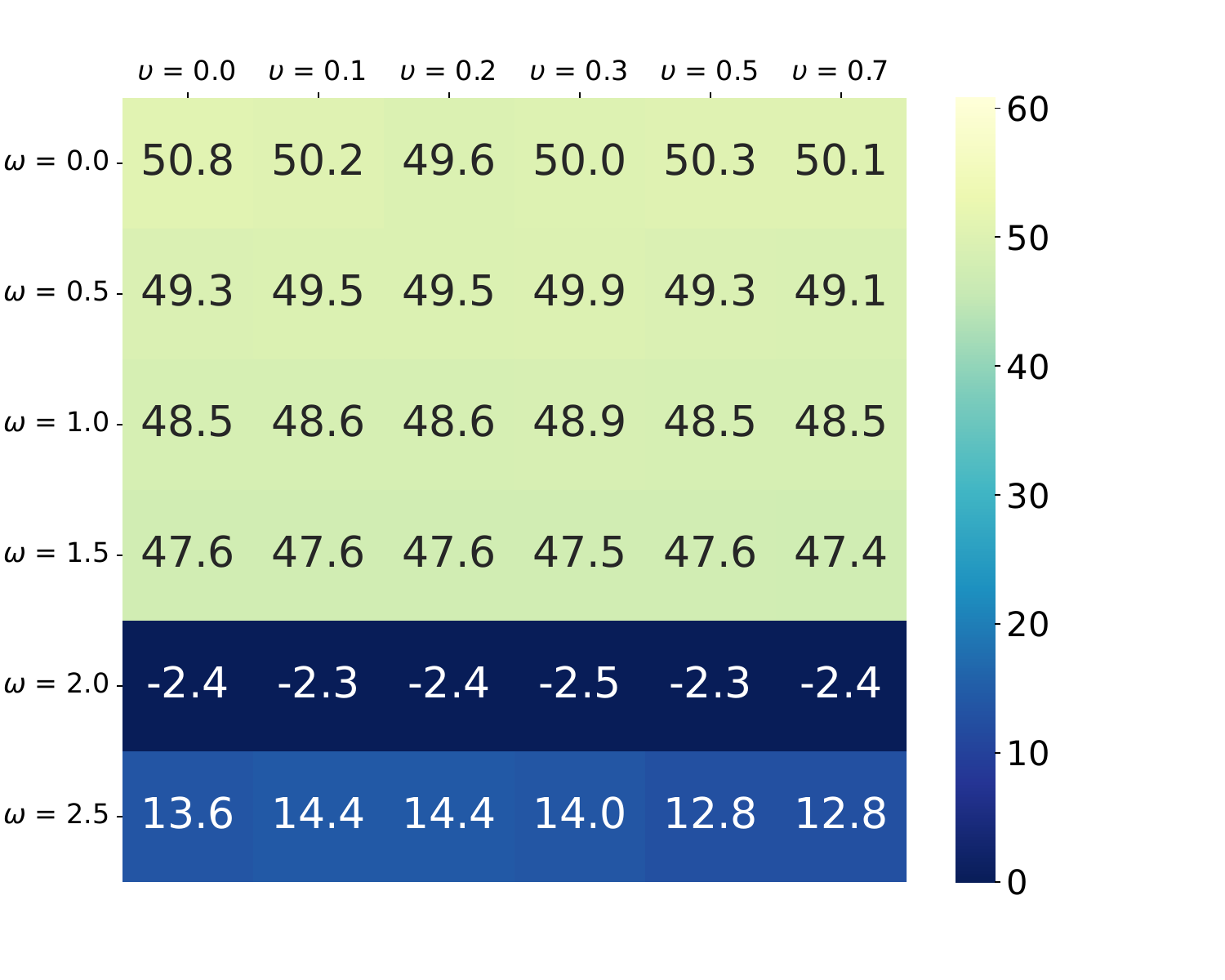}}
    \hfill
    \subcaptionbox{}{\includegraphics[width = 0.24\textwidth]{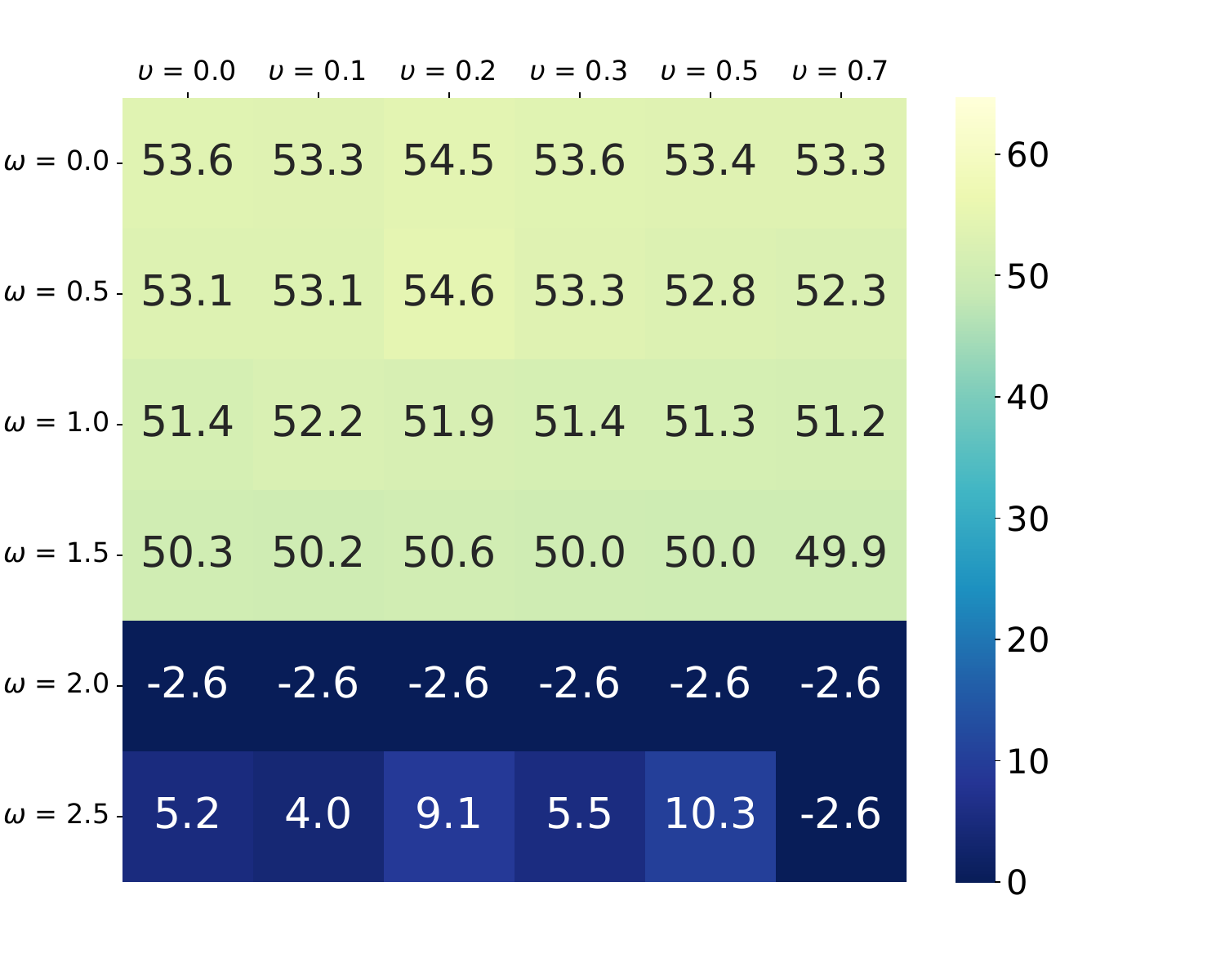}}
    \hfill
    \subcaptionbox{}{\includegraphics[width = 0.24\textwidth]{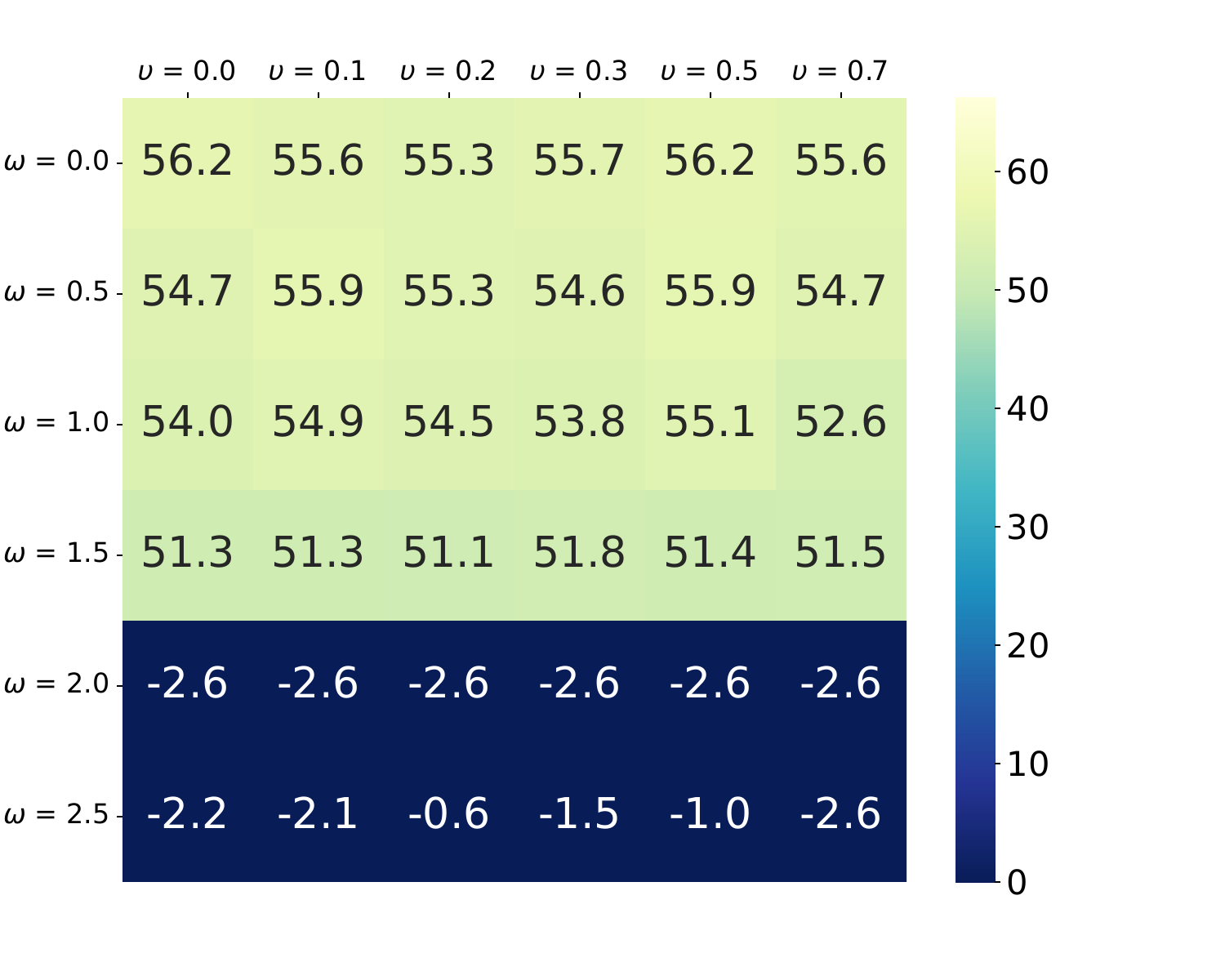}}
    \hfill
    \subcaptionbox{}{\includegraphics[width = 0.24\textwidth]{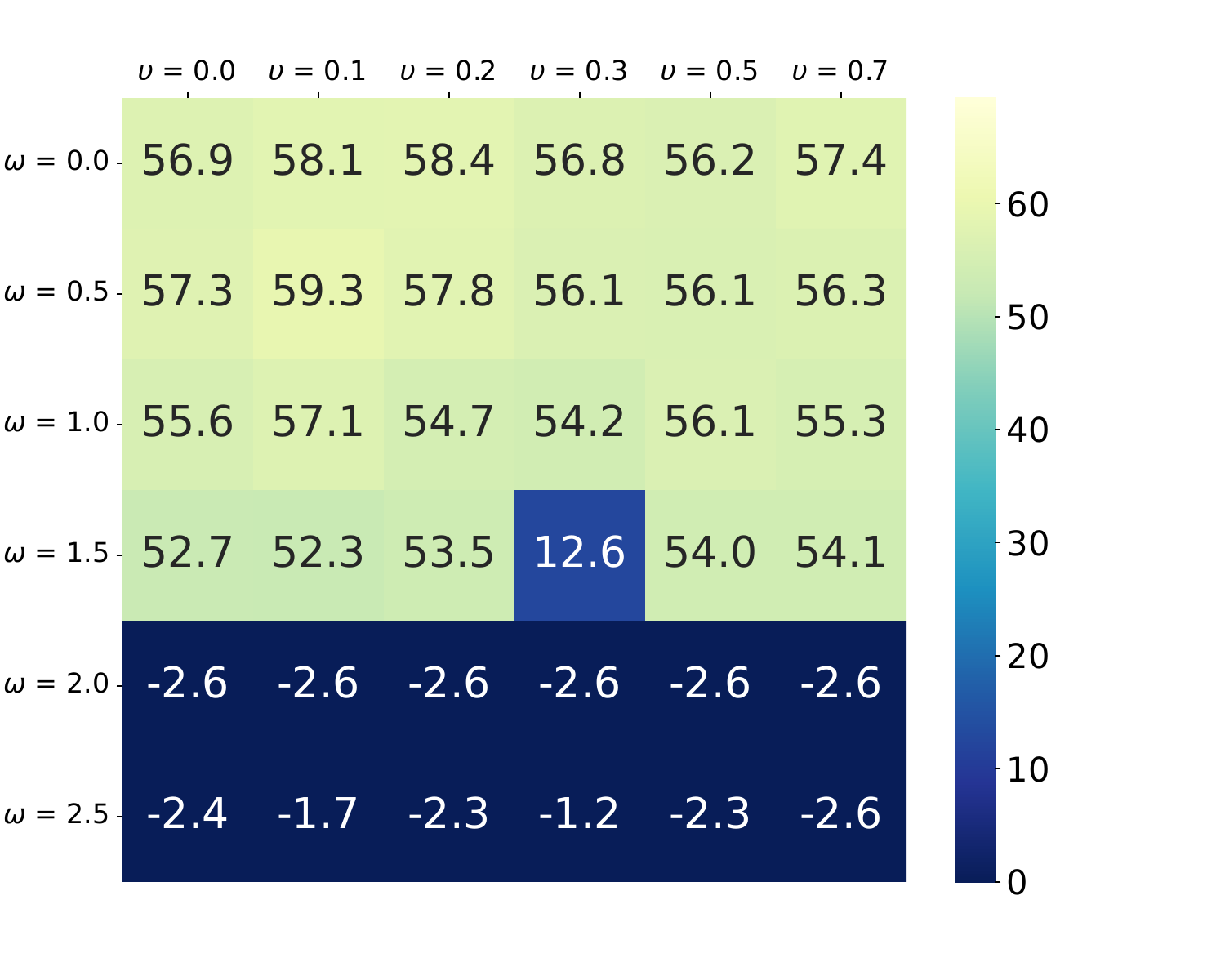}}
    \hfill
    \subcaptionbox{}{\includegraphics[width = 0.24\textwidth]{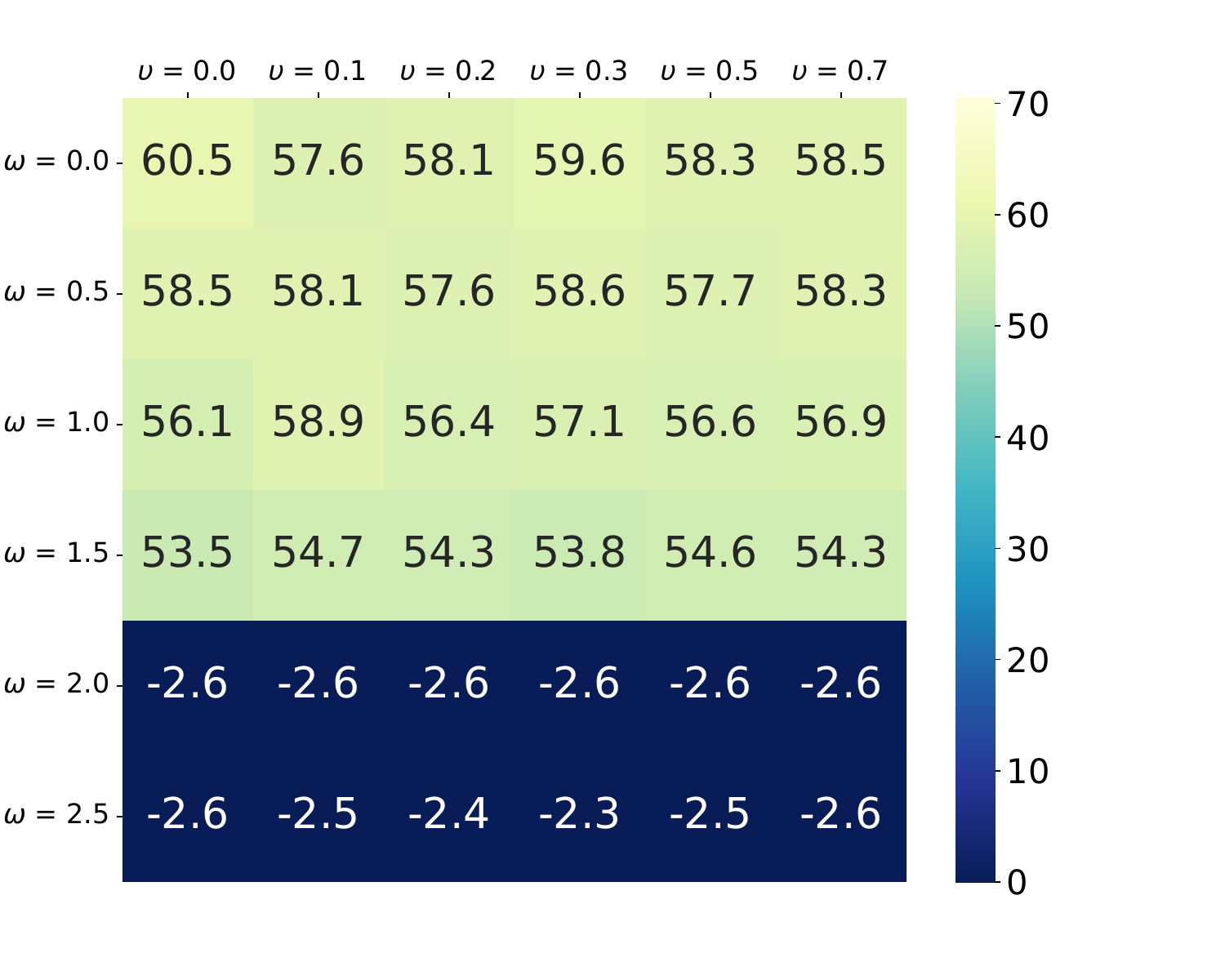}}
    \hfill
    \subcaptionbox{}{\includegraphics[width = 0.24\textwidth]{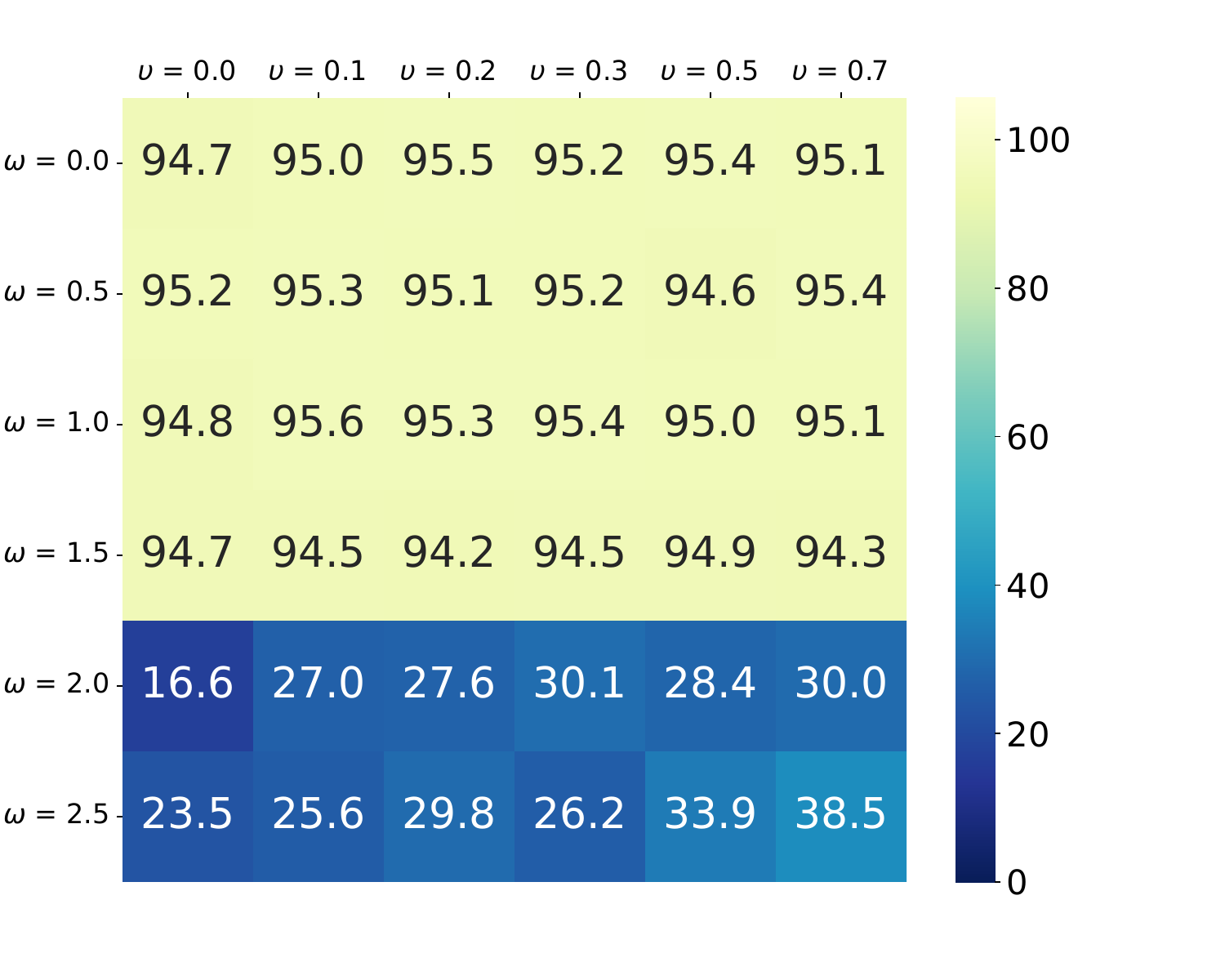}}
    \hfill
    \subcaptionbox{}{\includegraphics[width = 0.24\textwidth]{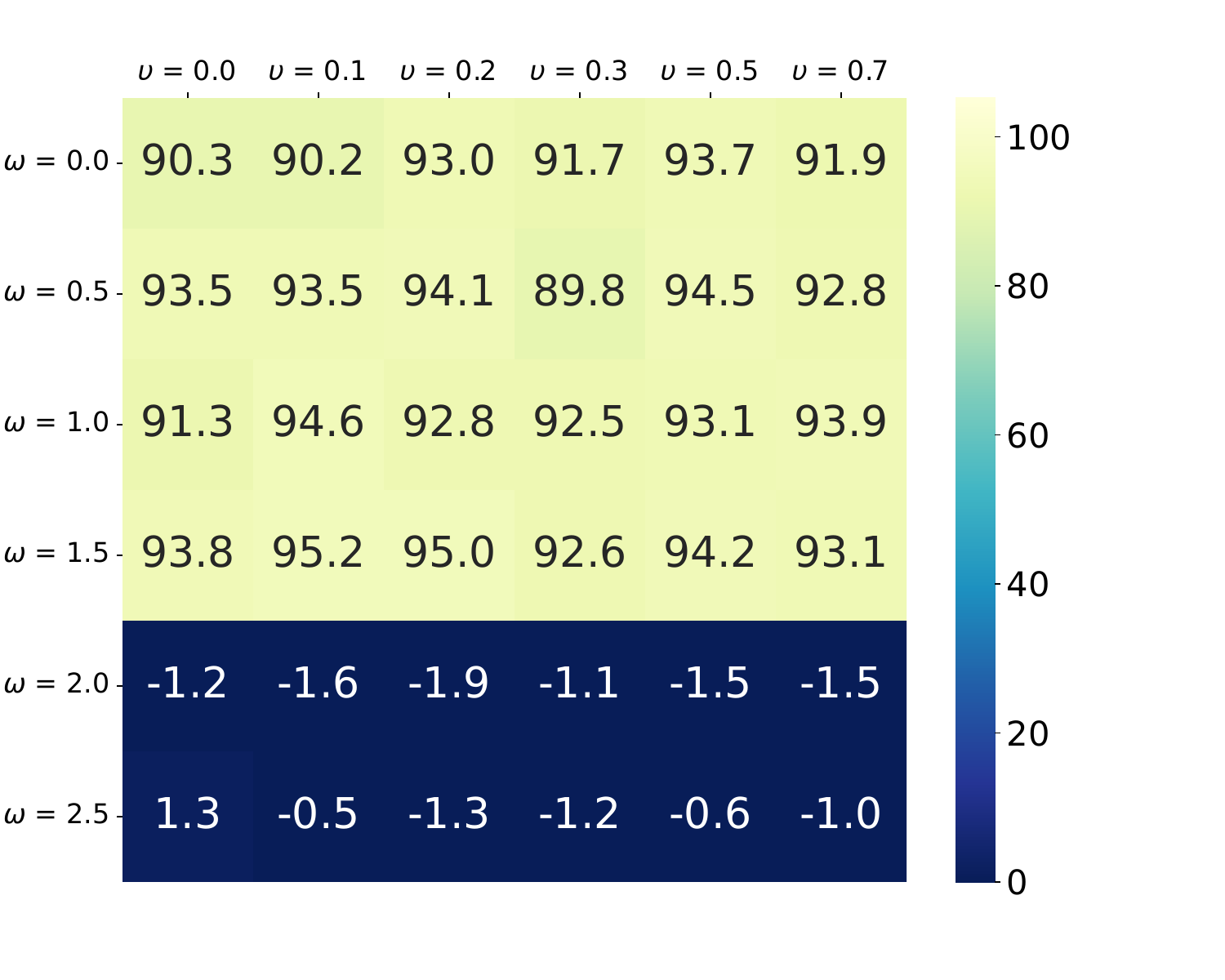}}
    \hfill
    \subcaptionbox{}{\includegraphics[width = 0.24\textwidth]{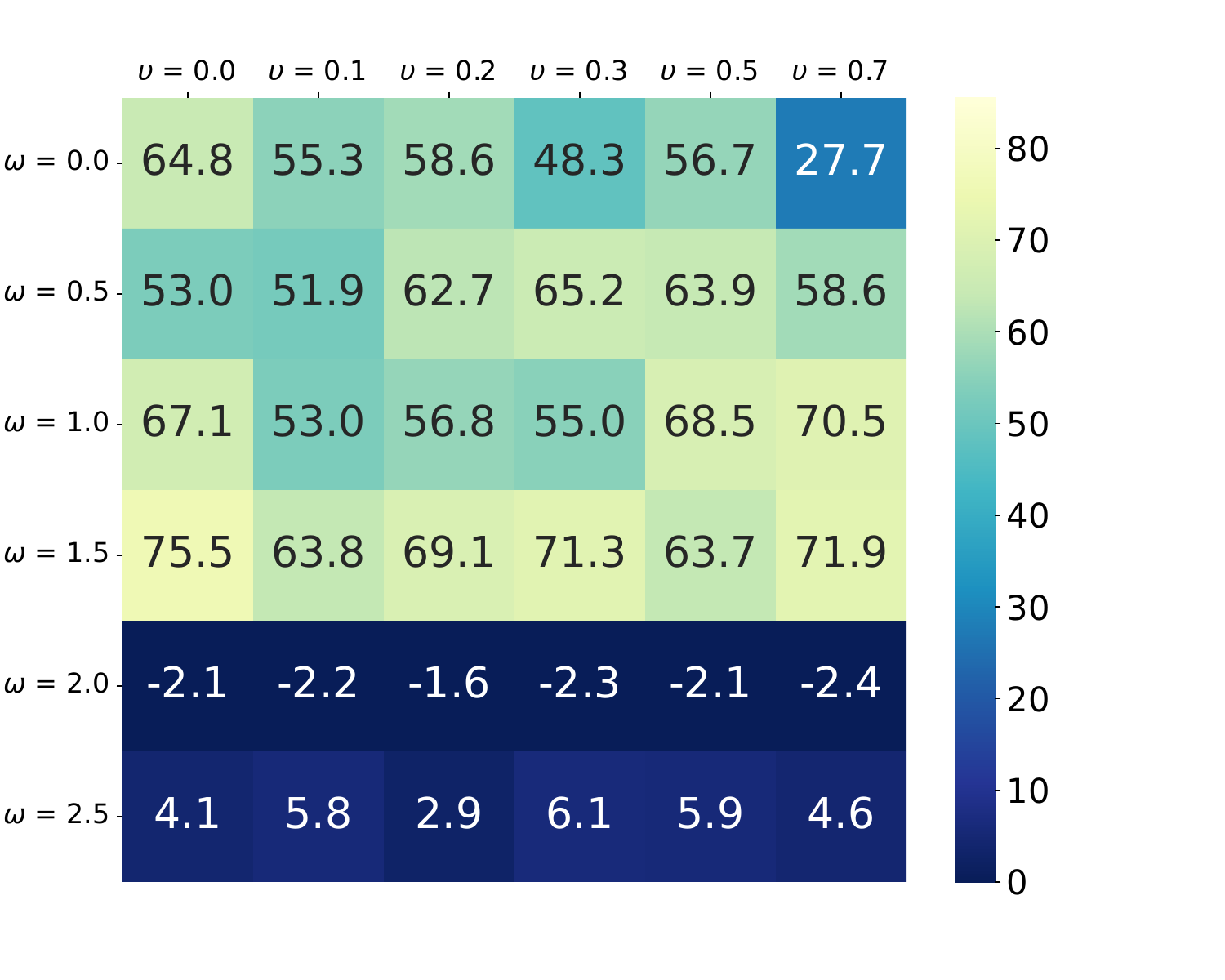}}
    \hfill
    \subcaptionbox{}{\includegraphics[width = 0.24\textwidth]{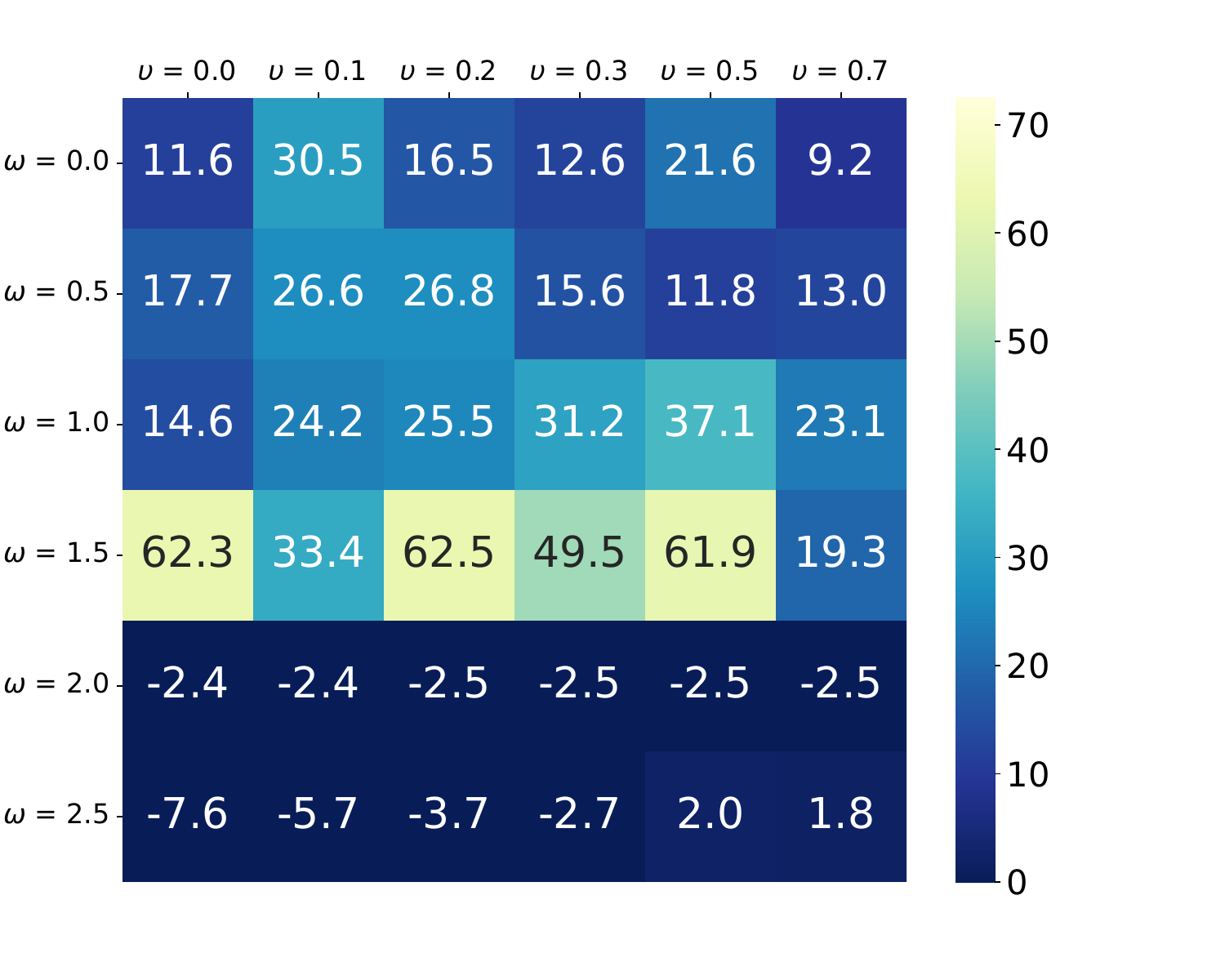}}
    \hfill
    \subcaptionbox{}{\includegraphics[width = 0.24\textwidth]{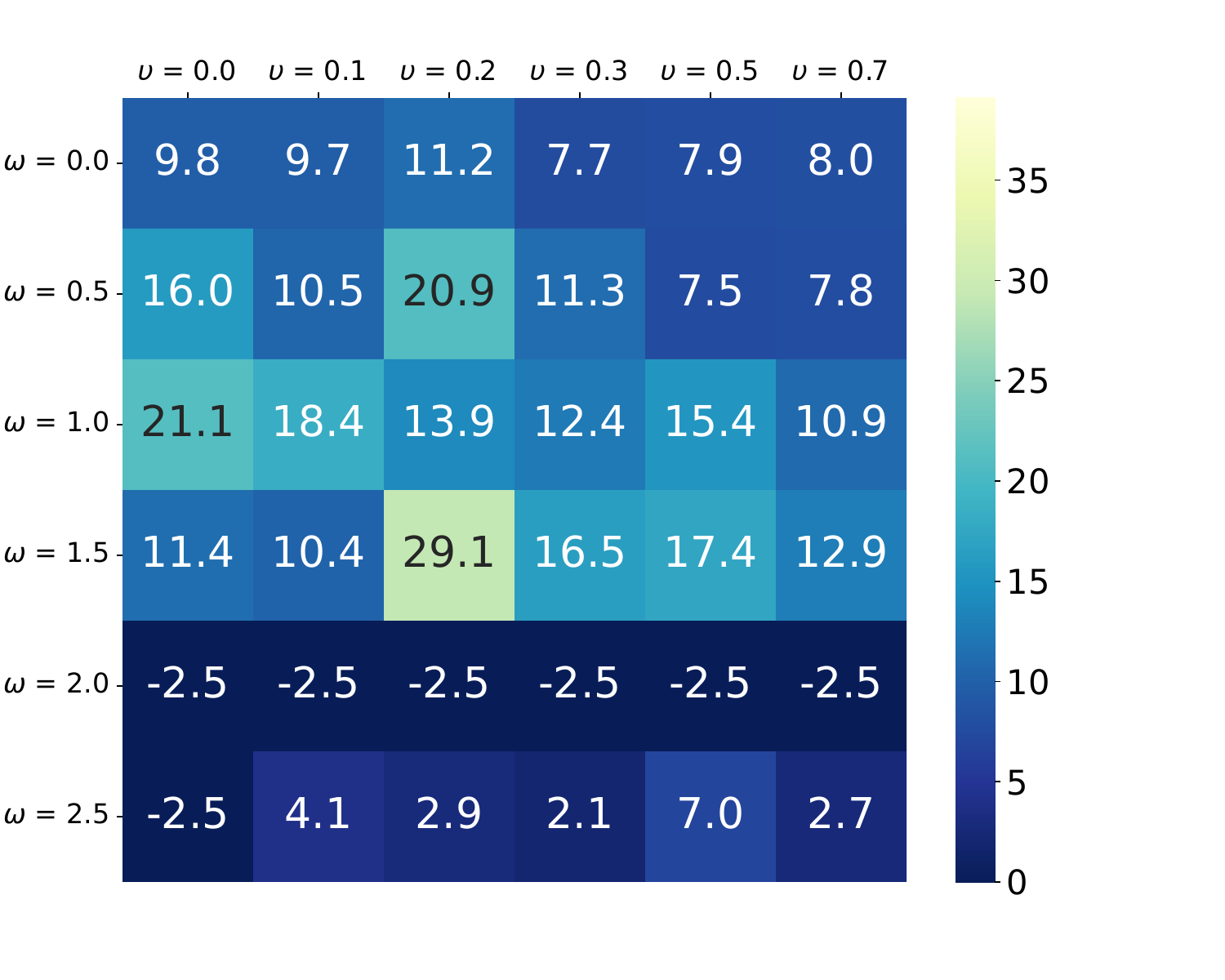}}
    \hfill
    \subcaptionbox{}{\includegraphics[width = 0.24\textwidth]{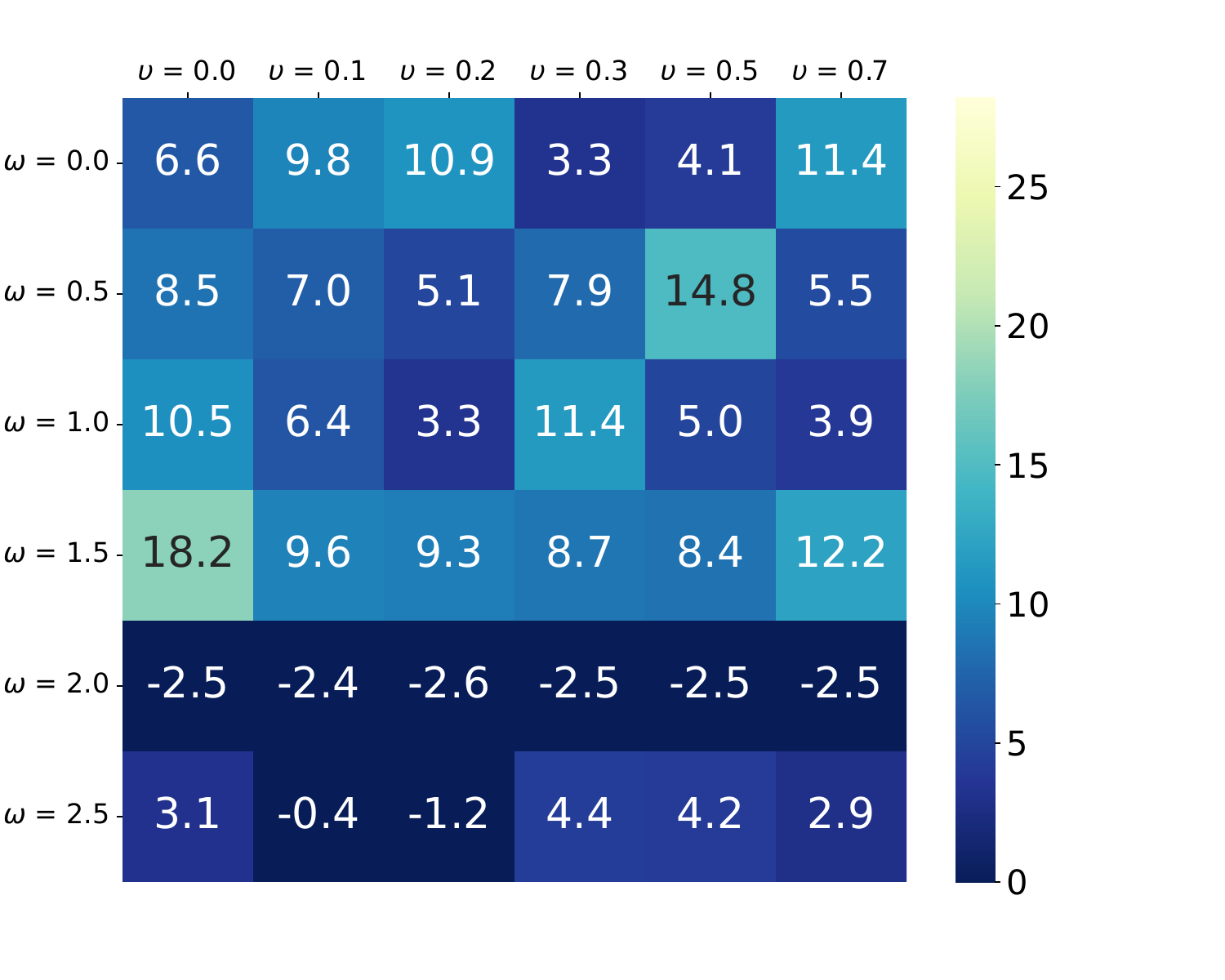}}
  \caption{Ablation study on $\upsilon$, $\omega$, and $\alpha$ in MCRQ conducted on halfcheetah-medium and halfcheetah-expert. Subfigures (a)-(f) show the results on halfcheetah-medium with $\alpha$ values of \{2.5, 5.0, 10.0, 15.0, 20.0, 25.0\}, while (g)-(l) present the corresponding results on halfcheetah-expert using the same set of $\alpha$ values. All values in the figure represent the mean performance.}
  \label{fig:D4RL_ablation}
\end{figure*}

\subsection{Comparative Visualization of OOD Actions}
\label{subsection:comparative-visualization-of-ood-actions}

To evaluate the ability to constrain OOD actions, we compare the action distributions generated by the learned policies of different algorithms with the action distribution of the offline dataset. Specifically, we train BCQ, TD3\_BC, and MCRQ, and collect 51,200 samples by running each trained policy separately. For comparison, we also randomly sample 51,200 actions from the offline dataset.
Fig.~\ref{fig:D4RL_OOD_distribution} (a)-(c) show the action distributions in halfcheetah-random with t-SNE, while Fig.~\ref{fig:D4RL_OOD_distribution} (d)-(f) depict the distributions in halfcheetah-medium-replay. The corresponding analysis and discussion are provided below.

BCQ generates almost entirely OOD actions, indicating poor adaptation to halfcheetah-random. TD3\_BC also produces predominantly OOD actions, with samples clustering unilaterally relative to those in the dataset. In contrast, MCRQ more closely aligns its action distribution with the dataset, though slight deviations from the distributional centroid remain. This suggests that while MCRQ outperforms BCQ and TD3\_BC in approximating the dataset distribution, it does not fully replicate it. Notably, since the random dataset is generated by stochastic policies containing both low and high quality actions with inherent uncertainty, strict adherence to its distribution may lead to suboptimal performance. On halfcheetah-medium-replay, all three algorithms align well with the offline dataset.

\begin{figure*}
  \centering
    \subcaptionbox{}{\includegraphics[width = 0.32\textwidth]{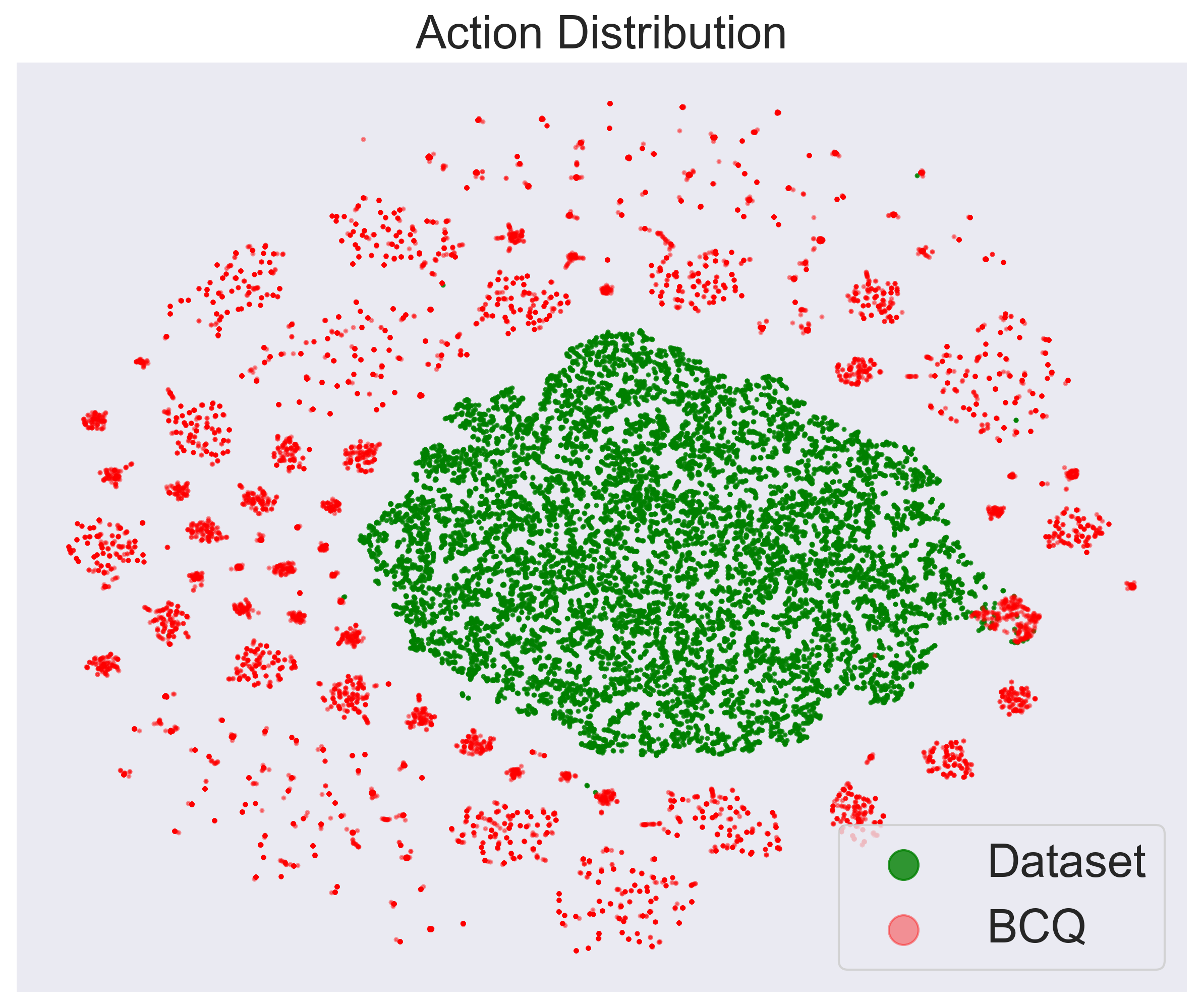}}
    \hfill
    \subcaptionbox{}{\includegraphics[width = 0.32\textwidth]{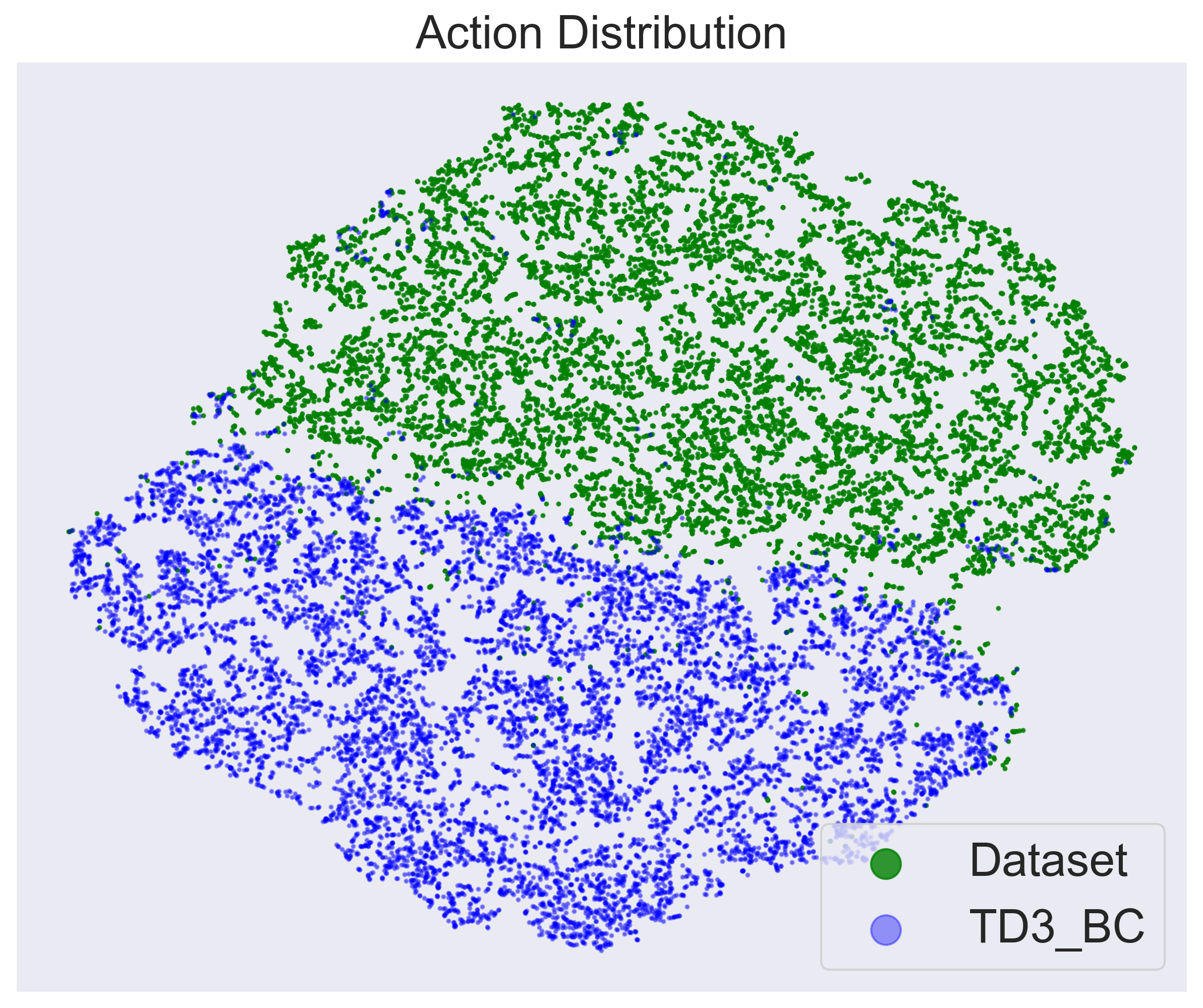}}
    \hfill
    \subcaptionbox{}{\includegraphics[width = 0.32\textwidth]{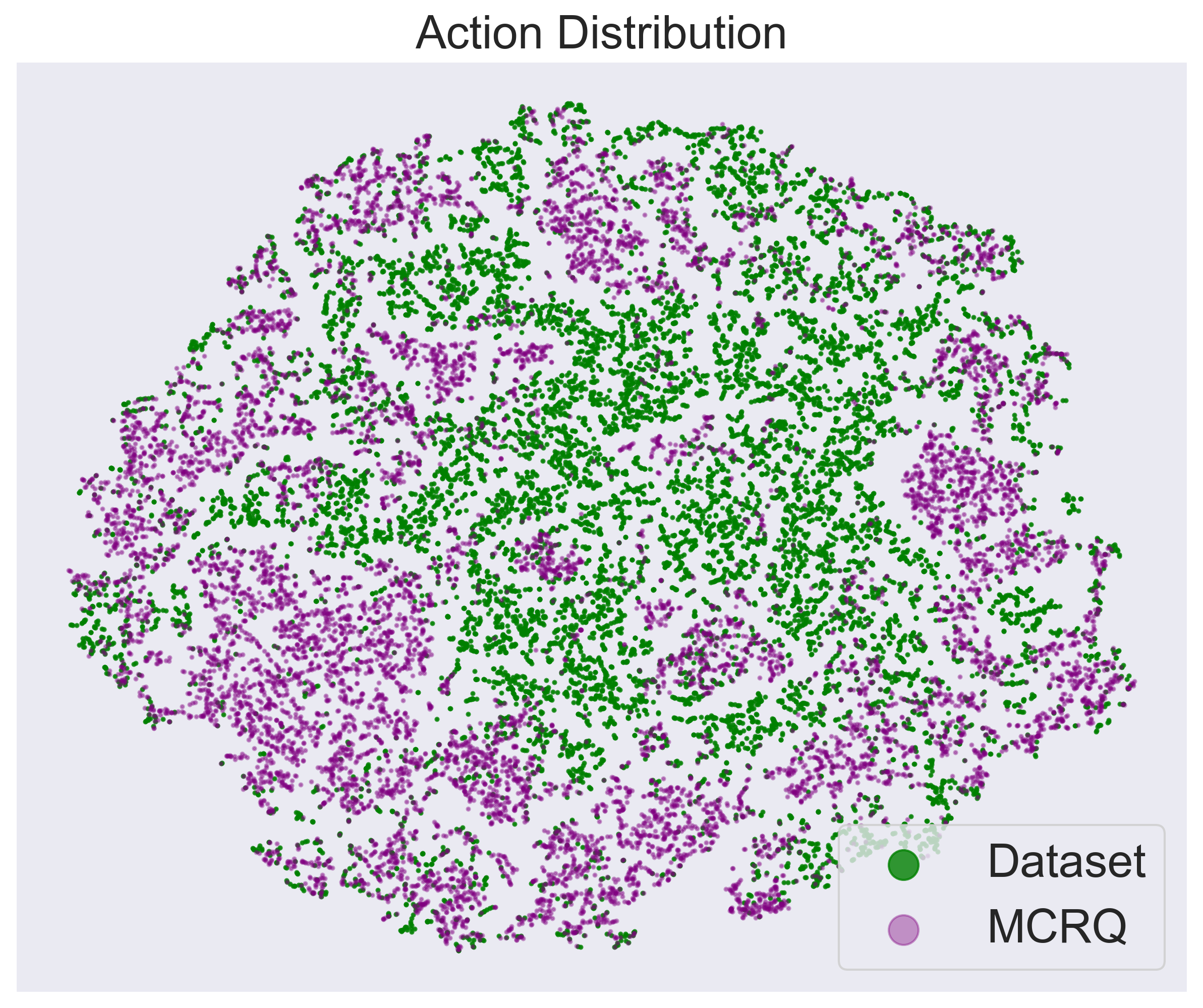}}
    \hfill
    \subcaptionbox{}{\includegraphics[width = 0.32\textwidth]{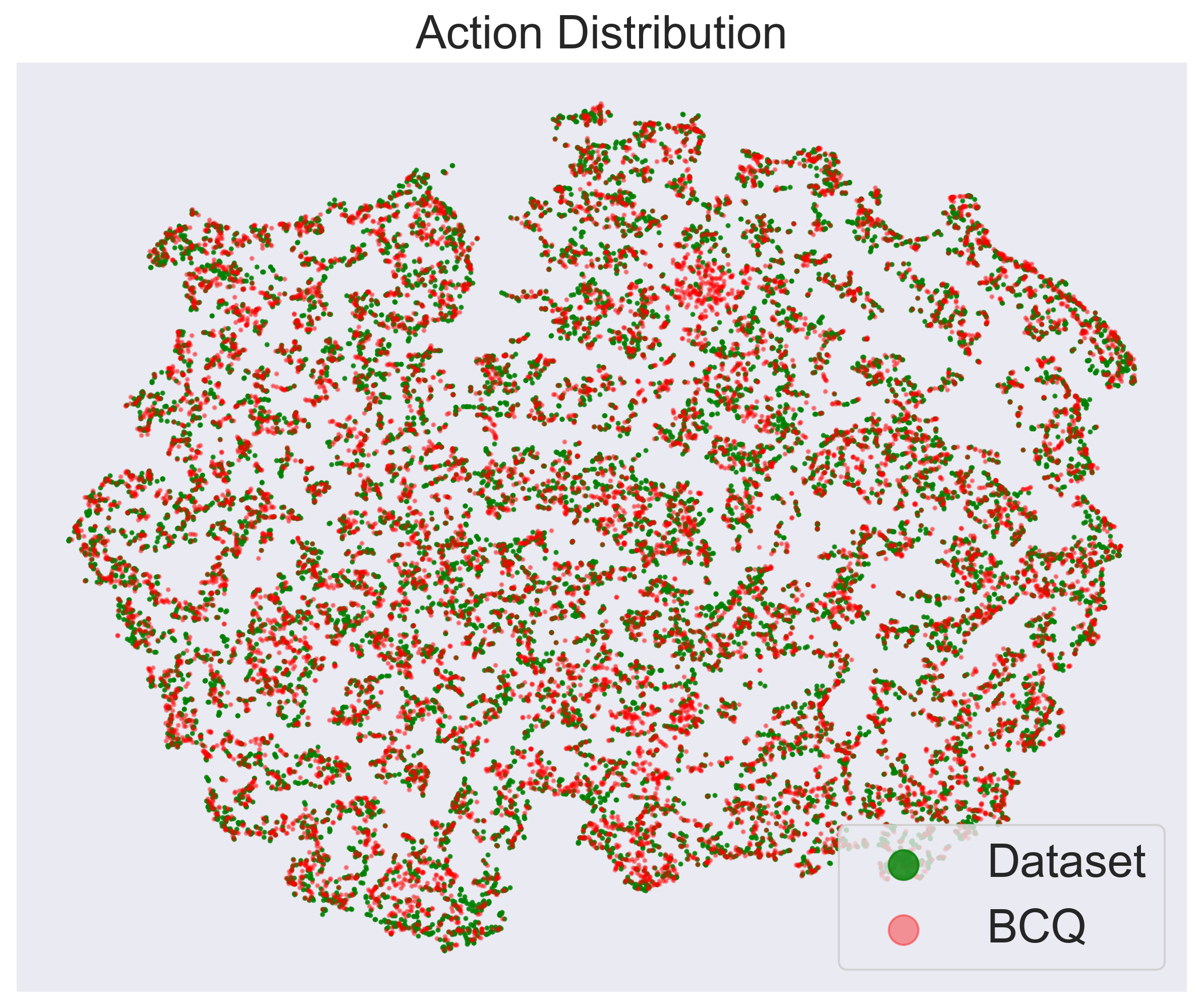}}
    \hfill
    \subcaptionbox{}{\includegraphics[width = 0.32\textwidth]{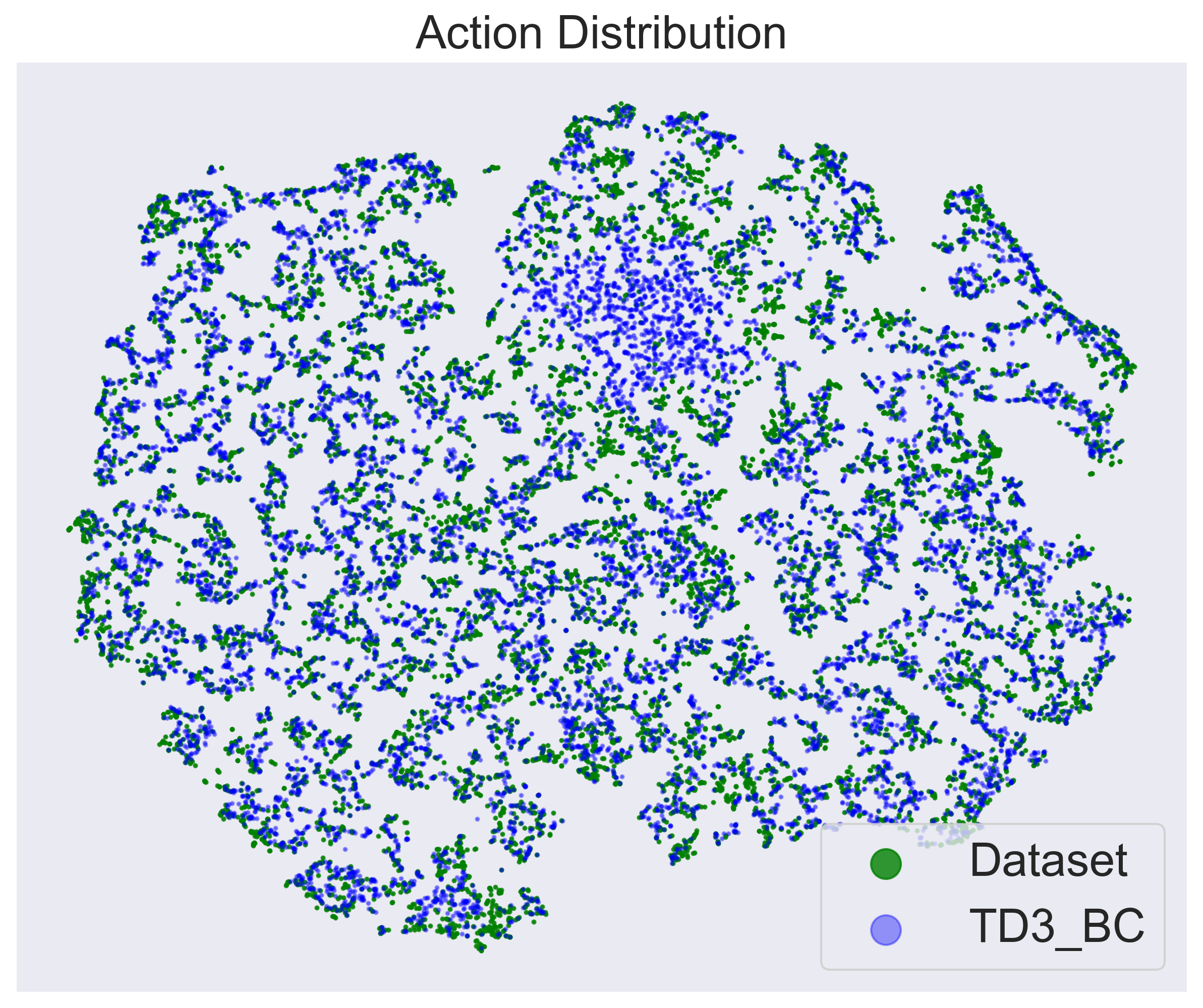}}
    \hfill
    \subcaptionbox{}{\includegraphics[width = 0.32\textwidth]{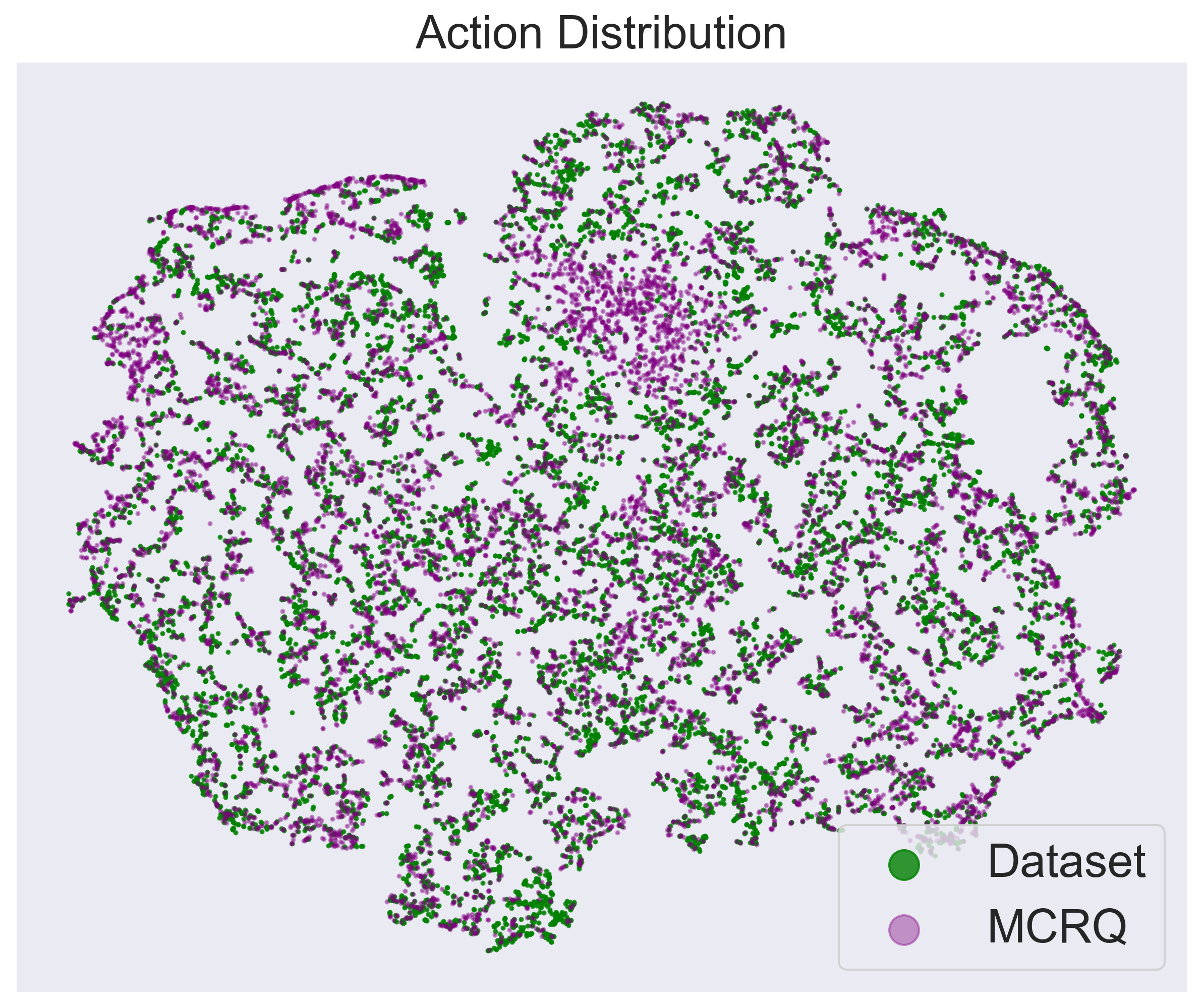}}
  \caption{Policy action visualization. (a), (b), and (c) show the action outputs of BCQ, TD3\_BC, and MCRQ compared with the actions within the halfcheetah-random. (d), (e), and (f) show the action outputs of the same algorithms compared with the actions within the halfcheetah-medium-replay. The points represent the action outputs projected onto a 2-D plane using the t-SNE algorithm.}
  \label{fig:D4RL_OOD_distribution}
\end{figure*}

\subsection{Comparison with SOTA Algorithms}
\label{subsection:comparison-with-sota-algorithms}

We further compare MCRQ with recent SOTA algorithms to demonstrate its superiority. The selected algorithms include  $\mathnormal{f}$-DVL \cite{sikchi2024dualf-DVL}, CGDT \cite{wang2024critic}, DD \cite{ajay2023isDD}, DStitch \cite{li2024diffstitch}, ODC \cite{kim2024decisionODC}, CSVE \cite{chen2023CSVE}, ACT \cite{gao2024act}, MISA \cite{xiao2023MISA}, O-DICE \cite{mao2024odiceO-DICE}, MOAC \cite{huangMildPolicyEvaluation2024}, and ORL-RC \cite{huangEfficientOfflineReinforcement2024}. 
These algorithms incorporate various design features, including critic regularization, imitation learning, decision transformers, and diffusion models.





We selected the medium, medium-replay, and medium-expert datasets to evaluate performance across various tasks. The results for each algorithm are sourced from their original papers. The same row-wise normalization used in Table~\ref{tabular:D4RL_comparison} is also applied in Table~\ref{tabular:D4RL_sota}. As shown in Table~\ref{tabular:D4RL_sota}, MCRQ does not achieve the best performance on every dataset. However, Fig.~\ref{fig:D4RL_sota_normalized} shows that MCRQ achieves the highest mean performance among all compared SOTA algorithms, using the same visualization format as 
Fig.~\ref{fig:D4RL_comparison_normalized}. In contrast, $\mathnormal{f}$-DVL exhibits the smallest variance but the lowest mean. It is reasonable to observe that different algorithms exhibit varying strengths and weaknesses across environments; thus, the choice of algorithm should be tailored to the specific environment.

\begin{table*}[t]
  \caption{Normalized average scores for for MCRQ and SOTA algorithms.} 
  \resizebox{\textwidth}{!}{
  \begin{tabular}{lcccccccccccc}
  \hline
  Algorithm & $\mathnormal{f}$-DVL  & CGDT & DD & DStitch & ODC & CSVE & ACT & MISA  & O-DICE & MOAC & ORL-RC & MCRQ \\
  \hline
  halfcheetah-m & 0.26 & 0.00 & 0.35 & 0.42 & 0.03 & 0.31 & 0.35 & 0.25 & 0.25 & 0.65 & 0.70 & 1.00 \\
  hopper-m & 0.10 & 1.00 & 0.52 & 0.00 & 0.91 & 0.99 & 0.20 & 0.19 & 0.70 & 0.63 & 0.60 & 0.92 \\
  walker2d-m & 0.22 & 0.00 & 0.36 & 0.46 & 0.15 & 0.44 & 0.19 & 0.53 & 0.62 & 0.81 & 1.00 & 0.68 \\
  \hline
  halfcheetah-m-r & 0.36 & 0.07 & 0.00 & 0.36 & 0.20 & 1.00 & 0.24 & 0.41 & 0.31 & 0.72 & 0.73 & 0.84 \\
  hopper-m-r & 0.90 & 0.68 & 1.00 & 0.00 & 0.71 & 0.59 & 0.92 & 0.93 & 0.95 & 0.90 & 0.88 & 0.35 \\
  walker2d-m-r & 0.37 & 0.65 & 0.56 & 0.99 & 0.75 & 0.65 & 0.00 & 0.89 & 0.81 & 1.00 & 0.97 & 0.96 \\
  \hline
  halfcheetah-m-e & 0.46 & 0.73 & 0.40 & 0.99 & 0.86 & 0.67 & 1.00 & 0.85 & 0.68 & 0.02 & 0.00 & 0.77 \\
  hopper-m-e & 0.00 & 0.73 & 0.95 & 0.71 & 0.94 & 0.04 & 0.93 & 0.85 & 0.90 & 0.45 & 0.43 & 0.26 \\
  walker2d-m-e & 0.18 & 0.11 & 0.00 & 0.31 & 0.02 & 0.04 & 1.00 & 0.13 & 0.44 & 0.58 & 0.67 & 0.49 \\
  \hline
  \end{tabular}
  }
  \label{tabular:D4RL_sota}
\end{table*}

\begin{figure}
  \centering
  \includegraphics[width = 0.48\textwidth]{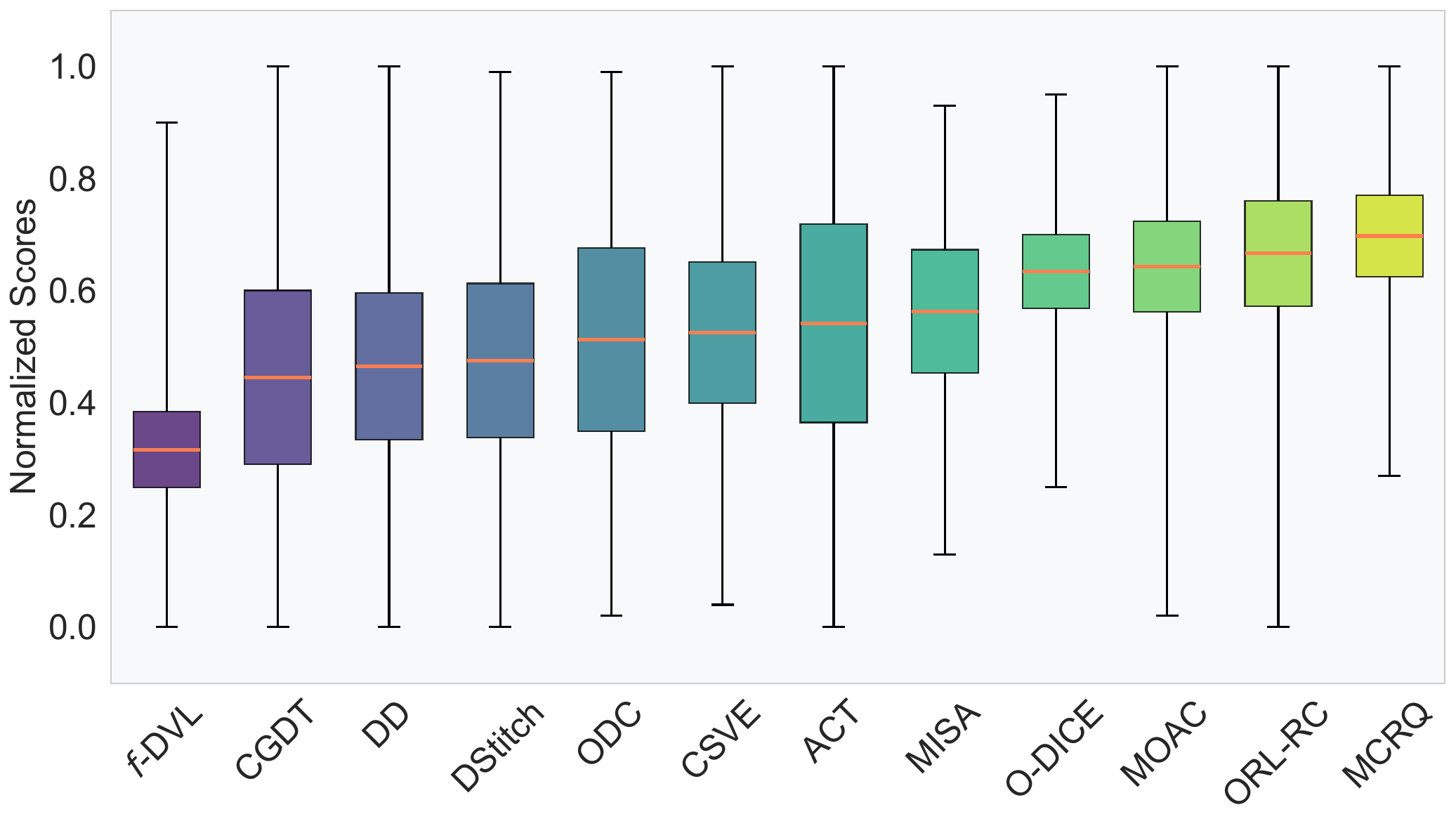}
  \hfill
  \caption{Box plot with minimum, maximum, mean, and variance of 9 normalized average scores for MCRQ and SOTA algorithms.}
  \label{fig:D4RL_sota_normalized}
\end{figure}

\subsection{Computational Efficiency}
\label{subsection:computational-efficiency-comparison}

Table~\ref{tabular:D4RL_Traning_time_scores} presents the training times of BCQ, TD3\_BC, CQL, IQL, and MCRQ on halfcheetah datasets. All experiments are conducted on a machine equipped with a single AMD EPYC 9004 series processor, 384 GB DDR5 RAM, and two NVIDIA RTX 4090 GPUs (each with 24 GB VRAM), running Ubuntu 20.04. MCRQ has a slightly longer training time than TD3\_BC. However, compared to BCQ, CQL, and IQL, it significantly improves training efficiency by minimizing computational overhead while maintaining strong performance.

\begin{table}[!t]
  \caption{Comparison of training time (single seed run).} 
  \centering
  \begin{tabular}{llllll}
    \toprule
    Algorithm & BCQ & TD3\_BC & CQL & IQL & MCRQ  \\
    \midrule
    Time (hours) & 21.09 & 12.85 & 34.21 & 19.31 & 14.05 \\
    \bottomrule
  \end{tabular}
  \label{tabular:D4RL_Traning_time_scores}
\end{table}

\section{Conclusion} \label{section:conclusion}

In this paper, we proposed MCRE to address distribution shift and OOD actions in offline RL. Our theoretical analysis shows that MCRE converges in both the presence and absence of sampling error. Moreover, the estimated Q-function, state value function, and the resulting suboptimal policy are all shown to be controllable under both conditions. Building on MCRE, we introduced MCRQ, an effective offline RL algorithm, and provided detailed implementation and experimental results on the D4RL benchmark. The experiments demonstrate that MCRQ outperforms both baseline and SOTA offline RL algorithms. Future work will explore extending the MCRE framework to other RL paradigms.

\bibliographystyle{IEEEtran}
\bibliography{IEEEabrv,refs}

\begin{thebibliography}{10}
\providecommand{\url}[1]{#1}
\csname url@samestyle\endcsname
\providecommand{\newblock}{\relax}
\providecommand{\bibinfo}[2]{#2}
\providecommand{\BIBentrySTDinterwordspacing}{\spaceskip=0pt\relax}
\providecommand{\BIBentryALTinterwordstretchfactor}{4}
\providecommand{\BIBentryALTinterwordspacing}{\spaceskip=\fontdimen2\font plus
\BIBentryALTinterwordstretchfactor\fontdimen3\font minus \fontdimen4\font\relax}
\providecommand{\BIBforeignlanguage}[2]{{%
\expandafter\ifx\csname l@#1\endcsname\relax
\typeout{** WARNING: IEEEtran.bst: No hyphenation pattern has been}%
\typeout{** loaded for the language `#1'. Using the pattern for}%
\typeout{** the default language instead.}%
\else
\language=\csname l@#1\endcsname
\fi
#2}}
\providecommand{\BIBdecl}{\relax}
\BIBdecl

\bibitem{lange2012batch}
S.~Lange, T.~Gabel, and M.~Riedmiller, ``Batch reinforcement learning,'' in \emph{Reinforcement Learning: State-of-the-Art}.\hskip 1em plus 0.5em minus 0.4em\relax Berlin, Germany: Springer, 2012, pp. 45--73.

\bibitem{huOptimizingReinforcementLearning2025}
J.~Hu, X.~Li, X.~Li, Z.~Hou, and Z.~Zhang, ``Optimizing reinforcement learning for large action spaces via generative models: {{Battery}} pattern selection,'' \emph{Pattern Recognit}, vol. 160, p. 111194, Apr. 2025.

\bibitem{liu2024diffskill}
S.~Liu, Y.~Liu, L.~Hu, Z.~Zhou, Y.~Xie, Z.~Zhao, W.~Li, and Z.~Gan, ``Diffskill: Improving reinforcement learning through diffusion-based skill denoiser for robotic manipulation,'' \emph{Knowl Based Syst}, vol. 300, p. 112190, Sep. 2024.

\bibitem{zhangTextBasedInteractiveRecommendation2022}
R.~Zhang, T.~Yu, Y.~Shen, and H.~Jin, ``Text-based interactive recommendation via offline reinforcement learning,'' in \emph{Proc. AAAI Conf. Artif. Intell.}, vol.~36, no.~10, 2022, pp. 11\,694--11\,702.

\bibitem{rezaeifar2022offline}
S.~Rezaeifar, R.~Dadashi, N.~Vieillard, L.~Hussenot, O.~Bachem, O.~Pietquin, and M.~Geist, ``Offline reinforcement learning as anti-exploration,'' in \emph{Proc. AAAI Conf. Artif. Intell.}, vol.~36, no.~7, 2022, pp. 8106--8114.

\bibitem{sikchi2024dualf-DVL}
H.~Sikchi, Q.~Zheng, A.~Zhang, and S.~Niekum, ``Dual {RL}: Unification and new methods for reinforcement and imitation learning,'' in \emph{Proc. Int. Conf. Learn. Represent.}, 2024.

\bibitem{pessimismofflineRLzhang2024}
D.~Zhang, B.~Lyu, S.~Qiu, M.~Kolar, and T.~Zhang, ``Pessimism meets risk: Risk-sensitive offline reinforcement learning,'' in \emph{Proc. Int. Conf. Mach. Learn.}, vol. 235, 2024, pp. 59\,459--59\,489.

\bibitem{rao2025isfors}
J.~Rao, C.~Wang, M.~Liu, J.~Lei, and W.~Giernacki, ``Isfors-mix: Multi-agent reinforcement learning with importance-sampling-free off-policy learning and regularized-softmax mixing network,'' \emph{Knowl Based Syst}, vol. 309, p. 112881, Jan. 2025.

\bibitem{yangHundredsGuideMillions2024}
Q.~Yang, S.~Wang, Q.~Zhang, G.~Huang, and S.~Song, ``Hundreds guide millions: Adaptive offline reinforcement learning with expert guidance,'' \emph{IEEE Trans. Neural Netw. Learn. Syst.}, vol.~35, no.~11, pp. 16\,288--16\,300, Nov. 2024.

\bibitem{lyuMildlyConservativeQLearning2022}
J.~Lyu, X.~Ma, X.~Li, and Z.~Lu, ``Mildly conservative q-learning for offline reinforcement learning,'' in \emph{Proc. Adv. Neural Inf. Process. Syst.}, vol.~35, 2022, pp. 1711--1724.

\bibitem{huangEfficientOfflineReinforcement2024}
L.~Huang, B.~Dong, and W.~Zhang, ``Efficient offline reinforcement learning with relaxed conservatism,'' \emph{IEEE Trans. Pattern Anal. Mach. Intell.}, vol.~46, no.~8, pp. 5260--5272, Aug. 2024.

\bibitem{Fakoor2021ContinuousDoublyConstrained}
R.~Fakoor, J.~W. Mueller, K.~Asadi, P.~Chaudhari, and A.~J. Smola, ``Continuous doubly constrained batch reinforcement learning,'' in \emph{Proc. Adv. Neural Inf. Process. Syst.}, vol.~34, 2021, pp. 11\,260--11\,273.

\bibitem{brandfonbrener2021offline}
D.~Brandfonbrener, W.~Whitney, R.~Ranganath, and J.~Bruna, ``Offline rl without off-policy evaluation,'' in \emph{Proc. Adv. Neural Inf. Process. Syst.}, vol.~34, 2021, pp. 4933--4946.

\bibitem{Fujimoto2021td3bc}
S.~Fujimoto and S.~S. Gu, ``A minimalist approach to offline reinforcement learning,'' in \emph{Proc. Adv. Neural Inf. Process. Syst.}, vol.~34, 2021, pp. 20\,132--20\,145.

\bibitem{kumarConservativeQLearningOffline2020}
A.~Kumar, A.~Zhou, G.~Tucker, and S.~Levine, ``Conservative q-learning for offline reinforcement learning,'' in \emph{Proc. Adv. Neural Inf. Process. Syst.}, vol.~33, 2020, pp. 1179--1191.

\bibitem{kostrikovOfflineReinforcementLearning2021}
I.~Kostrikov, R.~Fergus, J.~Tompson, and O.~Nachum, ``Offline reinforcement learning with fisher divergence critic regularization,'' in \emph{Proc. Int. Conf. Mach. Learn.}, vol. 139, 2021, pp. 5774--5783.

\bibitem{Kostrikovimplicitq-learning2021}
I.~Kostrikov, A.~Nair, and S.~Levine, ``Offline reinforcement learning with implicit q-learning,'' \emph{arXiv preprint arXiv:2110.06169}, 2021.

\bibitem{bai2022pessimistic}
C.~Bai, L.~Wang, Z.~Yang, Z.-H. Deng, A.~Garg, P.~Liu, and Z.~Wang, ``Pessimistic bootstrapping for uncertainty-driven offline reinforcement learning,'' in \emph{Proc. Int. Conf. Learn. Represent.}, 2022.

\bibitem{huangOfflineReinforcementLearning2024}
L.~Huang, B.~Dong, W.~Xie, and W.~Zhang, ``Offline reinforcement learning with behavior value regularization,'' \emph{IEEE Trans. Cybern.}, vol.~54, no.~6, pp. 3692--3704, Jun. 2024.

\bibitem{ran2023policy}
Y.~Ran, Y.-C. Li, F.~Zhang, Z.~Zhang, and Y.~Yu, ``Policy regularization with dataset constraint for offline reinforcement learning,'' in \emph{Proc. Int. Conf. Mach. Learn.}, 2023, pp. 28\,701--28\,717.

\bibitem{fujimotoOffPolicyDeepReinforcement2019}
S.~Fujimoto, D.~Meger, and D.~Precup, ``Off-policy deep reinforcement learning without exploration,'' in \emph{Proc. Int. Conf. Mach. Learn.}, vol.~97, 2019, pp. 2052--2062.

\bibitem{kumarStabilizingOffPolicyQLearning2019}
A.~Kumar, J.~Fu, M.~Soh, G.~Tucker, and S.~Levine, ``Stabilizing off-policy q-learning via bootstrapping error reduction,'' in \emph{Proc. Adv. Neural Inf. Process. Syst.}, vol.~32, 2019.

\bibitem{wuUncertaintyWeightedActorCritic2021}
Y.~Wu, S.~Zhai, N.~Srivastava, J.~M. Susskind, J.~Zhang, R.~Salakhutdinov, and H.~Goh, ``Uncertainty weighted actor-critic for offline reinforcement learning,'' in \emph{Proc. Int. Conf. Mach. Learn.}, vol. 139, 2021, pp. 11\,319--11\,328.

\bibitem{nair2020awac}
A.~Nair, A.~Gupta, M.~Dalal, and S.~Levine, ``Awac: Accelerating online reinforcement learning with offline datasets,'' \emph{arXiv preprint arXiv:2006.09359}, 2020.

\bibitem{Tarasov2023revisitingminimalist}
D.~Tarasov, V.~Kurenkov, A.~Nikulin, and S.~Kolesnikov, ``Revisiting the minimalist approach to offline reinforcement learning,'' in \emph{Proc. Adv. Neural Inf. Process. Syst.}, vol.~36, 2023, pp. 11\,592--11\,620.

\bibitem{silverMasteringGameGo2016}
D.~Silver, A.~Huang, C.~J. Maddison, A.~Guez, L.~Sifre, G.~Van Den~Driessche, J.~Schrittwieser, I.~Antonoglou, V.~Panneershelvam, M.~Lanctot \emph{et~al.}, ``Mastering the game of go with deep neural networks and tree search,'' \emph{Nature}, vol. 529, no. 7587, pp. 484--489, 2016.

\bibitem{zhao2025RLBook}
S.~Zhao, \emph{Mathematical Foundations of Reinforcement Learning}.\hskip 1em plus 0.5em minus 0.4em\relax Springer Nature Press and Tsinghua University Press, 2025.

\bibitem{mao2023SupportedValueRegularization}
Y.~Mao, H.~Zhang, C.~Chen, Y.~Xu, and X.~Ji, ``Supported value regularization for offline reinforcement learning,'' in \emph{Proc. Adv. Neural Inf. Process. Syst.}, vol.~36, 2023, pp. 40\,587--40\,609.

\bibitem{huangMildPolicyEvaluation2024}
L.~Huang, B.~Dong, J.~Lu, and W.~Zhang, ``Mild policy evaluation for offline actor–critic,'' \emph{IEEE Trans. Neural Netw. Learn. Syst.}, vol.~35, no.~12, pp. 17\,950--17\,964, Dec. 2024.

\bibitem{levin2017markov}
D.~A. Levin and Y.~Peres, \emph{Markov chains and mixing times}.\hskip 1em plus 0.5em minus 0.4em\relax American Mathematical Soc., 2017, vol. 107.

\bibitem{fujimotoAddressingFunctionApproximation}
S.~Fujimoto, H.~van Hoof, and D.~Meger, ``Addressing function approximation error in actor-critic methods,'' in \emph{Proc. Int. Conf. Mach. Learn.}, 2018, pp. 1--15.

\bibitem{fu2020d4rl}
J.~Fu, A.~Kumar, O.~Nachum, G.~Tucker, and S.~Levine, ``D4rl: Datasets for deep data-driven reinforcement learning,'' \emph{arXiv preprint arXiv:2004.07219}, 2020.

\bibitem{todorovMuJoCoPhysicsEngine2012}
E.~Todorov, T.~Erez, and Y.~Tassa, ``Mujoco: A physics engine for model-based control.'' in \emph{Proc. IEEE/RSJ Int. Conf. Intell. Robots Syst. (IROS)}.\hskip 1em plus 0.5em minus 0.4em\relax IEEE, Oct. 2012, pp. 5026--5033.

\bibitem{brockmanOpenAIGym2016}
G.~Brockman, V.~Cheung, L.~Pettersson, J.~Schneider, J.~Schulman, J.~Tang, and W.~Zaremba, ``Openai gym,'' \emph{arXiv preprint arXiv:1606.01540}, 2016.

\bibitem{wang2024critic}
Y.~Wang, C.~Yang, Y.~Wen, Y.~Liu, and Y.~Qiao, ``Critic-guided decision transformer for offline reinforcement learning,'' in \emph{Proc. AAAI Conf. Artif. Intell.}, vol.~38, no.~14, 2024, pp. 15\,706--15\,714.

\bibitem{ajay2023isDD}
A.~Ajay, Y.~Du, A.~Gupta, J.~B. Tenenbaum, T.~S. Jaakkola, and P.~Agrawal, ``Is conditional generative modeling all you need for decision making?'' in \emph{Proc. Int. Conf. Learn. Represent.}, 2023.

\bibitem{li2024diffstitch}
G.~Li, Y.~Shan, Z.~Zhu, T.~Long, and W.~Zhang, ``Diffstitch: Boosting offline reinforcement learning with diffusion-based trajectory stitching,'' in \emph{Proc. Int. Conf. Mach. Learn.}, vol. 235, 2024, pp. 28\,597--28\,609.

\bibitem{kim2024decisionODC}
J.~Kim, S.~Lee, W.~Kim, and Y.~Sung, ``Decision convformer: Local filtering in metaformer is sufficient for decision making,'' in \emph{Proc. Int. Conf. Learn. Represent.}, 2024.

\bibitem{chen2023CSVE}
L.~Chen, J.~Yan, Z.~Shao, L.~Wang, Q.~Lin, S.~Rajmohan, T.~Moscibroda, and D.~Zhang, ``Conservative state value estimation for offline reinforcement learning,'' in \emph{Proc. Adv. Neural Inf. Process. Syst.}, vol.~36, 2023, pp. 35\,064--35\,083.

\bibitem{gao2024act}
C.-X. Gao, C.~Wu, M.~Cao, R.~Kong, Z.~Zhang, and Y.~Yu, ``Act: Empowering decision transformer with dynamic programming via advantage conditioning,'' in \emph{Proc. AAAI Conf. Artif. Intell.}, vol.~38, no.~11, 2024, pp. 12\,127--12\,135.

\bibitem{xiao2023MISA}
X.~Ma, B.~Kang, Z.~Xu, M.~Lin, and S.~Yan, ``Mutual information regularized offline reinforcement learning,'' in \emph{Proc. Adv. Neural Inf. Process. Syst.}, vol.~36, 2023, pp. 19\,058--19\,072.

\bibitem{mao2024odiceO-DICE}
L.~Mao, H.~Xu, W.~Zhang, and X.~Zhan, ``{ODICE}: Revealing the mystery of distribution correction estimation via orthogonal-gradient update,'' in \emph{Proc. Int. Conf. Learn. Represent.}, 2024.

\end{thebibliography}

\end{document}